\documentclass[smallcondensed]{svjour3}
\usepackage[utf8]{inputenc}
\usepackage{pdfpages}
\usepackage{float}
\usepackage[colorinlistoftodos,prependcaption,
textwidth=\marginparwidth,textsize=footnotesize]{todonotes}
\usepackage{graphicx}
\usepackage{amsmath}
\usepackage{mathtools}
\usepackage{setspace}

\usepackage{algorithmic}
\usepackage{algorithm}

\usepackage{amssymb}
\usepackage{stmaryrd}
\usepackage{xcolor}
\usepackage{mathtools}
\usepackage{mathrsfs}
\usepackage{dsfont}
\usepackage{subcaption}
\usepackage{enumitem}
\usepackage{hyperref}
\usepackage{comment}
\usepackage{blindtext}
\usepackage{natbib}
\usepackage{array}
\usepackage{thmtools}

\definecolor{deep_blue}{rgb}{0,.2,.5}
\definecolor{dark_green}{rgb}{0,0.5,.15}
\definecolor{deep_brown}{RGB}{128, 70, 21}

\hypersetup{pdftex,  
  breaklinks=true,  
  colorlinks=true,
  urlcolor=deep_brown,
  linkcolor=deep_blue,
  citecolor=dark_green,
  }

\usepackage{enumitem}

\newcommand{\NA}{\mathtt{NA}}
\newcommand{\density}{g}
\newcommand{\regTilde}{f^{\star}_{\bXm}}
\newcommand{\regPrime}{f^{\star}_{\textrm{SI}}}
\newcommand{\bM}{\mathbf{M}}
\newcommand{\bm}{\mathbf{m}}

\newcommand{\regMultImput}{f^{\star}_{\textrm{MI}}}
\newcommand{\impConst}{\alpha}

\newcommand{\E}{\mathbb{E}}
\newcommand{\R}{\mathbb{R}}

\newcommand{\X}{\mathbf{X}}
\newcommand{\Xm}{\widetilde{X}}
\newcommand{\bXm}{\widetilde{\mathbf{X}}}

\newcommand{\M}{\mathbf{M}}
\newcommand{\x}{\mathbf{x}}
\newcommand{\bx}{\mathbf{x}}
\newcommand{\bX}{\mathbf{X}}

\newcommand{\distribution}{P}
\newcommand{\dataset}{\mathcal{D}}

\newcommand{\bxm}{\widetilde{\mathbf{x}}}
\newcommand{\m}{\mathbf{m}}

\newcommand{\argmin}{\mathrm{argmin}}
\newcommand*\diff{\mathop{}\!\mathrm{d}}
\renewcommand{\P}{\mathbb{P}}
 \usepackage[normalem]{ulem}
 
\definecolor{OliveGreen}{rgb}{0.24, 0.71, 0.54}
\definecolor{RoyalBlue}{rgb}{0.0, 0.47, 0.75}
\definecolor{BrickRed}{rgb}{0.77, 0.12, 0.23}
\definecolor{Vert}{RGB}{0,128,0}

\newcommand{\replace}[2]{ #2}

\newcommand{\indep}{\rotatebox[origin=c]{90}{$\models$}}

\makeatletter
\renewcommand*\env@matrix[1][*\c@MaxMatrixCols c]{%
  \hskip -\arraycolsep
  \let\@ifnextchar\new@ifnextchar
  \array{#1}}
\makeatother

\newtheorem{assumption}{Assumption}
\newtheorem{model}{Model}
\newtheorem{missing_pattern}{Missing Pattern}
\newcommand{\modif}[1]{{#1}}

\begin{document}

\title{On the consistency of supervised learning with missing values}
\author{Julie Josse
\and
Jacob M Chen
\and
Nicolas Prost
\and
Ga\"el Varoquaux
\and
Erwan Scornet%
}
\date{\today}

\institute{J. Josse
\at
Centre de Math\'emathiques Appliqu\'ees, Ecole Polytechnique\qquad
\email{julie.josse@inria.fr}\\
\emph{Present address: INRIA-INSERM University of Montpellier}
\and
J.M. Chen 
\at Williams College, Massachusetts
\and
N. Prost 
\at Centre de Math\'emathiques Appliqu\'ees, Ecole Polytechnique\\ Parietal project-team, INRIA Saclay
\and
E. Scornet 
\at  Sorbonne Université and Université Paris Cité, CNRS,\\
Laboratoire de Probabilités, Statistique et Modélisation, F-75005 Paris, France 
\and
G. Varoquaux
\at Parietal project-team, INRIA Saclay}

\maketitle 

\begin{abstract}
In many application settings, data have missing entries, which makes subsequent analyses challenging. An abundant literature addresses missing values in an inferential framework, aiming at estimating parameters and their variance from incomplete tables. Here, we consider supervised-learning settings: predicting a target when missing values appear in both training and test data. We first rewrite classic missing values results for this setting. We then show the consistency of two approaches, test-time multiple imputation and single imputation in prediction. A striking result is that the widely-used method of imputing with a constant prior to learning is consistent when missing values are not informative. This contrasts with inferential settings where mean imputation is frowned upon as it distorts the distribution of the data. The consistency of such a popular simple approach is important in practice. Finally, to contrast procedures based on imputation prior to learning with procedures that optimize the missing-value handling for prediction, we consider decision trees. Indeed, decision trees are among the few methods that can tackle empirical risk minimization with missing values, due to their ability to handle the half-discrete nature of incomplete variables. After comparing empirically different missing values strategies in trees, we recommend using the ``missing incorporated in attribute'' method as it can handle both non-informative and informative missing values. 

\end{abstract}


\section{Introduction}

As volumes of data increase, they are harder to curate and clean.
They may come from the aggregation of various
sources (\emph{e.g.} merging multiple databases)
and contain variables of different natures (\emph{e.g.}
different sensors). Such heterogeneous data
collections can lead to many missing values: samples only come with a
fraction of the features observed. Though there is a vast literature on
treating missing values, it focuses on
estimating parameters and their variance in the presence
of missing values in a single data set. In contrast, there are few
studies of supervised-learning settings where
the aim is to predict a target variable given input variables \modif{with missing entries}. These
settings only require to use  \emph{discriminative} (or conditional)
models, compared to the first inferential frameworks, which often assume
parametric data distributions (generative modelling). \replace{In
addition, in supervised learning, the quantities of interest are related to
the prediction target, and not to the input, hence a liberal stance
can be taken on what consists in a valid input, such as accepting
a special missingness symbol.}{}Besides, a
predictive model is applied on a test set, different from the training
set, a separation
seldom considered in inferential settings. Therefore, inference and prediction in the presence of missing values are intrinsically two very different problems. 

Beyond the aggregation of multiple sources, missing values can
appear for a variety  of reasons. For sensor data, missing values can 
arise from device failure. Informative missing values can be found 
for instance in poll data where participants may not answer
sensitive questions related to unpopular opinions.
In medical studies, some measurements may be impractical on
patients in a critical state, in which case the presence of missing
values can be related to the 
variable of interest, target of the prediction (\emph{e.g.} patient
status). These various scenarios correspond to different missing-value mechanisms.

The classical literature on missing values, led by
\citet{rubin1976inference}, defines missing-values mechanisms based on the
relationship between missingness and observed values: if they are
independent, the mechanism is said to be Missing Completely At Random
(MCAR); if the missingness only depends on the observed values, then it
is Missing At Random (MAR); otherwise it is Missing Not At Random (MNAR).
However, this nomenclature has seldom been discussed in the context of
supervised learning, accounting for the target variable of the prediction. 

Many statistical methods tackle missing values \citep{josse2018, mayer2019r}. 
Listwise deletion, \emph{i.e.} removing incomplete observations,
may allow to train the model on complete data. However, 
it may lead to the deletion of almost all data especially in high dimension and may result in biased sample depending on the missing values mechanism. 
In addition, it does not
suffice for supervised learning, as the test set may also contain
 incomplete data. Hence the prediction procedure should handle missing
values. A popular solution is to
impute missing values, that is to replace them with
plausible values to produce a completed data set. The benefit of
imputation is that it adapts existing pipelines and software to
the presence of missing values.
The widespread practice of
imputing with the mean of the variable on the observed entries has
serious drawbacks in inferential settings, as it distorts
the joint and marginal distributions of the data which induces biased
estimators \citep{little2002statistical}. Yet,
\modif{few studies focus on the impact of \replace{mean}{constant} imputation in a predictive setting.}
Imputation itself must be revisited for out-of-sample prediction
settings: users resort to different strategies such as imputing separately the train and test sets or imputing them jointly. 
More elaborate
strategies rely on using maximum likelihood with expectation maximization (EM) to fit a model on
incomplete data \citep{dempster1977maximum, little1992regression, jiang2018logistic}.
However, such techniques 
often rely on strong parametric assumptions. 
Alternatively, some learning algorithms, such as decision
trees, can readily handle missing values, accounting for their discrete
nature.

\bigskip

In this paper, we study the classic tools of missing values in the
context of supervised learning.
We start in Section  \ref{sec:definitions} by setting the notations and
briefly summarizing the missing-value literature.
Our first contribution, detailed in Section
\ref{sec:supervisedtraintest}, is to adapt the 
formalism for missing values to supervised learning: we show how to use
standard missing-values techniques to make predictions on a test set with missing values. 
Section \ref{sec:Bayesrisk} presents our main contribution:
studying the consistency of two approaches to estimate the prediction
function with missing values. The first theorem states that, given an
optimal predictor for the completely-observed data, a consistent
procedure can be built by predicting on a test set where missing entries
are replaced by multiple imputation. The second theorem, which is the
most striking and has important consequences in practice, shows that constant \replace{(and, in particular, mean)}{} imputation prior to learning is consistent for supervised learning. This
is, as far as we know, the first result justifying this very convenient practice of handling missing values.
In Section \ref{sec:simu}, we compare imputation to learning directly with missing values via decision trees. Indeed, their greedy
and discrete natures allow adapting them to handle missing values directly. We
compare the different tree methods \modif{(together with classic machine learning  approaches such as SVM or nearest neighbours)} on simulated data with missing
values and recommend to use the ``Missing incorporated in attributes'' (MIA, \citealt{Twala:2008:GMC:1352941.1353228}) approach, whose good predictive performances have been highlighted by \cite{kapelner2015prediction}, one of the few 
studies of trees with missing values for supervised learning.  \modif{Other experimental works have shown that tree-based methods are competitive in terms of predictive performances \citep[see][]{jager2021benchmark}. }
We also show the benefits for prediction
of an approach often used in practice, which consists in ``adding the
mask'', \textit{i.e.} adding binary variables that code for the
missingness of each variables as new covariates, even though this method
is not recommended for parameter estimation \citep{jones1996indicator}.

\section{Definitions, problem setting, prior art}\label{sec:definitions}

\paragraph{Notation.} \label{notationofthefirstsection}

Throughout the paper, $\mathbf{bold}$ letters refer to vectors; CAPITAL letters refer to random variables, while lower-case letters are realisations. 
In addition, as usual, for any two variables $A$ and $B$ of joint density $g$,
\begin{equation*}
\density(b) \coloneqq \density_B(b) \coloneqq \int \density(\alpha, b) \diff\mu(\alpha), \quad\qquad\qquad
\density(a|b) \coloneqq \density_{A|B=b}(a) \coloneqq \frac{\density(a, b)}{\density(b)}.
\end{equation*}
\subsection{Supervised learning}

Supervised learning is typically focused on learning to predict a
\emph{target} $Y \in \mathcal{Y}$ from inputs $\bX \in \mathcal X = \bigotimes_{j=1}^{d} \mathcal X_j$,
where the pair $(\X, Y)$ is considered as random, drawn from a
distribution $\distribution$. Formally, the goal is to find a function $f:
\mathcal{X} \rightarrow \mathcal{Y}$, that minimizes $\mathbb{E}[\ell(f(\X),
Y)]$ given a cost function $\ell: \mathcal{Y} \times \mathcal{Y} \rightarrow
\mathbb{R}$, called the \emph{loss} \citep{vapnik1999overview}. The best possible prediction
function is known as the \emph{Bayes predictor}, given by
\begin{equation}\label{eq:bayesrule}
    f^{\star} \in \underset{f:\, \mathcal{X} \rightarrow \mathcal{Y}}
	\argmin \;
	\mathbb{E}\left[ \ell(f(\X), Y) \right],
\end{equation}
and its expected loss is the \emph{Bayes loss} \citep{devroye2013probabilistic}.
A \emph{learning} procedure is used to create a function $f$ based on a set
of \emph{training} pairs $ \dataset_{n, \textrm{train}} = \{ (\X_i,
Y_i), i = 1, \hdots , n\}$. The function $f$ is therefore itself a function of
$\dataset_{n, \textrm{train}}$, and can be written
$\hat{f}_{\dataset_{n, \textrm{train}}}$ or simply $\hat{f}_n$. There are many different learning
procedures, including random forests \citep{breiman2001random} or support vector machines \citep[SVM, see][]{cortes1995support}. A
learning procedure that, given an infinite amount of data, yields a
function that achieves the Bayes loss is said to be \emph{Bayes
consistent}. In other words, $\hat{f}_n$ is Bayes consistent if
$$ \lim\limits_{n \to \infty} \E[\ell(\hat{f}_n(\X), Y)] =  \E[\ell(f^\star(\X),
Y)].$$
A learning procedure that is Bayes
consistent for every distribution $(\X, Y)$ is said to be (Bayes) universally consistent.
In a classification setting, $Y$ is drawn from a finite set of discrete values, and the cost $\ell$ is typically the zero-one loss: $\ell(Y_1, Y_2) = 
\mathds{1}_{Y_1 \neq Y_2}$.
In a regression setting, $Y$ is drawn from continuous
values in $\R$ and is assumed to satisfy $\E[Y^2]<\infty$. A common cost
is then the square loss, $\ell(Y_1, Y_2) = (Y_1 - Y_2)^2$.
Considering the zero-one loss \citep{rosasco2004loss} or the square loss (see \emph{e.g.} sec 1.5.5 of
\cite{bishop2006prml}), the  \replace{Bayes-optimal function}{Bayes predictor} $f^{\star}$, that
minimizes the expected loss,
satisfies $f^{\star}(\X) = \mathbb{E}[Y|\X]$.  

Note that the learning procedure has access to a finite sample
$\dataset_{n, \textrm{train}}$, and not to the distribution $\distribution$, 
hence it can only use the \emph{empirical} risk, $\sum_{i=1\hdots n}
\ell(f(\X_i), Y)$, rather than the expected risk.
A typical learning procedure is therefore the \emph{empirical risk minimization} defined as the following optimization problem
$$\hat{f}_n \in \underset{f:\mathcal X \to \mathcal Y}{\argmin}~
\left( \frac{1}{n} \sum_{i=1}^n \ell \left(f(\X_i), Y_i\right)\right).$$
A new data set $\dataset_{n,
\textrm{test}}$ is then needed to estimate the generalization error
of the resulting function \replace{$f$}{$\hat{f}_n$}.

\subsection{Background on missing values}

In this section, we introduce the classic work on missing values,
including the different missing-value mechanisms. We then
summarize the main
methods to handle missing values: imputation and likelihood-based methods.
Most of this prior art to deal with missing values is based on a
single data set with no distinction between training and test set.
One challenge in formalizing statistical learning with missing values is
adapting notations to describe precisely incomplete feature vectors: 
notations not explicit enough have led to confusions
\citep{seaman2013meant}; the
literature is in flux \cite[preface]{little2002statistical}.

\paragraph{Notations for missing values}

In presence of missing values, we do not observe a complete vector $\bX$.
To define precisely the observed quantity, we introduce the missing indicator
vector $\M \in \{0,1\}^d$ which satisfies, for all $1 \leq j \leq d$,
$M_j =1$ if and only if $X_j$ is not observed. The random vector $\bM$
acts as a mask on $\bX$. To formalize incomplete observations, 
we use the incomplete feature vector $\bXm$ (see
\citet{rubin1976inference}, \citealt[appendix B]{rosenbaum1984reducing}; \citealt{mohan2018graphical, yoon2018gain}) defined as $\Xm_{j}=\NA$ if $M_{j}=1$, and  $\Xm_{j}=X_j$ otherwise. As $\mathcal X$ is a cartesian product, 
$\bXm$ belongs to the space $\widetilde{\mathcal X} = \bigotimes_{j=1}^{d} (\mathcal X_j \cup \{\NA\})$. We have
$$\bXm = \X\odot(\mathbf 1 - \M) + \NA\odot\M,$$
where $\odot$ is the term-by-term product, with the convention that, for all one-dimensional $x \neq 0$, $\NA\cdot x=\NA$ and $\NA\cdot 0=0$.
As such, when the data are real, $\bXm$ can be seen as a mixed categorical and continuous variable, taking values in $\R\cup\{\NA\}$. Here is an example of realizations (lower-case letters) of previous random variables: for a given vector $\bx = (1.1, 2.3, -3.1, 8, 5.27)$ with the missing pattern $\bm = (0,1,0,0,1)$, we have
\begin{equation*}
\bxm = (1.1, ~~\texttt{NA}, ~-3.1, ~~8, ~~\texttt{NA}).
\end{equation*}

To write likelihoods (see Section~\ref{sec:missingdata}), we must also introduce
notations \replace{$\bX_{\textrm{o}}$ and $\bX_{\textrm{m}}$, classic in the missing value literature.}{$\bx_{obs(m)}$ and $\bx_{mis(m)}$, classic in the missing value literature. For any vector \textbf{x}, and any set $ J \subset \{1, \hdots, d\}$, we let $\textbf{x}_J$ be the subvector of $\textbf{x}$ composed of the components of $\textbf{x}$ indexed by $J$. We also let $|J|$ be the cardinal of $J$. For any vector $\bm \in \{0,1\}^d$, we let $obs(\bm) = \{j \in \{1, \hdots, d\}, \bm_j = 0\}$ and $mis(\bm) = \{j \in \{1, \hdots, d\}, \bm_j = 1\}$. }
\replace{The vector $\bX_{\textrm{o}} = o(\bX, \bM)$ is  composed of observed entries in $\bX$, whereas $\bX_{\textrm{m}} = o(\bX, 1-\bM)$ contains the missing components of $\bX$.}{Hence $\bx_{obs(\bm)} \in  \bigotimes_{j \in obs(\bm)} \mathcal X_j$ is  composed of observed entries in $\bx$ and $\bx_{mis(\bm)}\in  \bigotimes_{j \in mis(\bm)} \mathcal X_j$ contains the missing components in $\bx$. To shorten notations, we will sometimes write $\bx_o$ (resp. $\bx_m$) instead of $\bx_{obs(\bm)}$ (resp. $\bx_{mis(\bm)}$).} To continue the above example, we have
\begin{align*}
\bx_{obs(\bm)} = \bx_{(1,3,4)}= (1.1, ~-3.1, ~~8), \qquad \bx_{mis(\bm)} = \bx_{(2,5)} = (2.3, ~~5.27).
\end{align*}
\replace{These notations write partial feature vectors: ``$\cdot$'' in second and fifth positions in $\bx_{\textrm{o}}$ means that we do not specify the corresponding component. More precisely, $\bx_{\textrm{o}}$ specifies values of the first, third and fourth component of $\bx_{\textrm{o}}$ but not whether the second and fifth values are observed or not (and if observed, we do not know the exact values). Notation $\bxm$ is thus different from $\bx_{\textrm{o}}$ since $\NA$ specifies which components are missing.}{} \replace{}{Given a vector $\bx \in \bigotimes_{j=1}^{d} \mathcal X_j $ and a missingness pattern $\bm$, one can recompose $\bx$ based on $\bx_{obs(\bm)}$, $\bx_{mis(\bm)}$ and $\bm$ by the ordering operator $\bx = o(\bx_{obs(\bm)}, \bx_{mis(\bm)}; \bm)$.}


Finally, we use the generic notation $\hat{\bx}$ to denote $\bxm$ in which missing values have been imputed, based on any imputation procedure to be specified. For example, if we impute the missing values by $0$ in the previous example, we have 
$$\hat{\bx} = (1.1, ~~0 , -3.1, ~~8, ~~0).$$
Notations related to missing data are summarized in Table~\ref{tab:notations}.

\begin{table}[h!!]
  \centering
  \begin{tabular}{lp{.6\linewidth}}
    \textbf{Notation} & \textbf{Description} \\ \hline
    $\bx \in \mathcal{X}$  & Complete input vector  \\
    $y \in \mathcal{Y}$  & Complete output (always observed) \\
    $ \bm \in \{0,1\}^d$ & Missingness indicator vector\\
    $\bxm \in \bigotimes_{j=1}^{d} (\mathcal X_j \cup \{\NA\})$ & Observed vector $\bx$ where missing entries are written as $\texttt{NA}$ \\
    $\bx_{obs(\bm)}\in  \bigotimes_{j \in obs(\bm)} \mathcal X_j$  & Subvector of $\bx$ containing its observed components \\
    $\bx_{mis(\bm)}\in  \bigotimes_{j \in mis(\bm)} \mathcal X_j$  & Subvector of $\bx$ containing its missing components\\
    $\bx = o(\bx_{obs(\bm)}, \bx_{mis(\bm)}; \bm)$ & Complete vector $\bx$ recreated by merging  $\bx_{obs(\bm)}$ and $\bx_{mis(\bm)}$ according to the missing pattern $\bm$\\
    $\hat{\bx} \in \mathcal{X}$  & Vector $\bx$ for which missing entries have been imputed \\
  \end{tabular}
  \caption{Notations and definitions used throughout the paper.  Bold variables are
  vectors}
    \label{tab:notations}
\end{table}

\subsubsection{Missing data mechanisms}\label{sec:missingdata}

To follow the historical definitions which do not give to the response $Y$ a particular role, we temporarily consider $Y$ as part of the input vector
$\X$, though
we assume that $Y$ has no missing values. 
\citet{rubin1976inference} defines three missing data mechanisms and fundamental results for working
with likelihood models in the presence of missing data.
Let us consider that realizations $(\bx_i,\bm_i)$ are sampled i.i.d. from a distribution in
$\mathscr P=\{\density_\theta(\x) \density_\phi(\m|\x) : (\theta,\phi)\in\Omega_{\theta,\phi}\}$ where $\Omega_{\theta,\phi} \subset \Theta \times \Phi$ is the joint parameter space (marginally, $\theta\in\Theta$ and $\phi\in\Phi$).  The goal in statistical inference is to estimate the parameter $\theta$. This is usually done by maximizing the likelihood $\mathcal L(\theta)=\prod_{i=1}^n \density_\theta(\x_i)$, which is well defined when the $\x_i$ are fully observed. Recall that each $\x_i$ can be decomposed into an observed vector $\x_{i,o}$ and an unobserved vector $\x_{i,m}$. Here, the likelihood is integrated over the missing values, resulting in
\begin{flalign*}
\text{(full likelihood)}&&
\mathcal L_1(\theta, \phi)=
\prod_{i=1}^n 
\int \density_\theta(\replace{\x_{i,o}, \x_{i,{m}}}{o(\x_{i,o}, \x_{i,m}; \bm)}) \density_\phi(\m_i|\replace{\x_{i,{o}},  \x_{i,m}}{o(\x_{i,o}, \x_{i,m}; \bm)})  \diff \x_{i,m},
\end{flalign*}
where the integration is taken on the components of $\bx_{i,m}$ only.
The parameter $\phi$ is generally not considered as of interest. 
In addition, modelling the missing values mechanism may require strong parametric assumptions. An easier quantity would be
\begin{flalign*}
\text{(likelihood of observed data)}&&
\mathcal L_2(\theta)=
\prod_{i=1}^n 
\int \density_\theta(\replace{\x_{i,{\textrm{o}}}, \x_{i,{\textrm{m}}}}{o(\x_{i,o}, \x_{i,m}; \bm)}) \diff \x_{i,m}
\end{flalign*}
ignoring the missing data mechanism.
To leave the difficult term, \textit{i.e.} the missing values mechanism,  out of the expectation, \citet{rubin1976inference} introduces an \emph{ad hoc} assumption, called \emph{Missing At Random (MAR)}, which is that for all $\phi\in\Phi$, for all  $i\in\llbracket 1,n \rrbracket$, for all $\x'\in\mathcal X$,
\begin{align*}
\x'_{obs(\bm_i)} = \x_{i,o}, \;\;\Rightarrow\;\; g_\phi(\m_i|\x')=g_\phi(\m_i|\x_i),
\end{align*}
\begin{flalign*}
\text{for instance, for all~} a,b \in \R, &&
g_\phi((0,1,0,0)|(1,a,3,10))=g_\phi((0,1,0,0)|(1,b,3,10)).
&&&&
\end{flalign*}
Using this assumption, he states the following result.
\begin{theorem}[Theorem $7.1$ in \cite{rubin1976inference}]
\label{theo:mar}
Let $\phi$ such that for all $1 \leq i \leq n$, $g_\phi(\m_i|\x_i)$ $>0$. Assume (a) MAR, (b) $\Omega_{\theta,\phi}=\Theta\times\Phi$, then $\mathcal L_2(\theta)$ is proportional to $\mathcal L_1(\theta,\phi)$ with respect to $\theta$, so that the inference for $\theta$ can be obtained by maximizing the likelihood $\mathcal L_2$, which ignores the missing mechanism.
\end{theorem}

MAR has a stronger version, more intuitive: \emph{Missing Completely At Random (MCAR)}. In its simplest and strongest form, it states that $\M\indep\X$ (the model density is $\density_\theta(\x) \density_\phi(\m)$). At the other end of the spectrum, if it is not possible to ignore the mechanism, the corresponding model is called \emph{Missing Not At Random (MNAR)}. 

These definitions are often subject to debates \citep{seaman2013meant} but can be understood using the following example: let us consider  two variables, income and age with missing values on income. MCAR means that the missing values are independent of any values; MAR is satisfied if missing values on income depend on the values of age (older people are less incline to reveal their income) whereas MNAR holds if rich people are less prone to reveal their income. 

\medskip

There is little literature on missing data mechanism for supervised
learning or discriminative models.
\citet{kapelner2015prediction} formalise the problem by separating the role of the response $y$, factorising the likelihood as $\density_\theta(\x) \density_\phi(\m|\x) \density_\chi(y|\x,\m)$. 
Note that they do not write $\density_\phi(\m|\x,y)$. 
They justify this factorisation with the  -- somewhat causal --
consideration that the missing values are part of the features, which
precede the response. The need to represent the response variable in the
factorization show that it may be useful to extend the traditional
mechanisms for a supervised learning setting: the link between the
mechanism and the output variable can have a significant impact on the
results. \cite{lecturenote} and \cite{arel2018can} noticed that as long
as $\M$ does not depend on $Y$, it is possible to estimate regression
coefficients without bias even with listwise deletion and MNAR values. 
\citet{ding2010investigation} generalise the MAR assumption with the following nomenclature MXY: the missing mechanism can marginally depend on the target (**Y), on the features that are always observed (*X*) or on the features that can be missing (M**).



\subsubsection{Imputation prior to analysis}\label{sec:imputeprior}

Most statistical models and machine-learning procedures are not designed
for incomplete data. To use existing pipelines in the presence on missing
values,
\emph{imputing} the data is commonly used to form a completed data set.
To (single) impute data,
\emph{joint modeling} (JM) approaches
capture the joint distribution
across features \citep{little2002statistical}. 
A simple example of joint modeling imputation is to assume a
Gaussian distribution of the data and to estimate the mean vector and
covariance matrix from the incomplete data (using an EM algorithm, see Section \ref{sec:em}). 
Missing entries can then be
imputed with their conditional expectation knowing the
observed data and the estimated parameters. More powerful methods can be
based on low-rank models \citep{hastie2015matrix, josse2016denoiser}.  While KNN imputation is not intrinsically suited for high-dimensional data, such methods often yield  good predictive performances 
\citep[see, e.g.,][]{batista2003analysis, poulos2018missing}.
\modif{Recently, several deep learning (DL) architectures have been proposed to impute datasets, based on variational autoencoder \citep[][]{MIWAE}, GAN such as GAIN \citep[][]{yoon2018gain} or MisGAN \citep[][]{li2019misgan}, or on denoising autoencoders as MIDA \citep[][]{vincent2008extracting,gondara2018mida} just to name a few. 
The quality of an imputation strategy is usually assessed via its RMSE performance: several observed values are first hidden, then the imputed values are compared to the true ones in terms of RMSE.  
While such a protocol is easy to implement, it has been recently criticized \citep[see, e.g.,][]{naf2023imputation} as the optimal imputation is then the conditional mean, therefore reducing the variability of the distribution of imputed data, compared to the actual true distribution. While DL architectures show promising performances for imputation in terms of
RMSE, DL
architectures fails to provide a correct imputation in terms of distributional statistics \citep[mean, variance, correlation, see][]{wang2022deep}. A new line of research aims at designing imputation scores able to reflect how the distribution of imputed value is close to the actual distribution \citep[see, e.g., ][]{naf2023imputation}. 
In this paper, we are interested in the predictive performance of any methods (including imputation followed by a learning algorithm) capable of handling missing data. Therefore, the quality of imputation is not assessed on its own but through the predictive performance of subsequent learning algorithm applied on the imputed data set. }

\medskip 

Another class of popular approaches to impute data defines the joint
distribution implicitly by the conditional distributions of each variable.
These approaches are called \textrm{fully conditional
specification} (FCS) or \textrm{imputation with conditional equation}
(ICE) \citep{vanbuuren2018flexible}. 
This formulation is very flexible and can easily handle variables of a
different nature such as ordinal, categorical, numerical, etc, via a
separate model for each, \emph{e.g.} using supervised learning.
Well-known examples of such approach are {\tt missForest},
using iterative imputation of each variable by random
forests \citep{stekhoven2011missforest}, or {\tt MICE} \citep{buuren2010mice}.
Their computational scalability however prohibits their use on large
dataset. As they fit one
model per feature, their cost is at least $O(d^2)$: using random forests
as a base model leads to a complexity   $O(d^2 n \log n)$ and using a ridge model amounts to a complexity  $O(d^2 n
\min (n, d))$.

The role of the dependent variable $Y$
and whether or not to include it in the imputation model has been
a rather controversial point. Indeed, it is quite counter-intuitive to
include it when the aim is to apply a conditional model on the imputed
data set to predict the outcome $Y$. Nevertheless, it is recommended as
it can provide information for imputing covariates
\citep[][p.57]{allison2001missing}.  \cite{Sterneb2393} illustrated the
point for the simple case of a bivariate Gaussian data ($X, Y$) with a
positive structure of correlation and missing values on $X$. Imputing
using only $X$ is not appropriate when the aim is to estimate the parameters of the linear regression model of $Y$ given $X$.  

One important issue with ``single'' imputation, \textit{i.e.} predicting
only one value for each missing entries,  is that it forgets that some values were missing and considers imputed values and observed values in the same way. It leads to underestimation of the variance of the parameters \citep{little2002statistical} estimated on the completed data. One solution, to incorporate the uncertainty of the imputed values is to use multiple imputation (MI, \citealt{rubin:1987}) where many plausible values are generated for each missing entries, leading to many imputed data sets. Then, MI consists in applying an analysis on each imputed data sets and combining the results. Although many procedures to generate multiple imputed data sets are available \citep{murray2018}, 
here again, the case of discriminative models is rarely considered,
with the exception of \cite{model_select_mice} who use a variable
selection procedure on each imputed data set and propose to keep the
variables selected in all imputed data sets to construct the final model
\citep[see also][]{mirl}. 

\subsubsection{EM algorithm}\label{sec:em}

Imputation leads to two-step methods that are generic
in the sense that any analysis can be performed from the same
imputed data set.
On the contrary, the  expectation maximization (EM) algorithm \citep{dempster1977maximum} proceeds
directly in one step. It can thus be better suited to a specific problem but requires the development of a dedicated algorithm. 

The EM algorithm can be used in missing data settings to compute maximum likelihood estimates from an incomplete data set. 
Indeed, with the assumptions of Theorem \ref{theo:mar} (MAR settings),
maximizing the observed likelihood $\mathcal L_2$  gives principle estimation of parameters $\theta$.
The log-likelihood of the observed data is
\begin{flalign*}
\log \mathcal{L}_2 (\theta)= \sum_{i=1}^n \log \int \density_\theta(o(\x_{i,o}, \x_{i,m}; \bm)) \diff \x_{i,m}.
\end{flalign*}
Starting from an initial parameter $\theta^{(0)}$, the algorithm alternates the two following steps, 
\begin{flalign*}
\text{\textbf{(E-step)}}&&
Q(\theta|\theta^{(t)}) &= 
\sum_{i=1}^n \int (\log \density_\theta(o(\x_{i,o}, \x_{i,m}; \bm))) \density_{\theta^{(t)}}(o(\x_{i,o}, \x_{i,m}; \bm)) \diff \x_{i,m}.
&& \\
\text{\textbf{(M-step)}}&&
\theta^{(t+1)} &\in \underset{\theta\in\Theta}{\mathrm{argmax}}~ Q(\theta|\theta^{(t)}).
&&
\end{flalign*}
The well-known property of the EM algorithm states that at each step $t$, the observed log-likelihood increases, although there is no guarantee to find the global maximum.
In Appendix \ref{sec:exem} we give an example of an EM algorithm to estimate the parameters of a bivariate Gaussian distribution from incomplete data. The interested reader can refer to \cite{roche2011algorithm} and the references therein for a theoretical review on the EM algorithm.

\section{Supervised learning procedures with missing data on train and test set } \label{sec:supervisedtraintest}

Supervised learning typically assumes that the data are i.i.d. In
particular, an out-of-sample observation (test set) is supposed to be drawn from the
same distribution as the original sample (train set). Hence, it must possess the same
missing-values mechanism. An appropriate method should then be able to predict
on new data with missing values. Here we discuss how to adapt classic
missing-values techniques to machine-learning settings, and vice versa.


\subsection{Out-of-sample imputation}
\label{subsec:imputation_train_test}

Using missing-value imputation in a supervised learning setting is
not straightforward as it 
 requires to impute new, out-of-sample, test data, where the target $Y$ is unavailable. 
 
A simple strategy is to fit an imputation model on the training set;  let
us consider a parametric imputation model governed by a parameter
$\alpha$\replace{, it yields $\hat\alpha$.   We denote, $\hat\X_{\mathrm{train}}$
the imputed training data set and}{. Based on the training set, we estimate a value $\hat{\alpha}$ and use this value to impute the training set, which is then denoted by $\hat{\X}_{\mathrm{train}}$. Then} a supervised-learning model is learned using
$\hat{\X}_{\mathrm{train}}$ and $Y_{\mathrm{train}}$. If the
supervised-learning procedure is indexed by a parameter $\beta$, it yields the estimated parameter $\hat\beta$. Finally, on the test set, the covariates
must be imputed with the same imputation model (using $\hat{ \alpha}$) and
the prediction is built using the imputed test set and the
estimated learning model (using $\hat{ \beta}$).

This approach is easy to implement for \emph{univariate
imputation}
methods that consider
each feature separately, for instance with mean imputation: parameters
$\hat \alpha$ correspond to the mean $\hat\mu_j$ of each column which is
learned on the training set, and new observations on the test set are imputed by $(\hat\mu_1,\ldots,\hat\mu_d)$.
This approach can also be implemented with a joint Gaussian model on 
$(\X,Y)$, learning parameters of the Gaussian with the EM algorithm on the
training set. Indeed, one can then impute the test set using the conditional expectations of the missing features given the observed features (without $Y$) and the estimated parameters.

For more general imputation methods, two issues hinder out-of-sample
imputation. First, many available imputation methods are ``black-boxes''
that take as input an incomplete data set and output a completed
data set: they do not separate the estimation of model parameters from
their use to complete the data. This is the case for many implementations
of iterative conditional imputation such as missForest
\citep{stekhoven2011missforest}, through scikit-learn \citep{scikit-learn} and  recent versions of MICE
\citep{vanbuuren2018flexible} (using an argument "ignore")
provides out-of-sample iterative conditional imputation. It is also difficult for powerful
imputers presented in Section \ref{sec:imputeprior} such as low-rank
matrix completion, which cannot be easily marginalised on $\X$ alone.

As most existing implementations cannot easily impute
a new data set with the same imputation model, some
analysts resort to
performing separate imputation of the training set and the test set.
But the smaller the test set, the more suboptimal this strategy is,
and it completely fails in the case where only one observation has to be
predicted. 
Another option is to consider semi-supervised settings, when the test
set is available at train time: 
grouped imputation can then
simultaneously impute the train and the test set
\citep{kapelner2015prediction}, while the predictive model is
subsequently learned on the training set only. 

\subsection{EM and out-of-sample prediction}\label{sec:empred}

The likelihood framework (Section \ref{sec:missingdata}) enables predicting
new observation, though it has not been much discussed.
\citet{jiang2018logistic} consider a special case of this approach for a logistic regression where covariates $\X$ are assumed to be Gaussian.

Let the assumptions of Theorem~\ref{theo:mar} be verified (MAR
settings). 
The true parameters of the model can then be estimated by maximizing
the observed log-likelihood $\log \mathcal{L}_2$ with an EM algorithm 
on the train data (Section \ref{sec:em}). The
corresponding estimates $\hat\theta_n$ can be used  for out-of-sample
prediction with missing values. 
The probability distribution of $y$ as a function of the observed values
$\x_o = \x_{obs(\bm)}$ only, can be related to that on a fully-observed data set:
\begin{align}
\density_{\hat \theta_n}(y | \bX_o = \x_{o})
&= \frac{\density_{\hat \theta_n}(y, \x_{o})}{\density_{\hat \theta_n}( \x_{o})}  \notag\\
& = \frac{1}{\density_{\hat \theta_n}(\x_{o})} \int \density_{\hat \theta_n}(y, o(\x_{o}, \x_{m}; \bm)) \diff\x_{m}  \notag\\
& =  \int \density_{\hat \theta_n}(y | o(\x_{o}, \x_{m}; \bm)) \frac{\density_{\hat \theta_n}( o(\x_{o}, \x_{m}; \bm))}{\density_{\hat \theta_n}(\x_{o})} \diff\x_{m}  \notag\\
&= \E_{\X_{m}|\X_{o}=\x_{o}}\left[\density_{\hat \theta_n}(y|o(\x_o, \X_{m}; \bm))\right]
\label{eq:em_imputation}
\end{align}
It is then possible to approximate the expectation with Monte Carlo
sampling from the distribution $\density_{\hat \theta_n}(\X_{\textrm{m}}|\X_{\textrm{o}}=\x_{\textrm{o}})$. Such a
sampling is easy in simple models, \emph{e.g.} using Schur's complements
for Gaussian distributions in linear regression settings. But in more
complex settings, such as logistic regression, there is no explicit solution and one option is to use  
Metropolis-Hasting algorithms \citep[][]{hastings1970monte}.

%
%



\subsection{Empirical risk minimization with missing data}

The two approaches discussed above are specifically designed
to fix the missing-values issue: imputing or specifying a parametric
model and computing the probability of the response given the observed
values. However, in supervised-learning settings, the goal is rather to build a prediction function that
minimizes an expected risk. Empirical risk minimization, the workhorse of 
machine learning, can be adapted to deal with missing data. 

Recall that in missing-values settings, we do not have access to $\X$ but rather to the incomplete vector $\bXm$. Therefore, we aim at minimizing the empirical risk over the set of measurable functions from $\widetilde{\mathcal X}$ to $\mathcal Y$, that is
\begin{align}
\widehat f_n \in \underset{f:\widetilde{\mathcal X} \to \mathcal Y}{\argmin}~
\frac{1}{n}\sum_{i=1}^{n} \ell \left(f(\bXm_i), Y_i\right).
\label{eq:emp_estimate_missing_data}
\end{align}

Unfortunately,  the half-discrete nature of $\widetilde{\mathcal X} =
\bigotimes_{j=1}^{d} (\mathcal X_j \cup \{\NA\})$ makes the problem
difficult. Indeed, many learning algorithms do not work with mixed data
types, such as $\R\cup\{\NA\}$, but rather require a vector space. This is
true in particular for gradient-based algorithms. As a result, the optimization 
problem (\ref{eq:emp_estimate_missing_data}) is hard to solve with typical
learning tools.

Another point of view can be adopted for losses which leads to 
Bayes-optimal solutions \replace{such that}{defined as} 
$\regTilde (\bXm) = \E[Y | \bXm]$.
As there are at most $2^d$ admissible missing patterns, we can rewrite the Bayes estimate as
\begin{equation}\label{eq:2d_mod}
  \regTilde (\bXm) =\sum_{\m\in\{0,1\}^d}
\E\left[Y\middle|\X_{obs(\bm)}, \M=\m\right] ~ \mathds 1_{\M=\m},
\end{equation}
This formulation highlights the combinatorial issues: solving (\ref{eq:emp_estimate_missing_data}) may require, as suggested by \citet[Appendix B]{rosenbaum1984reducing}, to estimate $2^d$ different submodels, that is $\E\left[Y\middle|\X_{obs(\bm)}, \M=\m\right]$ appearing in (\ref{eq:2d_mod}) for each $ \m \in \{0,1\}^d$, which grows exponentially with the number of variables. 

Modifying existing algorithms or creating new ones to deal with the optimization problem  (\ref{eq:emp_estimate_missing_data}) is in general a difficult task, due to the numerous possible missing data patterns. Nevertheless, we will see in  \modif{Section \ref{sec:simu}} and Appendix \ref{sec:trees}  that decision trees are particularly well suited to address this problem. 

\begin{remark}

Note that, in practice, not all patterns may be possible in the training
and test sets. For instance, if there are only complete data in the train
set, the only
submodel of interest is $\E\left[Y\middle|\X_{obs(\bm)}, \M=\m\right]$ for
$ \m = (0,\ldots,0)$, which boils down to the regular supervised-learning
scenario on a complete data set. 
However, the train and test sets are assumed to be drawn from
the same data distribution. Hence, we expect to observe similar patterns
of missingness in train and test sets. If this is not the case, we are in presence of a distributional shift, which should be tackled with
dedicated methods \citep[see, e.g.,][]{sugiyama2017dataset}.
This may happen for instance, when a study conducted on past data leads to operational recommendations, advising practitioners to focus on certain variables of interest. In that case, they will more likely measure them systematically.
\end{remark}




\section{Consistency of imputation procedures} \label{sec:Bayesrisk}

In this section, we show theoretically that, without assuming any parametric distribution for the data, \replace{single}{} imputation procedures can lead to
a Bayes-optimal predictor in the presence of missing
data on covariates (in both train and test sets), \emph{i.e.} it asymptotically targets the function
$\regTilde (\bXm) = \E[Y | \bXm]$.

In Section~\ref{sec:plugin}, we assume that we are given the true
regression function $f^{\star}$ on the fully-observed data and study the performance of applying this regression function to a test set with missing values, using several imputation
strategies: unconditional mean, conditional mean and multiple imputation.
\modif{Note that, for MCAR data, 
the function $f^{\star}$ can be estimated using the complete
observations only, i.e., by deleting observations with missing values in
the train set and applying a supervised procedure on the remaining
observations. Such a strategy is relevant for very large training sets
and MCAR missing values.}



In Section~\ref{sec:true_mean_imput}, we consider the full problem of tackling missing values in the train and the
test set, which is of particular interest when the training set is of
reasonable size as it can leverage the information contained in
incomplete observations. 
We study a classical approach, described in Section
\ref{subsec:imputation_train_test}, which consists in imputing the
training set, fitting a learning algorithm on the imputed data, and predicting on a test set
which has been imputed with the same method.
Although mean imputation of variables is one of the most widely used
approaches, it is highly criticised in the classic literature for missing
data \citep{little2002statistical}. Indeed, it
leads to a distortion of the data distribution
and, consequently, statistics calculated on the imputed data table
are biased.
A simple example is the correlation coefficient between two variables, which is biased towards zero if the missing data are imputed by the mean.
However, in a supervised-learning setting, the aim is not to
compute statistics representative of the data set, but to minimize a
prediction risk by estimating a
regression function. For this purpose, we show in Section
\ref{sec:true_mean_imput} that \replace{mean}{constant} imputation may be completely
appropriate and leads to consistent estimation of the prediction function. This result is remarkable and extremely useful in practice.

\subsection{Test-time imputation}\label{sec:plugin}

\replace{}{Here we consider that we have access to the optimal (Bayes) predictor $f^{\star}$
for the complete data, \emph{i.e.} $f^{\star}(\bX) = \E[Y|\bX]$, and we show that, in MAR settings,
when missing data appears in the test set, the optimal predictor for incomplete data can be computed by using multiple imputation with $f^{\star}$.}

\subsubsection{Test-time conditional multiple imputation is consistent}\label{sec:testtimeMI}

Let us first make explicit the multiple imputation procedure for
prediction. Recall that we observe  $\x_{obs(\bm)}$.
We then draw the missing values $\X_{m}$ from the conditional distribution $\X_{m} | \bX_{\textrm{o}} = \bx_{\textrm{o}}$ and compute the regression function on
these completed observations. The resulting multiple imputation function is given by:  
\begin{align}
    \regMultImput (\bxm) = \E_{{\X_{m}|\X_{o}=\x_{o}}} [ f^{\star} (o(\x_o, \X_{m}; \bm))]. \label{def:multiple_imputation}
\end{align}

Note that this expression is
similar to the expression \eqref{eq:em_imputation} given for EM, but assuming that we know the true nonparametric distribution of the data. 


\begin{assumption}[Regression model]
\label{reg_assumption1}
The regression model satisfies $Y = f^{\star}(\bX) + \varepsilon,$ where $\bX$ takes values in $\R^d$ and $\varepsilon$ satisfies a.s. $\E [ \varepsilon | \bX_{obs(\M)}] = 0$. 
\end{assumption}

\begin{assumption}[Missingness pattern - MAR-$Y$]
\label{miss_assumption1}
\replace{The first $0 < p < d$ variables in $\bX$ are always observed and the distribution of $M$ depends only on these observed values. }{We have $Y \indep \bM | X_{obs(\bM)}$}
\end{assumption}

\replace{}{The missingness pattern in Assumption~\ref{miss_assumption1} is more generic than MAR, since it is notably verified if the missingness pattern is MAR, that is if $\mathbb{P}[\bM = \bm | \bX] = \mathbb{P}[\bM = \bm | \bX_{obs(\bm)}]$ (classic definition recalled in Section~\ref{sec:missingdata}), \textit{i.e.} if the probability to observe a given pattern depends only on the observed values. }

\begin{theorem}
\label{theo:multiple_imputation}
Grant Assumption~\ref{reg_assumption1} and \ref{miss_assumption1}. Then the multiple imputation procedure, defined in  \eqref{def:multiple_imputation} is Bayes optimal, that is, for all $\bxm \in (\R \cup \{\texttt{NA} \})^d$,  
\begin{align*}
 \regMultImput (\bxm) = \regTilde (\bxm).
\end{align*}
\end{theorem}

The proof is given in Appendix \ref{apx:split}. Theorem \ref{theo:multiple_imputation} justifies the use of multiple imputation when missing pattern is MAR and when we have access to an estimate of $f^{\star}$ and of the conditional distribution $\bX_{m}|\bX_{o} = \bx_{o}$.

\subsubsection{Single mean imputation is not consistent}

Given the success of multiple imputation, it is worth checking that
single imputation is not sufficient. We show with 
two simple examples
that indeed, single imputation on the test set is not consistent even in
MAR setting.

We first show, that (unconditional) mean
imputation is not consistent, if the learning algorithm has been trained
on the complete cases only.
\begin{example}
\label{example1}
In one dimension, consider the following simple example, 
\begin{equation*}
    X_1 \sim \mathcal \mathcal U(0, 1),
    \quad
    Y = X_1^2 + \varepsilon,
    \quad
    M_1 \sim \mathcal B(1/2) \indep\ (X_1, Y) ,
\end{equation*}
with $\varepsilon$ an independent centered Gaussian noise.
Here, $\E[Y|X_1] = X_1^2$, and the regression function $\regTilde (\bXm) = \E[Y | \bXm]$ satisfies
\begin{align}
\regTilde (\bXm)
&= X_1^2\cdot\mathds 1_{M_1=0}
+ \E[Y|\Xm= \NA]\cdot\mathds 1_{M_1=1} \nonumber \\
&= X_1^2\cdot\mathds 1_{M_1=0}
+ \E[X_1^2]\cdot\mathds 1_{M_1=1} \nonumber \\
&= X_1^2\cdot\mathds 1_{M_1=0}
+ (1/3)\cdot\mathds 1_{M_1=1}. \label{eq:mcar}
\end{align}
In the oracle setting where the distribution of $(X_1, Y, M_1)$ is known, "plugging in" the mean imputation of $X_1$ yields the Plug-in Imputation function prediction 
\begin{align}
f_{PI}(\bXm) & = X_1^2\cdot\mathds 1_{M_1=0}
+ (\E[X_1])^2\cdot\mathds 1_{M_1=1} \nonumber \\
 & = X_1^2\cdot\mathds 1_{M_1=0}
+ (1/4) \cdot\mathds 1_{M_1=1}. \label{eq:mcar2}
\end{align}
\end{example}
In this example, mean imputation is not optimal: when $X_1$ is missing, the prediction obtained by mean imputation is $1/4$, whereas the optimal prediction (the one which minimizes the square loss) is $1/3$ as seen in (\ref{eq:mcar}). 

Inspecting (\ref{eq:mcar}) and (\ref{eq:mcar2}) reveals that
the poor performance of mean imputation is due to the fact that
$\E[X_1^2] \neq (\E[X_1])^2$. The non-linear relation between $Y$ and
$X_1$ breaks mean imputation. 
This highlights the fact that the imputation method should be chosen in
accordance with the learning algorithm that will be applied later on.
This is related to the concept of congeniality \citep{meng1994multiple}
defined in multiple imputation.

\medskip

\subsubsection{Conditional mean imputation is consistent if there are deterministic relations between input variables}

We now consider conditional mean imputation, using information of
other observed variables to impute. Conditional mean imputation may work
in situations where there is redundancy between variables, as highlighted
in Example \ref{example2}. However, we give a simple example below stressing that using it to impute the test may not be Bayes optimal. 

\begin{example}
\label{example2}
Consider the following regression problem with two identical input variables:
\begin{equation*}
X_1  = X_2 \sim \mathcal U([0,1]),
\quad
Y = X_1 + X_2^2 + \varepsilon, 
\quad
M_2 \sim \mathcal B(1/2) \indep (X_1, X_2, Y)
\end{equation*}
The Bayes predictor is then given by 
\begin{align*}
\regTilde (\bXm)
&= \left\{
\begin{array}{ll}
X_1 + X_2^2
& \mbox{if } \Xm_2\neq \NA \\
X_1 + \E[X_2^2|X_1, \Xm_2=\NA]
& \mbox{if } \Xm_2=\NA
\end{array}
\right. \\
&= \left\{
\begin{array}{ll}
X_1 + X_2^2
& \mbox{if } \Xm_2\neq \NA \\
X_1 + X_1^2
& \mbox{if } \Xm_2=\NA
\end{array}
\right.
\end{align*}
\replace{Imputation with the mean of $X_2$ in this function leads to}{Plugging in the previous expression the mean of $X_2$ when missing leads to the Plug-in Imputation function} 
\begin{align*}
f_{PI}(\bXm) &= \left\{
\begin{array}{ll}
X_1 + X_2^2
& \mbox{if } \Xm_2\neq \NA \\
X_1 + (1/4)
& \mbox{if } \Xm_2=\NA,
\end{array}
\right. 
\end{align*} 
whereas plugging in the Mean Imputation of $X_2$ conditional on $X_1$ leads to 
\begin{align*}
f_{PMI}(\bXm)  &= \left\{
\begin{array}{ll}
X_1 + X_2^2
& \mbox{if } \Xm_2\neq \NA \\
X_1 + X_1^2
& \mbox{if } \Xm_2=\NA
\end{array}
\right. ,
\end{align*} 
as $(\E[X_2|X_1])^2 = X_1^2$.
\end{example}


If there is no deterministic link between variables, conditional mean imputation fails to recover the regression function, in the case where the regression function is not linear (see Example \ref{example2}, where $X_1 = X_2$ is replaced by $X_1 = X_2 + \varepsilon$).

\subsubsection{Pathological case: missingness is a covariate}


Example \ref{example_mnar2} below shows a situation in which any imputation method, single or multiple, fails, because the missingness contains information about the response variable $Y$. In this univariate setting, there is no distinction between conditional and unconditional mean.

\begin{example}\label{example_mnar2}
Consider the following regression model, 
\begin{equation*}
    X_1 \sim \mathcal U(0, 1) \qquad
    M_1 \sim \mathcal B(1/2) \indep X_1 \qquad
    Y = X_1 \cdot \mathds 1_{M_1=0} + 3X_1 \cdot \mathds 1_{M_1=1} + \varepsilon.
\end{equation*}
\begin{flalign*}
\text{Here,}&&
\E[Y | X_1] = X_1\cdot\mathbb P(M_1=0) + 3X_1\cdot\mathbb P(M_1=1) = 2X_1\, . 
&&&
\end{flalign*}
Plugging in the mean imputation of $X_1$ leads to 
\begin{align*}
\replace{\regTilde(\Xm)}{f_{PI}(\Xm)} & = X_1\cdot\mathds 1_{M_1=0} + \E[X_1] \cdot\mathds 1_{M_1=1}\\
& = X_1\cdot\mathds 1_{M_1=0} + (1/2) \cdot\mathds 1_{M_1=1},
\end{align*}
whereas the regression function satisfies
\begin{align*}
\regTilde(\Xm) & = X_1\cdot\mathds 1_{M_1=0} + 3 \E[X_1|\Xm = \NA] \cdot\mathds 1_{M_1=1}\\
& = X_1\cdot\mathds 1_{M_1=0} + (3/2) \cdot\mathds 1_{M_1=1}.
\end{align*}
\end{example}

In this case, the presence of missing values is informative in itself, and having access to the complete data set (all values of $X_1$) does not provide enough information.
Such a scenario advocates for considering the missingness as an additional
input variable.
Indeed, in such situations, single and multiple imputation fail to recover
the targeted regression function, without adding a missingness indicator to the input variables. 


\subsection{Constant imputation at train and test time is consistent}
\label{sec:true_mean_imput}

We now show that the same single imputation used in both train and test sets leads to consistent procedures. More precisely, we allow missing data on $X_1$ only, and replace its value by some constant $\impConst \in \R$ if $X_1$ is missing. More precisely, for each observed $\tilde{\x} \in (\R \cup \{\texttt{NA}\}) \times \R^{d-1}$, the imputed entry is defined as $\x' = (x_1', x_2, \hdots, x_d)$ where 
\begin{flalign*}
\text{\textbf{(constant imputation)}}&&
x_1' = x_1 \mathds{1}_{M_1=0} + \impConst \mathds{1}_{M_1=1}.&&
\end{flalign*}
We consider the following procedure: $(i)$ impute the missing values on $X_1$ in the training set by $\impConst$ $(ii)$ use a universally consistent algorithm (see the definition in Section~\ref{sec:definitions}) on the training set to approach the regression function $\regPrime (\x') = \E [Y|\bX' = \bx']$. Theorem~\ref{thm:imp_mean_valid} shows that this procedure is consistent under the following assumptions.


\begin{assumption}[Regression model]
\label{reg_assumption2}
Let $ Y = f^{\star}(\bX) + \varepsilon$ where $\bX$ has a continuous density $g>0$ on $[0,1]^d$,  $f^{\star}$ is continuous, and $\varepsilon$ is a centred noise independent of $(\bX, M_1)$.
\end{assumption}

\begin{assumption}[Missingness pattern - MAR]
\label{miss_assumption2}
The variables $X_2, \hdots, X_d$ are fully observed and  the missingness pattern $M_1$ on
variable $X_1$ satisfies $M_1 \indep X_1 | X_2, $ $\hdots, X_d$ and is such that the function $(x_2, \hdots, x_d) \mapsto \P[M_1 = 1|X_2=x_2, \hdots, X_d = x_d]$ is continuous. 
\end{assumption}

As for Assumption~\ref{miss_assumption1}, Assumption~\ref{miss_assumption2} states that the missingness pattern is a specific MAR process since only $X_1$ can be missing with a probability that depends only on the other variables, which are always  observed. The conditional distribution of $M$ is also assumed to be continuous to avoid technical complexities in the proof of Theorem~\ref{thm:imp_mean_valid}.

\begin{theorem}
\label{thm:imp_mean_valid}
Grant Assumption~\ref{reg_assumption2} and \ref{miss_assumption2}.
The single imputation procedure described above satisfies, for all imputed entries $\bx' \in \R^d$, 
\begin{align*}
    \regPrime (\x') & = \E[Y|X_2 = x_2, \hdots, X_d = x_d, M_1 = 1] \mathds{1}_{x_1' = \impConst} \mathds{1}_{\P[M_1 =1| X_2 = x_2, \hdots, X_d = x_d] >0} \\
    & \quad + \E[Y|\X = \x']  \mathds{1}_{x_1' = \impConst} \mathds{1}_{\P[M_1 =1| X_2 = x_2, \hdots, X_d = x_d] =0} \\
    & \quad + \E[Y|\X = \x', M_1 = 0] \mathds{1}_{x_1' \neq \impConst}.
\end{align*}
\begin{flalign*}
\text{Consequently, letting}&&
    \bXm = \left\lbrace 
    \begin{array}{cc}
       \bX'  &  \textrm{if}~X_1' \neq \impConst \\
       (\texttt{NA}, X_2, \hdots, X_d)  & \textrm{if}~X_1' = \impConst
    \end{array}
    \right., &&&&
\end{flalign*}
the single imputation procedure is equal to the Bayes function almost everywhere, that is
\begin{align}
   \regPrime (\bX') & =  \regTilde (\bXm). \label{everywhere_equality}
\end{align}

\end{theorem}
The proof is given in Appendix \ref{apx:split}.
Theorem \ref{thm:imp_mean_valid}  confirms that it is preferable to
use
the same imputation for the training and test sets. Indeed, the
learning algorithm can learn the imputed value and use
that information to detect that the entry was initially missing. If the imputed value changes from train set to test set (for example, if instead of imputing the test set with the mean of the variables of the train set, one imputes by the mean of the variables on the test set), the learning algorithm may fail, since the imputed data distribution would differ between train and test sets.

\paragraph{Multivariate missingness.} 
Interestingly, Theorem~\ref{thm:imp_mean_valid} remains valid when missing values occur for variables $X_1, \hdots,$ $ X_j$ under the assumption that $(M_1, \hdots, M_j) \indep (X_1,$ $ \hdots, X_j)$ conditional on $(X_{j+1},$ $ \hdots, X_d)$ and if for every pattern $\m \in \{0,1\}^j \times \{0\}^{d-j}$, the functions $(x_{j+1}, \hdots, x_d) \mapsto \P[\M = \m|X_{j+1} = x_{j+1},  \hdots, X_d = x_d]$ are continuous.\\

Note that the precise imputed value $\impConst$ does not matter if the learning algorithm is universally consistent. By default, the mean is not a bad choice, as it preserves the first order statistic (mean) of the sample. \replace{}{However, our analysis does not stress out any particular role of mean imputation. In fact, any imputation in the support of $X_1$ is equivalent and leads to a consistent procedure.} The comment below stresses out the benefit of choosing $\impConst$ outside of the support of $X_1$.

\paragraph{Almost everywhere consistency.}
The equality between the constant imputation learner $\regPrime$ and the Bayes function $ \regTilde$ holds almost surely but not for every $\tilde{\bx}$. Indeed, under the setting of Theorem \ref{thm:imp_mean_valid}, let $\tilde{\bx} = (\impConst, x_2, \hdots, x_d)$, for any $x_2, \hdots, x_d \in [0,1]$ such that 
$\P[M_1 = 1|X_2 = x_2, \hdots, X_d = x_d]  >0.$ In this case, $\bx' = (\impConst, x_2, \hdots, x_d)$ and 
\begin{align*}
\regPrime (\bx') & = \E[Y|X_2 = x_2, \hdots, X_d = x_d, M_1 = 1],
\end{align*}
which is different, in general, from
\begin{align*}
\regTilde (\bxm) = \E[Y | X_1 = \impConst, X_2 = x_2, \hdots, X_d = x_d].
\end{align*}
Therefore, on the event $A_1 = \{ \tilde{\bX}, \tilde{X}_1 = \impConst\}$, the two functions $\regPrime$ and $\regTilde$ differ, and thus the equality between these functions does not hold pointwise. However, since $A_1$ is a zero probability event, the equality  
$\regPrime(\bX') =  \regTilde (\bXm)$
hold almost everywhere, as stated in Theorem~\ref{thm:imp_mean_valid}. A
simple way to obtain the pointwise equality in equation
\eqref{everywhere_equality} is to impute missing entries by values that are out of the range of the true distribution, which echoes the "separate class" method advocated by \citet{ding2010investigation}.

\paragraph{Discrete/categorical variables.} According to Assumption~\ref{reg_assumption2}, the variables $X_1, \hdots , X_d$ admit a density. Therefore, the proof of  Theorem~\ref{thm:imp_mean_valid} is not valid for discrete variables. However, Theorem~\ref{thm:imp_mean_valid} can be extended to handle discrete variables, if missing entries in each variable are imputed by an extra category ``missing'' for each variable. In this framework, consistency boils down to estimating the expectation of $Y$ given a category which directly results from the universal consistency of the selected algorithm. 

\paragraph{Classification.} Theorem~\ref{thm:imp_mean_valid} is established for a regression problem. However, in a binary classification setting where $Y \in \{0,1\}$, the plug-in classifier $\hat{g}_n(\bx) = \mathds{1}_{\hat{\eta}_n(\bx) \geq 1/2}$ is universally consistent if the estimate $\hat{\eta}_n$ (an estimator of $\E[Y|\bX]$) is universally consistent, according to Corollary~6.2 in \citet{devroye2013probabilistic}.  Theorem~\ref{thm:imp_mean_valid} can be applied to consistently estimate $\E[Y|\bX=\bx]$ via $\hat{\eta}_n(\bx)$, which leads to a universally consistent plug-in classifier $\hat{g}_n$. Consequently, Theorem~\ref{thm:imp_mean_valid} is also valid for classification frameworks. 

\paragraph{Universal consistency.} In Theorem~\ref{thm:imp_mean_valid},
we assume to be given a universally consistent algorithm which may appear
as a strong restriction on the choice of the algorithm. However many
estimators exhibit this property as, for example, local averaging
estimate \citep[kernel methods, nearest neighbors and decision trees,
see][]{devroye2013probabilistic}. The key point of
Theorem~\ref{thm:imp_mean_valid} is to state that the universal consistency
of a procedure with missing values results directly from the universal
consistency of an algorithm applied to an imputed data set, therefore
providing guarantees that consistent algorithm and imputation are useful
tools to handle missing values.

\modif{
\paragraph{Optimizing the imputed values.}
\cite{bertsimas2018predictive}\footnote{Their work were submitted after our contribution.} propose a general framework to optimize the imputation values, and instantiate this framework for different learning algorithm. Extended simulations show that their proposal is competitive in terms of imputation quality (measured via MAE/RMSE) or predictive performances. }


\paragraph{Consistency for some specific distributions.} In
Theorem~\ref{thm:imp_mean_valid}, we assume to be given a universally
consistent algorithm. One can legitimately ask if the result still holds
if an algorithm which is consistent only for some specific data
distribution was used instead. For example, assume that data are
generated via a linear model and a linear estimator is used after missing
values have been imputed. One can show that the $2^d$ submodels are not
linear in general and consequently, using a single linear model on
imputed data does not yield the Bayes loss \citep{morvan2020linear}. In a
nutshell, Theorem~\ref{thm:imp_mean_valid} is not valid for procedures
that are consistent for some specific data distributions only. The
interested reader can refer to \cite{morvan2020linear} for further
details on missing values in linear generative models.

\section{Simulations}\label{sec:simu}

In Section~\ref{sec:supervisedtraintest}, we provided theoretical guarantees justifying the common practice of imputing by a constant, in a specific MAR setting and in an asymptotic sample regime. More precisely, Theorem~\ref{thm:imp_mean_valid} gives us a positive response in principle regarding the usage of constant imputation for prediction purpose but does not establish the predictive performances of such methods in a finite-sample regime. To go beyond the framework of Theorem~\ref{thm:imp_mean_valid}, we now evaluate the predictive performances of learning algorithms trained on imputed data sets (following the imputation prior-to-learning strategy of Theorem~\ref{thm:imp_mean_valid}) and, to give a broader perspective, we compare it to that of procedures  intrinsically able to handle missing values.
In general such procedures are hard to design, but decision trees offer natural approaches for empirical risk minimization with missing values \citep{saar2007handling,Twala:2008:GMC:1352941.1353228}. This is due to their ability to handle the half-discrete nature of $\tilde \X$, as they rely on greedy decisions rather than smooth optimization.

We first present the different approaches available to handle
missing values in tree-based methods in Section~\ref{subsec:presenting_decision_trees}. We then compare numerically the two different strategies: imputation prior to learning and learning directly with missing data.



\subsection{Dealing with missing values using decision trees}
\label{subsec:presenting_decision_trees}

There exist several approaches to build decision trees based on data sets containing missing data. These approaches can be divided into two parts: $(i)$ those for which the splitting mechanism omits missing data (namely Probabilistic splits, Block propagation and surrogate splits) and $(ii)$ those for which missing data are taken into account to build the splits (Missing incorporated in attributes, MIA). These approaches are detailed below. 

For the first set of approaches, in each cell, a splitting criterion  (see equation~\ref{split_criterion_cart} in Appendix~\ref{sec:trees}) is computed for each variable $j \in \{1, \hdots, d\}$, based on observations that have non-missing entries $X_j$. The best split $(j^{\star}, z^{\star})$ (corresponding to a split at position $z^{\star}$ along variable $j^{\star}$), is then chosen as the one optimizing this criterion (see Algorithm~\ref{alg-split-missdata} in Appendix~\ref{sec:availablecase} for details). Since such an approach omits the missing data for building the best split, one need to specify a method to send partially-observed data down the tree. The following approaches propose different strategies.

\paragraph{Surrogate splits}
Surrogate splits search for a split on another variable that induces a
data partition close to the original one. More precisely, for a selected
split $(j^{\star}_0, z^{\star}_0)$, to  send
down the tree observations with no $j^{\star}_0$th variable, a new stump,
\textit{i.e.}, a tree with one cut, is fitted to the response
$\mathds{1}_{X_{j^{\star}_0}\leq z^{\star}_0}$, using variables
$(X_j)_{j\neq j^{\star}_0}$. The split $(j^{\star}_1, z^{\star}_1)$ which
minimizes the misclassification error is selected, and observations are
split accordingly. Those that lack both variables $j^{\star}_0$ and
$j^{\star}_1$ are split with the second best, $j^{\star}_2$, and so on
until the proposed split has a worse misclassification error than the
blind rule of sending all remaining missing values to the same daughter,
the most populated one. To predict, the training surrogates are kept (see Algorithm~\ref{alg-split-surrogate} in Appendix~\ref{sec:availablecase} for details).
This construction is  the default method in {\tt rpart} \citep{therneau1997introduction}.
Surrogate method is expected to be appropriate when there are
relationships between covariates.

\paragraph{Probabilistic splits}
Another option is to propagate missing observations according to a Bernoulli distribution $\mathcal B(\frac{\#L}{\#L+\#R})$, where $\#L$ (resp. $\#R$) is the number of points already on the left (resp. right) (see Algorithm~\ref{alg-probabilistic-split} in Appendix~\ref{subsec:app:algo}). This is the default method in C4.5 \citep{quinlan2014c4}.

\paragraph{Block propagation} A third option is to choose the split on the
observed values, and then send all incomplete observations as a block, to
a side chosen by minimizing the error (see
Algorithm~\ref{alg-block-propagation} in Appendix~\ref{subsec:app:algo}).
This is the method used in LightGBM \citep{ke2017lightgbm}.\\

\paragraph{Missing incorporated in attribute (MIA, \citealt{Twala:2008:GMC:1352941.1353228})}
A second class of methods uses missing values to compute the
splitting criterion itself. Consequently, the splitting location depends
on the missing values, contrary to all methods presented above. Its most common instance is  ``missing incorporated in attribute" (MIA), which considers the following splits, for all splits $(j,z)$: 
\begin{itemize}
\item $\{\Xm_j\leq z ~\textrm{or}~  \Xm_j=\NA\}$ vs $ \{\Xm_j>z\}$,
    
\item $\{\Xm_j\leq z\}$ vs $\{\Xm_j>z ~\textrm{or}~ \Xm_j=\NA\}$,
    
\item  $\{\Xm_j\neq\NA\}$ vs $\{\Xm_j=\NA\}$.
\end{itemize}
In a nutshell, MIA tries to send all missing
values to the left or to the right for each possible split, or to separate observed values from missing ones. For each option, the prediction error is computed and the selected option is the one minimizing the prediction error  (see Algorithm~\ref{alg-MIA}  in Appendix~\ref{subsec:app:algo} for details). 

Missing values are treated as a category by MIA, which is thus nothing but a greedy algorithm
minimizing the square loss between $Y$ and a function of $\bXm$ and consequently targets
the quantity \eqref{eq:2d_mod} which separate $\E[Y |\bXm]$  into $2^d$ terms. However, it is not exhaustive: at each step,
the tree can cut for each variable according to missing or non missing and selects this cut when it is relevant, \textit{i.e.} when it minimizes the prediction error. The final leaves can correspond to a cluster of missing values patterns (observations with missing values on the two first variables for instance and any missing patterns for the other variables).

MIA is thought to be a good method to apply when missing pattern is
informative, as this procedure allows to cut with respect to missing/non
missing and uses missing data to compute the best splits. Note this
latter property implies that the MIA approach does not require a
different method to propagate missing data down the tree. Notably, MIA is
implemented in the R packages \texttt{partykit} \citep{partykit} and
\texttt{grf}  \citep{grf}, as well as in XGBoost \citep{chen2016xgboost} and 
for the {\tt HistGradientBoosting} models in scikit-learn \citep{scikit-learn}.

In Appendix~\ref{sec:comparbre}, we conduct a theoretical analysis on a very simple regression models to highlight differences between the above strategies. In particular, we compare in Proposition~\ref{thm:comp_tree_theo} the risk of the different splitting strategies (probabilistic split, block propagation, surrogate split, and MIA) and prove that MIA and surrogate splits are the two best strategies, one of which may be better than the other depending on the dependence structure of  covariates.
Note that block propagation can be seen as a greedy way of successively optimizing the choices in the two first options in MIA. However, as we show in 
Proposition~\ref{prop_criterion}, these successive choices are sub-optimal.

\subsection{Simulation settings}\label{sec:setting}
We consider three regression models with  covariates $(X_1, \hdots, X_d)$   distributed as $\mathcal{N}(\mu, \Sigma)$ with $\mu = \mathbf{1}_d$ and $\Sigma = \rho \mathbf{1}_d \mathbf{1}_d^T + (1 - \rho) I_d$,  where $\mathbf{1}_d$ is the $d$-dimensional vector composed of ones and $I_d$ the $d \times d$ identity matrix. \modif{By default, $\rho$ is set to $0.5$ in our experiments.}  
The first model is quadratic, the second one is linear, and the third one has been used as a benchmark for
tree methods by several authors, including
\citet{friedman1991multivariate} and \citet{breiman1996bagging}.
We also consider a last regression model where the relationship between
covariates  are nonlinear. In all four models, $\varepsilon$ is a
centered Gaussian noise with standard deviation 0.1.

\begin{model}[Quadratic]
\label{model1} $ Y = X_1^2 + X_2^2 + X_3^2 + \varepsilon$

\end{model} 

\begin{model}[Linear]
\label{model2} $ Y = X\beta + \varepsilon$
with 
$\beta=(1, 2, -1, 3, -0.5, -1, 0.3, 1.7, 0.4, -0.3).$
\end{model}

\begin{model}[Friedman]
\label{model3}
 $Y = 10\sin(\pi X_1X_2) + 20(X_3-0.5)^2 + 10X_4 + 5X_5 + \varepsilon$
\end{model}

\begin{model}[Friedman, Nonlinear]
\label{model4}
$ Y = \sin(\pi X_1X_2) + 2(X_3-0.5)^2 + X_4 + .5X_5 + \varepsilon$
where $X$ is a hidden uniform variable on $[-3, 0]$ and the 
covariates $(X_1, \hdots, X_d)$  are distributed as 
\begin{align*}
\left\lbrace
\begin{array}{ll}
    X_1 &= X^2 +\varepsilon_1\\
  X_2 &= \sin(X) +\varepsilon_2\\
  X_3 &= \tanh(X)\exp(X)\sin(X)+\varepsilon_3\\
 X_4 &= \sin(X-1)+\cos(X-3)^3+\varepsilon_4\\
  X_5 &= (1-X)^3+\varepsilon_5
\end{array}
\right.
& \quad 
\left\lbrace
\begin{array}{ll}
  X_6 &= \sqrt{\sin(X^2)+2}+\varepsilon_6\\
   X_7 &= X-3+\varepsilon_7\\
  X_8 &= (1-X)\sin(X)\cosh(X)+\varepsilon_8\\
  X_9 &= \frac{1}{\sin(2X)-2}+\varepsilon_9\\
  X_{10} &= X^4+\varepsilon_{10}
\end{array}
\right. , 
\end{align*}
where $\varepsilon_i$ are independent centered Gaussian with standard deviation 0.05. 
\end{model}

In our experiments, we generate missing data as follows.

\begin{missing_pattern}[MCAR]
For $p \in [0,1]$,  the missingness on variable $j$ is generated according to a Bernoulli distribution
\begin{equation*}
\forall i\in\llbracket 1,n\rrbracket, 
M_{i,j} \sim \mathcal B(p).
\end{equation*}
\end{missing_pattern}

\begin{missing_pattern}[Censoring MNAR]
A direct way to select a proportion $p$ of missing values on a variable
$X_j$, that depends on the underlying value, is to crop them above the $1-p$-th quantile
\begin{equation*}
\forall i\in\llbracket 1,n\rrbracket, 
M_{i,j} = \mathds 1_{X_{i,j}>[X_j]_{(1-p)n}}.
\end{equation*}
\end{missing_pattern}

\begin{missing_pattern}[Predictive Missingness]
Last, we consider a pattern mixture model, letting $M_j$ be part of
the regression function, with $M_j \indep\X$,
\begin{equation*}
\forall i\in\llbracket 1,n\rrbracket, M_{i,j} \sim \mathcal B(p),
\quad\text{and}\quad
    Y = \sum_{j=1}^{3} \bigl(X_j^2 + 2\,M_j \bigr) + \varepsilon.
\end{equation*}
\end{missing_pattern}



%
%
%

%


\modif{MCAR is the most simple instance of missing patterns. In particular, it falls in the framework of Theorem~\ref{thm:imp_mean_valid}, for which constant imputation leads to consistent predictive methods. MNAR is a more complex pattern where the missingness indicator depends on the true value of the variable. For such a pattern, adding the mask can help the prediction, but using the dependencies between the missing covariate and the existing ones can help to build an accurate prediction. The third missing pattern is, in some sense, pathological (see Example~\ref{example_mnar2}) as the output depends directly on the missingness indicator. In such extreme settings, good performances can only be obtained if the mask is added as input variable.}

We compare the following methods using implementation in R \citep{softwR} and default values for the tuning parameters. \modif{We compare tree-based methods (decision trees, random forests, gradient boosting) able to handle directly missing values, with more classic machine learning methods like Support Vector Machines \citep[SVM, see][]{cortes1995support} or K nearest neighbours \citep[KNN, see][]{cover1967nearest}. }
\modif{Unless stated otherwise, we use the following packages: rpart \citep{rpartpack} for decision trees, ranger \citep{rangerpack} for random forests,  XGBoost \citep{chen2016xgboost} for gradient boosting, e1071 for SVM and caret for KNN}. Note that we have used surrogate splits only with decision trees. \modif{More precisely, we compare for decision trees the following strategies:}
\begin{itemize}[leftmargin=3ex, itemsep=.2ex, parsep=.1ex, topsep=.1ex]

\item \textbf{MIA}: missing in attributes (see Remark~\ref{rem:mia} in Appendix~\ref{seq:MIA} for implementation details)
\item \textbf{rpart+mask}/ \textbf{rpart}: CART with surrogate splits, with or without  the indicator $\M$ in the covariates
\item \textbf{ctree+mask}/ \textbf{ctree}: conditional trees, implemented in package partykit \citep{partykit} with or without the indicator $\M$ in the covariates
\item \textbf{impute mean+mask}/ \textbf{impute mean}: missing values are imputed by unconditional mean with or without the indicator $\M$ added in the covariates 
\item \textbf{impute OOR+mask} / \textbf{impute OOR}: missing
values are imputed by a constant value, chosen out of range (OOR) from
the values of the corresponding covariate in the training set, with or without the indicator $\M$ added in the covariates
\item \textbf{impute Gaussian+mask} /\textbf{impute Gaussian}: missing values are imputed by conditional expectation when data are assumed to follow a Gaussian multivariate distribution. More precisely, the parameters of the Gaussian distribution are estimated with an EM algorithm (R package norm \citep{fox2013package}). Note that for numerical reasons, we shrink the estimated covariance matrix (replacing $\hat \Sigma$ by $0.99\times\hat \Sigma + 0.01\times\rm{tr}(\hat \Sigma)I_d$) before imputing.  The method can also be applied with the indicator $\M$ added in the imputed values.
\end{itemize}

\modif{For KNN and SVM, we compare the last three methods, namely mean imputation, Out-Of-Range (OOR) imputation, Gaussian imputation (as other methods do not apply), with and without adding the mask, whereas for random forests and gradient boosting methods, we also use MIA.}

To evaluate the performance of the methods, we repeatedly draw a
training set and a testing set of the same size
1000 times. \modif{We choose to display the percentage of explained
variance, \emph{i.e.} the $R^2$ statistic computed on the test set, defined as 
\begin{align}
    R^2 = 1 - \frac{\frac{1}{n_{\textrm{test}}} \sum_{i=1}^{n_{\textrm{test}}} (\hat{y}_i - y_i)^2}{\mathds{V}[Y]},
\end{align}

The numerator corresponds to the Mean Square Error (MSE) computed on the test set, and the denominator is the true variance of $Y$, available since we know the generative model. Using the $R^2$ metric allows us to quickly assess the quality of a regression model: values close to one indicate a very good predictive model. Note that, depending on the regression model and the missing mechanism at hand, some predictive tasks are easier than other, thus explaining the differences in the average predictive performance of all methods. }
The code for these experiments is available
online\footnote{\url{https://github.com/jacobmchen/supervised_missing/}}. 

\paragraph{Experiment 1}
In the first experiment, we use Model \ref{model1} with $d=9$ and
introduce missing values on $X_1$, $X_2$, and $X_3$  according to the mechanisms described hereafter. Results are depicted in Figure \ref{fig:simu_mcar_nonlinear_multivar_miss2_rho5}.

\paragraph{Experiment 2}
In the second experiment, the other three models are used with $d=10$, with a MCAR mechanism on all variables. Results are shown in Figure \ref{fig:simu_mcar_nonlinear_multivar2}.

\subsection{Results comparing strategies for fixed sample sizes}
\label{sec:simu_finding_the_best_tree}




\begin{figure}[h!]
    \centering
    \includegraphics[scale=0.45]{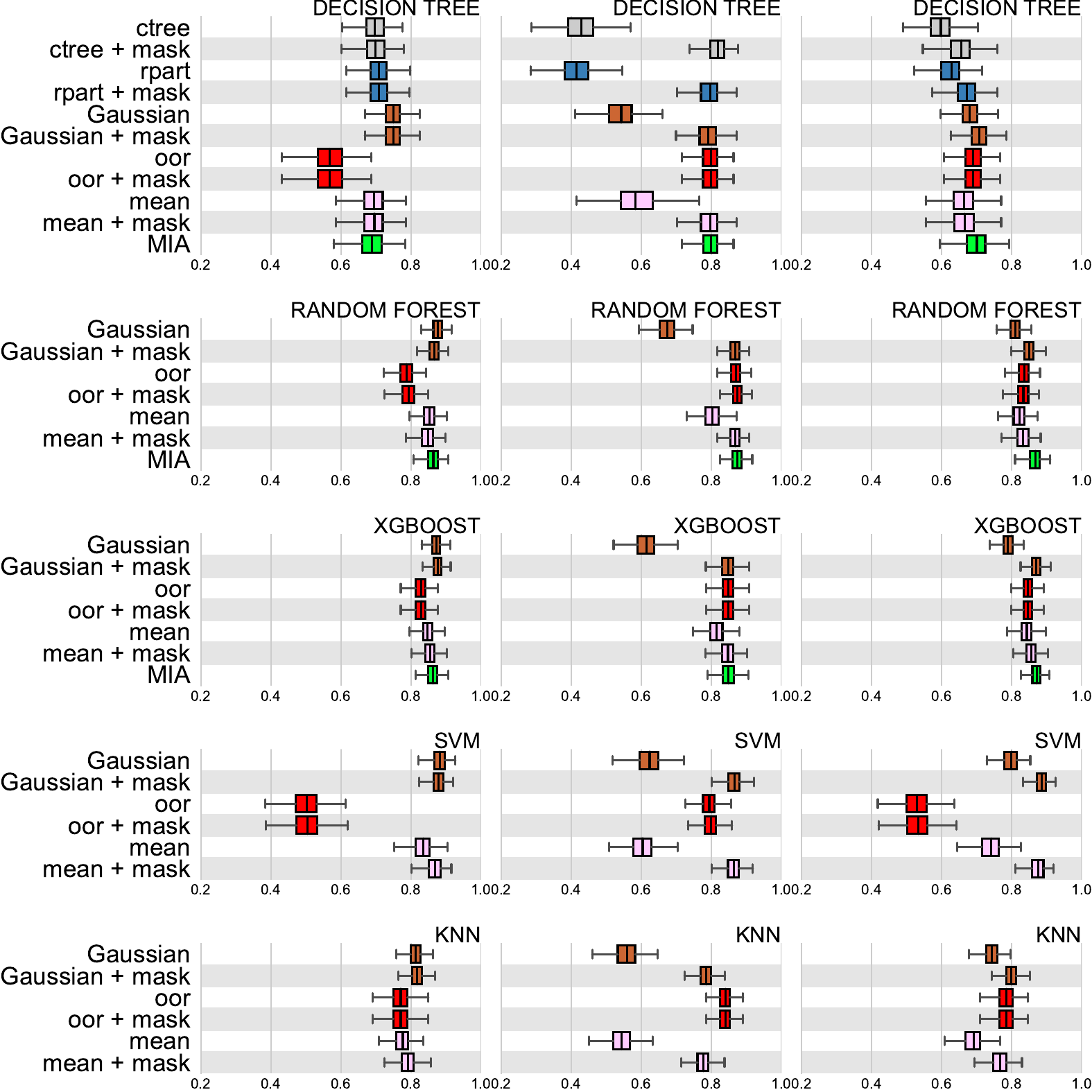}
    \llap{\raisebox{1.00\linewidth}{%
    {\sffamily%
    \hspace{.4\textwidth}MCAR
    \hspace{.2\textwidth}MNAR
    \hspace{.2\textwidth} Predictive M
    \hspace*{.1\textwidth}}}}

    \caption{\textbf{$R^2$ scores on model 1 $\bullet$} \modif{Normalized}
explained variance for the three missing data mechanisms (MCAR, Censoring MNAR, Predictive Missingness) introduced above, with 20\% of missing values,
$n=1000$, $d=9$ and $\rho = 0.5$.}
    \label{fig:simu_mcar_nonlinear_multivar_miss2_rho5}
\end{figure}

%

Figure \ref{fig:simu_mcar_nonlinear_multivar_miss2_rho5} presents the results for
one choice of correlation ($\rho = 0.5$) between covariates and percentage of missing
entries ($20\%$), as others give similar interpretation \modif{(see  Appendix~\ref{sec:appendix_additional_experiments} for different values of the correlation $\rho$ and the missing rate).}
 
In the MCAR case, all decision-tree methods perform similarly aside from
out-of-range imputation which significantly under-performs.
Performing a ``good'' conditional imputation, \textit{i.e.} one that
captures the relationships between variables such as \textit{impute
Gaussian}, slightly helps prediction. This is all the most true as the
correlation between variables increases, which we have not displayed.
Adding the missing-value mask is not important.
For powerful models, random forests and gradient boosting, the
benefit of conditional imputation compared to MIA is much reduced.
These models give significantly better prediction accuracy \modif{than tree-based models as expected.}

More complex patterns (MNAR or predictive missingness) reveal
the importance of adding the missing mask in most imputation strategies, except for the OOR imputation. MIA achieves excellent performance even for these more
complex missing-values mechanisms.
Remarkably, mean imputation also achieves good performances \modif{with random forests and xgboost methods} though adding
a mask helps. \modif{This is not the case for KNN and SVM.}
\modif{
Note that, whatever the missing data mechanism, combining OOR imputation with  SVM learners leads to sub-optimal predictive performances.  
}
%
%

\paragraph{Statistical tests} \modif{In order to assess whether the differences observed in Figure~\ref{fig:simu_mcar_nonlinear_multivar_miss2_rho5} are  significant, we employ statistical tests. More precisely, for each learning algorithm and each pair of imputation strategy, we consider the $R^2$ of the $1000$ repetitions. As these values are computed on similar individuals (the generated data are the same for all imputation methods and a given repetition), we employ paired t-tests. A p-value lower than $0.05$ indicates that the two considered imputation method exhibits different predictive performances. Such results help us to assess the significance of the difference observed via the boxplots in Figure~\ref{fig:simu_mcar_nonlinear_multivar_miss2_rho5}. P-values are displayed in Appendix~\ref{sec:app_statistical_tests}. The vast majority of differences observed in Figure~\ref{fig:simu_mcar_nonlinear_multivar_miss2_rho5} appear to be significant. For example, in the MCAR model, only the pair Gaussian/Gaussian+mask and Mean/Mean+mask are not statistically different for decision trees and random forests at the level $0.05$.}

\paragraph{Computational complexity} \modif{The computational complexity together with the predictive performances are two main criteria to assess the usefulness of a method handling missing values. We display in Figure~\ref{fig:simu_running_time} the training time (taking into account the imputation time and the training time on imputed data) of each method. } \modif{As expected, the computational time increases when adding the mask except for the forests. In this low-dimensional setting, KNN is the quickest method, followed by decision trees and SVMs. The most computationally intensive methods, as expected, are random forests and gradient boosting. Nevertheless, the computational time for data of this size is not substantial. The computation time is primarily attributable to the learners, and imputation methods have minimal impact.}

\begin{figure}
    \centering
    \includegraphics[scale=0.4]{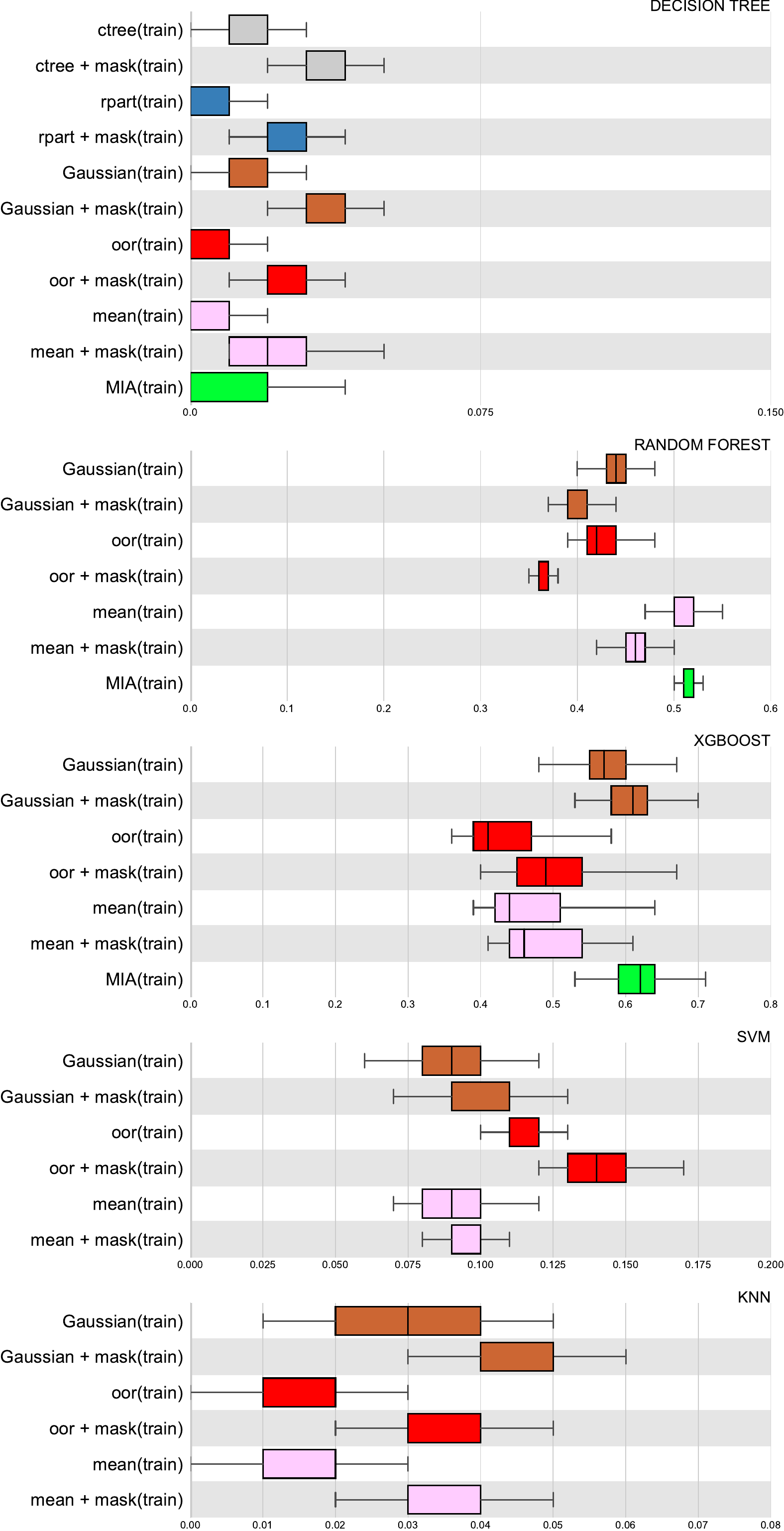}

    \caption{\textbf{Computation time (in seconds)} of the different imputation methods/learning procedures for the MCAR generative mechanism with 20\% of missing values,
$n=1000$, $d=9$ and $\rho = 0.5$.}
    \label{fig:simu_running_time}
\end{figure}

Figure \ref{fig:simu_mcar_nonlinear_multivar2} compares methods for
datasets with a non-linear generative model and
values missing completely at random. The figure
focuses on methods without adding the mask, as it 
makes little difference in MCAR settings. Even with
non linearities, Gaussian
imputation often performs as well or better than mean imputation or MIA,
likely because the non-linear supervised model compensates for the non
linearity.
All in all,
MIA proves to be a strong option in all the scenarios that we have
experimented \modif{when using tree-based methods}, although Gaussian imputation with the mask can outperform
it in these MCAR settings.

\begin{figure}
    \hspace*{.05\linewidth}%
    \includegraphics[clip, height=.46\textwidth]{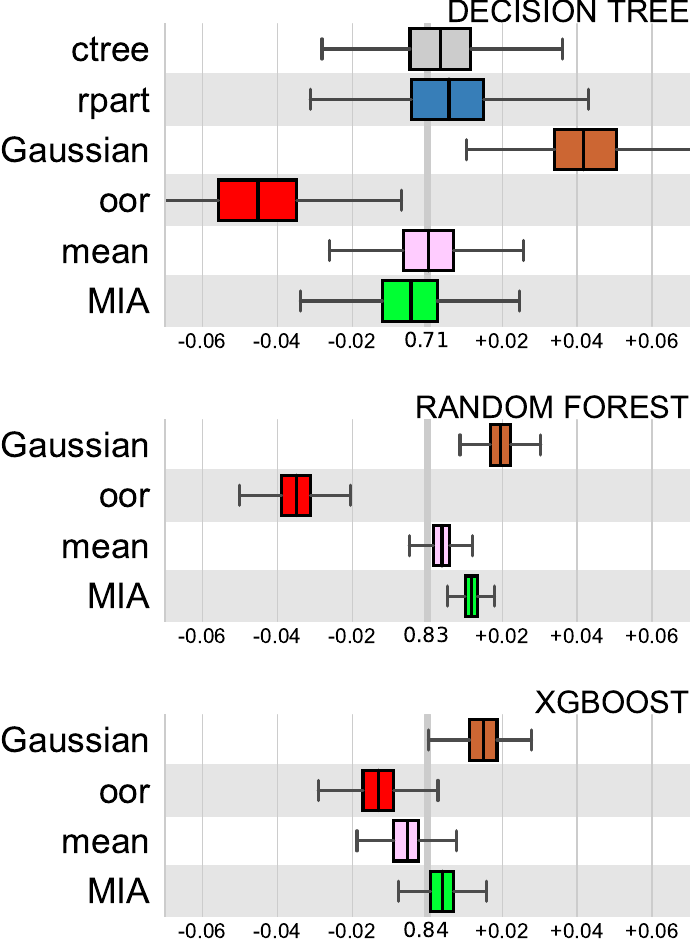}%
    \hfill%
    \includegraphics[trim={2.7cm 0 0 0}, clip, height=.46\textwidth]{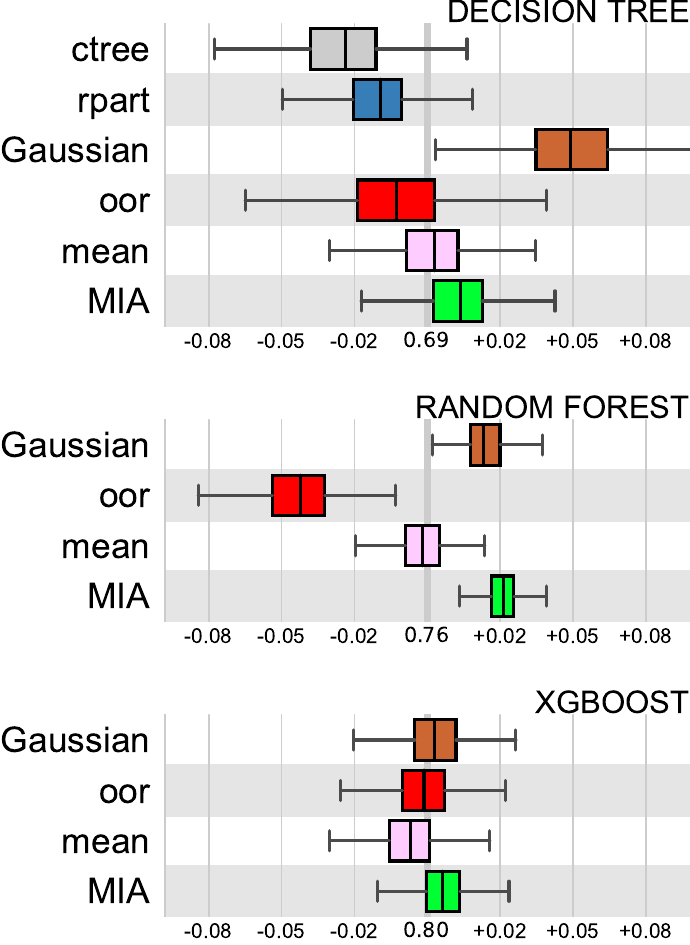}%
    \hfill%
    \includegraphics[trim={2.7cm 0 0 0}, clip, height=.46\textwidth]{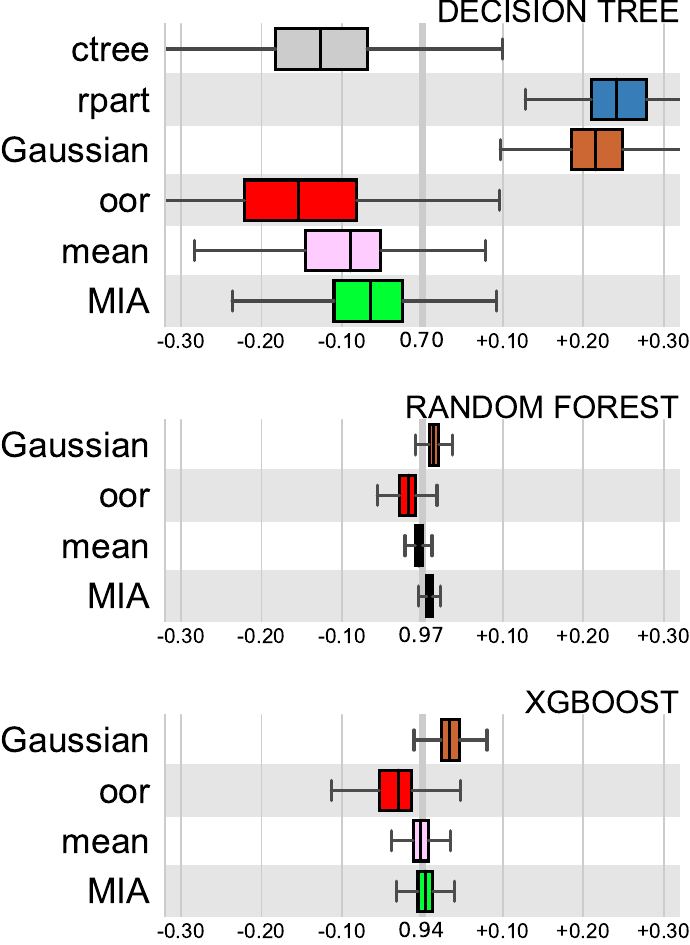}%
    \llap{\raisebox{.47\linewidth}{%
    {\sffamily%
    MCAR:\hspace{.06\textwidth}Model 2 (linear)
    \hspace{.12\textwidth}Model 3 (Friedman)
    \hspace{.1\textwidth}Model 4 (nonlinear)
    \hspace*{.02\textwidth}}}}

    \caption{\textbf{Relative scores on different models in MCAR $\bullet$} Relative explained variance for models 2, 3, 4, MCAR with 20\% of missing values, $n=1000$, $d=10$ and $\rho = 0.5$.}
    \label{fig:simu_mcar_nonlinear_multivar2}
\end{figure}  

\subsection{Consistency and learning curves}
\label{sec:simu_consistency}

\begin{figure}[!b]
    \includegraphics[width=1.05\textwidth]{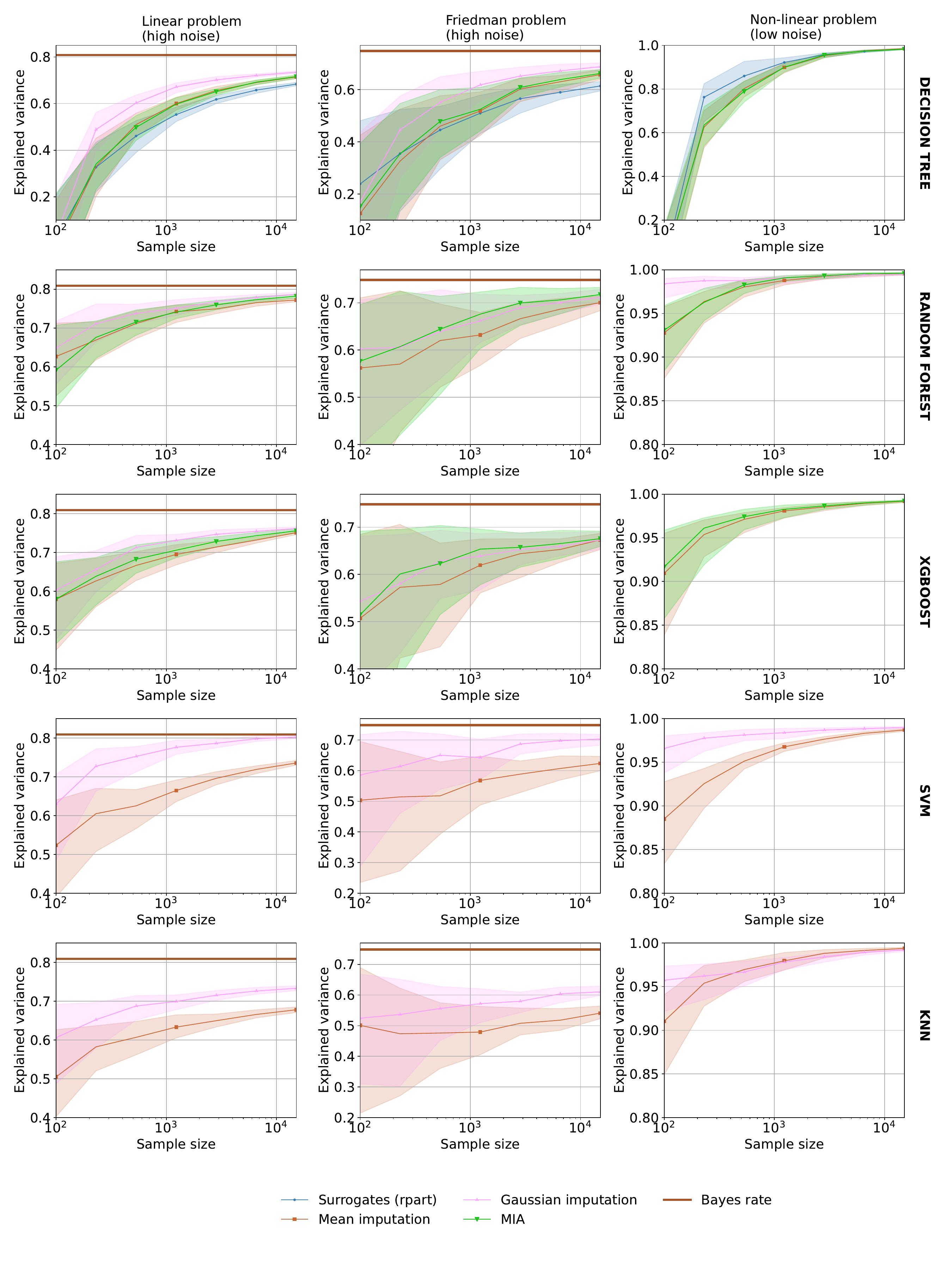}%

    \caption{\textbf{Bayes consistency in MCAR $\bullet$} Consistency with 40\% of MCAR values
on all variables, on models 2 (linear), 3 (Friedman), 4 (non-linear).}
    \label{fig:consistency}
\end{figure}

 In the third experiment, we compare the methods of Section \ref{sec:setting} varying sample size to assess their asymptotic performances, on models \ref{model2}, \ref{model3} and \ref{model4}. 
 We wish to compare the tree performance with respect to the Bayes loss.
For each sample size (between 300 and $10^4$), we summarize 200
repetitions by their median and quartiles (as in the boxplots). Assuming MCAR, the Bayes estimator is the expectation of $Y$ conditionally to the observed values,
\[
\E[Y|\bXm] = \E[f(\X)|\bXm] = \E[f(\X)| \bX_{obs(\bm)} = \bx_{obs(\bm)}, \bM = \bm].
\]
It has an simple expression only if the joint distribution of
$(\X,Y)$ is Gaussian. To compute an approximate Bayes loss for a nonlinear regression with Gaussian features, we apply joint Gaussian multiple imputation, as justified in  Section \ref{sec:testtimeMI},
 on a very large sample. 
 For the third scenario with non-Gaussian features, we have not computed the Bayes loss.




  
In linear settings, Figure \ref{fig:consistency} (left) shows that \textit{impute Gaussian} 
benefits from correlations between features and is the best-performing
method; For decision trees and forests, mean imputation, MIA and
surrogate splits are also consistent but with a slower convergence rate
(we have not displayed conditional trees as they exhibit the same behaviour as rpart 
with surrogate splits). Adding the indicator matrix in the data changes almost nothing here, so we have not displayed the corresponding curves.
For non-linear associations (Figure~\ref{fig:consistency}, middle and
right), the benefit brought by Gaussian imputation over the others methods
seems to carry over though it is less pronounced for random forests and
boosting.
For low-noise settings (Figure~\ref{fig:consistency}, right) MIA and mean imputation seem equivalent. \modif{Gaussian imputation is also the best method compared to mean imputation when using KNN or SVM. Note that SVM is not a local averaging method (in contrast to tree-based methods and nearest neighbors) and thus poor imputation may damage the predictive performances in the whole input space. On the contrary, imputed data points have only a local influence in tree-based methods and nearest neighbors, which may make them less sensitive to a bad imputation technique. }

For boosting, the difference between methods vanishes with large $n$, as
can be expected from boosting's ability to turn weak learners into strong
ones \citep{schapire1990strength}.
Gaussian imputation is still beneficial for small sample sizes.
Note that MIA can easily be implemented as a preprocessing step (see Remark~\ref{rem:mia} in Appendix~\ref{seq:MIA} for implementation details).

\replace{}{%
\paragraph{Link to studies on real-life data}
The experiments above are simple synthetic problems to showcase the
principles behind prediction with missing values.
A benchmark on real-life data is beyond the scope of this work because
it requires studying several datasets for generalizable findings. 
\cite{perez2022benchmarking} conducted such study on 13 different
prediction tasks on data with missing values. Their findings confirms
that our results apply to real-life data. In particular: \emph{1)} given
sufficient data, constant imputation (\emph{eg} with the mean) performs
as well as conditional imputation; \emph{2)} using MIA in tree-based
models gives leading methods for all sample sizes; \emph{3)} often adding
the missing indicator $\mathbf{M}$ in the features helps predictions.}

\paragraph{Take-home messages}

First, mean imputation can be appropriate and is consistent in a
supervised-learning setting when missing values are MAR and not related
to the outcome. Second, tree-based methods are an efficient way to target
$\widetilde{f}^\star (\bXm) = \E[Y | \bXm]$ especially when using MIA
(Section \ref{seq:MIA}) and can handle well informative pattern
of missing values. 

We compare imputation methods, using the ``proper way'' to impute as described in Section \ref{subsec:imputation_train_test}, \textit{i.e.}, where imputation values from the training set are used to impute the test set.

In addition, we consider imputation with the missing indicator $\M$ in the features. The rationale behind this indicator is that it can be  useful to improve the prediction when going beyond the hypothesis of Theorem \ref{thm:imp_mean_valid}, \textit{i.e.} considering a finite sample, a learning algorithm with a low approximation capacity (as linear regression) and with missing values that can either be MNAR or depend on $Y$. 

\section{Discussion and conclusion}

We have studied procedures for supervised learning with missing data.
Unlike in the classic missing data literature, the goal of the
procedures is to yield the best possible prediction on test data with
missing values. Focusing on simple ways of adapting existing procedures,
our theoretical and empirical results outline simple
practical recommendations:
\begin{itemize}[leftmargin=3ex, itemsep=.2ex, parsep=.1ex, topsep=.1ex]
\item
\modif{In presence of MAR missing values, Bayes-consistent estimate can be built by applying the Bayes predictor (on complete data) to data imputed via conditional multiple imputation, and averaging the resulting prediction (Theorem \ref{theo:multiple_imputation}). Thus, Theorem \ref{theo:multiple_imputation} justifies the use of multiple conditional imputation in a MAR setting. }
\item To train and test on data with missing values, the same imputation
model should be used. \modif{Constant imputation} is consistent, \modif{if used in conjunction with a
powerful, non-linear learning model}
(Theorem~\ref{thm:imp_mean_valid}).
\item \modif{There are a variety of manners to handle missing values inside decision trees. Among them,  
Missing Incorporated in Attribute 
(MIA, \citealt{Twala:2008:GMC:1352941.1353228}, see Remark~\ref{rem:mia} in Appendix~\ref{seq:MIA}), which works by  optimizing jointly
the split and the handling of the missing values is the most versatile, as it adapts to different missing data scenarios (see Section~\ref{sec:simu} and theoretical analysis in Appendix~\ref{sec:comparbre}).}
\item \modif{Empirically, the choice of  imputation methods (applied at train and
test time) may lead to a reduction of the number of samples required to reach a given prediction performance
(Figure~\ref{fig:consistency}).}
\item When missingness is related to the prediction target, imputation
does not suffice and it is useful to add the missingness indicators as features
(Example \ref{example_mnar2} and Figure~\ref{fig:simu_mcar_nonlinear_multivar_miss2_rho5}).
\end{itemize}
These recommendations hold to minimize the prediction error in an
asymptotic regime. More work is
needed to establish theoretical results in the finite sample regime.  In addition, different practices may be needed to
control for the uncertainty associated to a prediction. 

\section*{Declarations}

\paragraph{Ethics approval and Consent to participate}

As the work is purely mathematical and numerical, and all data are
synthetic, no ethic approval or consent to participate was necessary.


\paragraph{Competing interests}

We have no competing interests to disclose.

\paragraph{Funding}

JJ, NP, and GV acknowledge funding from ANR (``DirtyData'' grant
ANR-17-CE23-0018) and DataIA (``MissingBigData'' grant).



\paragraph{Acknowledgements}

We thank Stefan Wager for fruitful discussion and Julie Tibshirani for the suggestion to implement MIA.  
We also thank Antoine Ogier who initiated our work on imputation in supervised learning during his internship. We thank the reviewers who provided constructive feedback that helped to improve the quality of the manuscript. 

\paragraph{Consent for publication} Not applicable: no data or individual
images were used in this publication

\paragraph{Code availability} 

The code to reproduce the experiments of this work is available on
\url{https://github.com/dirty-data/supervised_missing}

\paragraph{Availability of data and material - (data transparency)} Not
applicable as all
the experiments in this study use simulated data.

\bibliographystyle{plainnat}
\bibliography{biblio}

\appendix


\section{Proofs of Section~\ref{sec:Bayesrisk}}
\label{apx:split}

\subsection{Proof of Theorem \ref{theo:multiple_imputation}}
\begin{proof}[Proof of Theorem \ref{theo:multiple_imputation}:
consistency of test-time conditional multiple imputation]

Let $\bxm \in (\R \cup \{ \texttt{NA}\})^d$. 
%
By Assumption~\ref{reg_assumption1}, $\E [ \varepsilon | \bX_{obs(\M)}] = 0$ a.s., which yields, a.s., 
\begin{align*}
\E [ Y | \bX_{obs(\bm)} = \bx_{obs(\bm)}] & = \E [ f^{\star}(\bX) + \varepsilon |  \bX_{obs(\bm)} = \bx_{obs(\bm)}]\\
& = \E [ f^{\star}(\bX) |  \bX_{obs(\bm)} = \bx_{obs(\bm)}]\\
& = \E [ f^{\star} (o(\bX_{obs(\bm)}, \X_{mis(\bm)} ; \bm)) |  \bX_{obs(\bm)} = \bx_{obs(\bm)}].
\end{align*}
By definition, the multiple imputation procedure described in Theorem~\ref{theo:multiple_imputation} is given by
\begin{align}
\regMultImput (\bxm) & = \E_{{\X_{mis(\bm)}|\X_{obs(\bm)}=\x_{obs(\bm)}}} [ f^{\star} (o(\x_{obs(\bm)}, \X_{mis(\bm)} ; \bm))] \nonumber \\
& = \E [ f^{\star} (o(\bX_{obs(\bm)}, \X_{mis(\bm)} ; \bm)) |  \bX_{obs(\bm)} = \bx_{obs(\bm)}] \nonumber \\
& = \E [ Y | \bX_{obs(\bm)} = \bx_{obs(\bm)}].
\label{eq:prop_mult_imp_1}
\end{align}
Besides, since the missing pattern is MAR, 
\begin{align}
\regTilde (\bxm) & = \E[Y|\bXm = \bxm] \nonumber \\
&= \E[Y | \bX_{obs(\bm)} = \bx_{obs(\bm)}, \bM = \bm] \nonumber \\
&= \E[Y | \bX_{obs(\bm)} = \bx_{obs(\bm)} ] 
\label{eq:prop_mult_imp_2}
\end{align}
Combining (\ref{eq:prop_mult_imp_1}) and (\ref{eq:prop_mult_imp_2}), we finally obtain
\begin{align*}
 \regMultImput (\bxm) =  \regTilde (\bxm).
\end{align*}
\end{proof}

\subsection{Proof of Theorem \ref{thm:imp_mean_valid}}
\begin{proof}[Proof of Theorem \ref{thm:imp_mean_valid}: consistency of
mean imputation at train and test time]

We distinguish the three following cases in order to make explicit the expression of $\E[Y|\bX' = \bx ]$. 

\paragraph{First case : let $\bx \in [0,1]^d$ such that $x_1 \neq \impConst$.}\

For $0 < \rho < |x_1 - \impConst|$,  letting $B(\bx, \rho)$ be the euclidean ball centered at $\bx$ of radius $\rho$, 
\begin{align}
    \E[Y|\bX' \in B(\bx,\rho) ] & = \frac{\E[Y \mathds{1}_{\bX' \in B(\bx, \rho)} ]}{\P[\bX' \in B(\bx, \rho) ]} \nonumber \\
    & = \frac{\E[Y \mathds{1}_{\bX \in B(\bx, \rho)} \mathds{1}_{M_1=0} ]}{\P[\bX \in B(\bx, \rho), M_1=0 ]} \nonumber \\
    & = \E[Y|\bX \in B(\bx, \rho), M_1=0 ]\,. \label{proof_thm1_eq1}
\end{align}
Taking the limit of (\ref{proof_thm1_eq1}) when $\rho$ tends to zero,  
\begin{align}
    \E[Y|\bX' = \bx ] & = \lim_{\rho \to 0} \E[Y|\bX' \in B(\bx, \rho) ] = \E[Y|\bX = \bx, M_1=0 ]. \label{eq:proof_thm_mean1}
\end{align}

\paragraph{Second case : let $\bx \in [0,1]^d$ such that $x_1 = \impConst$.}\

\emph{First Subcase: assume $\P[M_1 = 1 | X_2 = x_2, \hdots, X_d = x_d] = 0$.}\

We have $\{\bX' = \bx \} = \{\bX' = \bx, M_1 = 0 \} = \{\bX = \bx \} $, and consequently,
\begin{align}
    \E[Y | \bX' = \bx] & = \E[Y | \bX = \bx]. \label{proof_thm_eq_final0}
\end{align}

\emph{Second Subcase: assume $\P[M_1 = 1 | X_2 = x_2, \hdots, X_d = x_d] > 0$.}\

We have 
\begin{align*}
\P[\bX' \in B(\bx, \rho) ] & = \E[\mathds{1}_{\bX' \in B(\bx, \rho)} \mathds{1}_{M_1 = 0}] + \E[\mathds{1}_{\bX' \in B(\bx, \rho)} \mathds{1}_{M_1 = 1}]\\
& = \E[\mathds{1}_{\bX \in B(\bx, \rho)} \mathds{1}_{M_1 = 0}] + \E[\mathds{1}_{(X_2, \hdots, X_d) \in B((x_2, \hdots, x_d), \rho)} \mathds{1}_{M_1 = 1}],
\end{align*}
and 
\begin{align*}
\E[f(\bX) \mathds{1}_{\bX' \in B(\bx, \rho)} ]
& = E[f(\bX) \mathds{1}_{\bX \in B(\bx, \rho)} \mathds{1}_{M_1=0} ]\\
& \quad + E[f(\bX) \mathds{1}_{(X_2, \hdots, X_d) \in B((x_2, \hdots, x_d), \rho)} \mathds{1}_{M_1=1} ].
\end{align*}
Therefore, 
\begin{align}
 \E[Y|\bX' \in B(\bx, \rho) ] 
= & \frac{\E[f(\bX) \mathds{1}_{\bX' \in B(\bx, \rho)} ]}{\P[\bX' \in B(\bx, \rho) ]} \nonumber \\
= & \frac{\E[f(\bX) \mathds{1}_{\bX \in B(\bx, \rho)} \mathds{1}_{M_1=0} ] + \E[f(\bX) \mathds{1}_{(X_2, \hdots, X_d) \in B((x_2, \hdots, x_d), \rho)} \mathds{1}_{M_1=1} ]}{ \E[\mathds{1}_{\bX \in B(\bx, \rho)} \mathds{1}_{M_1 = 0}] + \E[\mathds{1}_{(X_2, \hdots, X_d) \in B((x_2, \hdots, x_d), \rho)} \mathds{1}_{M_1 = 1}]}.\label{eq:proof_thm_mean2}
\end{align}
The terms in (\ref{eq:proof_thm_mean2}) involving $M_1 = 0$ satisfy
\begin{align}
\E[\mathds{1}_{\bX \in B(\bx, \rho)} \mathds{1}_{M_1 = 0}] & \leq \mu( B(\bx, \rho)) 
\leq \frac{\pi^{d/2}}{\Gamma(\frac{d}{2}+1)} \|g\|_{\infty} \rho^d, \label{eq:proof_thm_mean2.1}
\end{align}
and
\begin{align}
|\E[f(\bX) \mathds{1}_{\bX \in B(\bx, \rho)} \mathds{1}_{M_1=0} ]| & \leq \E[|f(\bX)| \mathds{1}_{\bX \in B(\bx, \rho)} ]| \nonumber \\
& \leq   \frac{\pi^{d/2}}{\Gamma(\frac{d}{2}+1)} \|g\|_{\infty} \|f\|_{\infty} \rho^d. \label{eq:proof_thm_mean2.2}
\end{align}
The second term of the denominator in (\ref{eq:proof_thm_mean2}) can be bounded from below, 
\begin{align}
& \E[\mathds{1}_{(X_2, \hdots, X_d) \in B((x_2, \hdots, x_d), \rho)} \mathds{1}_{M_1 = 1}] \nonumber \\
& =   
\E[\mathds{1}_{(X_2, \hdots, X_d) \in B((x_2, \hdots, x_d), \rho)} \P[M_1 = 1 | X_2, \hdots, X_d]] \nonumber \\
& \geq   \frac{\pi^{(d-1)/2}}{\Gamma(\frac{d-1}{2}+1)} \big(\inf_{[0,1]^d } g \big) \rho^{d-1} \eta. \label{eq:proof_thm_mean2.3}
\end{align}
The second term of the numerator in (\ref{eq:proof_thm_mean2}) verifies 
\begin{align*}
& \E[f(\bX) \mathds{1}_{(X_2, \hdots, X_d) \in B((x_2, \hdots, x_d), \rho)} \mathds{1}_{M_1=1} ] \\
&  = 
\E[\mathds{1}_{(X_2, \hdots, X_d) \in B((x_2, \hdots, x_d), \rho)} \E[ f(\bX)  \mathds{1}_{M_1=1} | X_2, \hdots, X_d]]\\
& = \E[\mathds{1}_{(X_2, \hdots, X_d) \in B((x_2, \hdots, x_d), \rho)} \E[ f(\bX) | X_2, \hdots, X_d] \E[ \mathds{1}_{M_1=1} | X_2, \hdots, X_d] ].
\end{align*}
If $\E[f(\bX)|X_2 = x_2, \hdots, X_d = x_d] > 0$, by uniform continuity of $f$ and $g$,  
\begin{align*}
& \E[\mathds{1}_{(X_2, \hdots, X_d) \in B((x_2, \hdots, x_d), \rho)} \E[ f(\bX) | X_2, \hdots, X_d] \E[ \mathds{1}_{M_1=1} | X_2, \hdots, X_d] ] \\
 & \geq   \E[f(\bX) |X_2 = x_2, \hdots, X_d = x_d]  \frac{\pi^{(d-1)/2}}{\Gamma(\frac{d-1}{2}+1)} \big(\inf_{[0,1]^d } g\big) \rho^{d-1} \eta.
\end{align*}
Similarly, if $\E[f(\bX)|X_2 = x_2, \hdots, X_d = x_d] < 0$, we have
\begin{align*}
& \E[\mathds{1}_{(X_2, \hdots, X_d) \in B((x_2, \hdots, x_d), \rho)} \E[ f(\bX) | X_2, \hdots, X_d] \E[ \mathds{1}_{M_1=1} | X_2, \hdots, X_d] ] \\
&  \leq   \E[f(\bX) |X_2 = x_2, \hdots, X_d = x_d]  \frac{\pi^{(d-1)/2}}{\Gamma(\frac{d-1}{2}+1)} \big(\inf_{[0,1]^d } g\big) \rho^{d-1} \eta\\
&  \leq 0.
\end{align*}
Hence, if $\E[f(\bX)|X_2 = x_2, \hdots, X_d = x_d] \neq 0$
\begin{align}
& |\E[f(\bX) \mathds{1}_{(X_2, \hdots, X_d) \in B((x_2, \hdots, x_d), \rho)} \mathds{1}_{M_1=1} ]| \nonumber \\
& \geq  |\E[f(\bX) |X_2 = x_2, \hdots, X_d = x_d] | \frac{\pi^{(d-1)/2}}{\Gamma(\frac{d-1}{2}+1)} \big(\inf_{[0,1]^d } g\big) \rho^{d-1} \eta. \label{eq:proof_thm_mean2.4}
\end{align}
Gathering inequalities (\ref{eq:proof_thm_mean2.1})-(\ref{eq:proof_thm_mean2.4}) and using equation (\ref{eq:proof_thm_mean2}), we have, if $\E[f(\bX)|X_2 = x_2, \hdots, X_d = x_d] \neq 0$
\begin{align*}
\lim_{\rho \to 0} \E[Y|\bX' \in B(\bx, \rho) ] & = \lim_{\rho \to 0} \frac{ E[f(\bX) \mathds{1}_{(X_2, \hdots, X_d) \in B((x_2, \hdots, x_d), \rho)} \mathds{1}_{M_1=1} ]}{\E[\mathds{1}_{(X_2, \hdots, X_d) \in B((x_2, \hdots, x_d), \rho)} \mathds{1}_{M_1 = 1}]}\\
& = \E[f(\bX) | X_2 = x_2, \hdots, X_d = x_d, M_1 = 1].
\end{align*}
Finally, if $\E[f(\bX)|X_2 = x_2, \hdots, X_d = x_d] = 0$ then by uniform continuity of $f$, there exists $\varepsilon_{\rho}$ such that $\varepsilon_{\rho} \to 0$ as $\rho \to 0$ satisfying, 
\begin{align*}
|\E[f(\bX) \mathds{1}_{(X_2, \hdots, X_d) \in B((x_2, \hdots, x_d), \rho)} \mathds{1}_{M_1=1} ]| & \leq \varepsilon_{\rho} \rho^{d-1} \frac{\pi^{d/2}}{\Gamma(\frac{d}{2}+1)} \|g\|_{\infty},
\end{align*}
hence 
\begin{align*}
 \lim_{\rho \to 0} \E[Y|\bX' \in B(\bx, \rho) ] & = 0\\
& = \E[f(\bX)|X_2 = x_2, \hdots, X_d = x_d]\\
& = \E[f(\bX)|X_2 = x_2, \hdots, X_d = x_d, M_1 = 1],
\end{align*}
since $M_1 \indep X_1 | (X_2, \hdots, X_d)$. Consequently, for all $\bx \in [0,1]^d$ such that $x_1 = \impConst$, 
\begin{align}
 \lim_{\rho \to 0} \E[Y|\bX' \in B(\bx, \rho) ]  = \E[f(\bX)|X_2 = x_2, \hdots, X_d = x_d, M_1 = 1]. \label{proof_thm1_eq_final}
\end{align}
Combining equations (\ref{eq:proof_thm_mean1}), (\ref{proof_thm_eq_final0}) and (\ref{proof_thm1_eq_final}), the prediction given by the mean imputation followed by learning is, for all $\bx' \in \R^d$,
\begin{align*}
    \regTilde (\bx') & = \E[Y|X_2 = x_2, \hdots, X_d = x_d, M_1 = 1] \mathds{1}_{x_1' = \impConst} \mathds{1}_{\P[M_1 =1| X_2 = x_2, \hdots, X_d = x_d] >0} \\
    & \quad + \E[Y|\bX = \bx']  \mathds{1}_{x_1' = \impConst} \mathds{1}_{\P[M_1 =1| X_2 = x_2, \hdots, X_d = x_d] =0} \\
    & \quad + \E[Y|X_2 = x_2, \hdots, X_d = x_d, M_1 = 0] \mathds{1}_{x_1' \neq \impConst},
\end{align*}
which concludes the proof.
\end{proof}

\section{Decision trees: an example of empirical risk minimization with
missing data}\label{sec:trees}


Aside from using procedures that require imputation prior to learning in order to handle missing values (as theoretically discussed in the previous section), one can use procedures intrinsically able to handle missing values. Among them, decision trees offer a natural way for empirical risk minimization with
missing values \citep{saar2007handling,Twala:2008:GMC:1352941.1353228}. This is due to their ability to handle the half-discrete
nature of $\tilde \X$, as they rely on greedy decisions rather than
smooth optimization.

We first present the different approaches available to handle
missing values in tree-based methods in Sections \ref{sec:availablecase}
and \ref{seq:MIA}. We then compare them theoretically in Section
\ref{sec:comparbre}, highlighting the interest of using the ``missing
incorporated in attribute'' approach in particular when the missing values are informative. In Section~\ref{sec:simu}, we will compare numerically the two different strategies: imputation prior to learning and learning directly with missing data.

\subsection{Tree construction with CART}

CART (Classification And Regression Trees, \citealt{DBLP:books/wa/BreimanFOS84}) is one of the most popular tree algorithm, originally designed for complete data sets. It recursively builds a partition of the input space $\mathcal{X} = \R^d$, and predicts by aggregating the observation labels (average or majority vote) inside each terminal node, also called leaf. For each  node $A = \prod_{j=1}^d [a_{j,L}, a_{j,R}] \subset \R^d$, CART algorithm finds the best split $(j^{\star}, z^{\star})$ among the set of all eligible splits $\mathcal{S} = \{(j,z), j \in \llbracket 1,d\rrbracket,~z\in\R, z_j \in
[a_{j,L}, a_{j,R}] \}$, 
where a split is defined by the variable $j$ along which the split is performed and the position $z$ of the split. More precisely, 
the best split $(j^{\star}, z^{\star})$ in any node $A$ is the solution of the following optimization problem
\begin{align}
(j^{\star}, z^{\star}) \in \underset{(j,z) \in \mathcal{S}}{\argmin}~
& \E\Big[  \big(Y-\E[Y|X_j\leq z, \X \in A]\big)^2 \cdot\mathds 1_{X_j\leq z, \X \in A} \nonumber\\ 
& \quad   + \big(Y-\E[Y|X_j> z, \X \in A]\big)^2 \cdot\mathds 1_{X_j>z, \X \in A} \Big]. 
\label{split_criterion_cart}
\end{align}

For each cell $A$, problem (\ref{split_criterion_cart}) can be rewritten as
\begin{align}
f^{\star} \in \underset{f \in \mathcal{P}_c}{\argmin}~
& \E\Big[  \big(Y- f(\X)\big)^2 \mathds{1}_{\X \in A}  \Big], \label{optimization_pb_cart} 
\end{align}
where $\mathcal{P}_c$ is the set of piecewise-constant functions on $A \cap \{x_j \leq s\}$ and $A \cap \{x_j > s\}$ for $(j,s) \in \mathcal{S}$. Therefore the optimization problem  (\ref{optimization_pb_cart}) amounts to solving a least square problem on the subclass of functions $\mathcal{P}_c$.
Thus, by minimizing the mean squared error, the CART procedure targets the quantity $\E[Y|\X]$.
In the presence of missing values, this criterion must be adapted and
several ways to do so have been proposed. Section~\ref{sec:availablecase} and Section~\ref{seq:MIA} detail the existing criteria that can be used when dealing with missing values. 

\subsection{Splitting criterion discarding missing values}
\label{sec:availablecase}

A simple option to extend CART methodology in presence of missing values
is to select the best split only on the available cases for each
variable. More precisely, for any node $A$, the best split in presence of
missing values is a solution of the new optimization problem
\begin{align}
 (j^{\star}, z^{\star}) \in \underset{(j,z) \in \mathcal{S}}{\argmin}~
\E\Big[&  \big(Y-\E[Y|X_j\leq z,  M_j = 0, \X \in A ]\big)^2 \cdot\mathds 1_{X_j\leq z, M_j = 0, \X \in A} \nonumber \\
&  + \big(Y-\E[Y|X_j> z,  M_j = 0, \X \in A]\big)^2 \cdot\mathds 1_{X_j>z,  M_j = 0, \X \in A} \Big],  \label{split_criterion_cart_missvalue}
\end{align}
which is nothing but problem \eqref{split_criterion_cart} computed, for each $j \in \{1, \hdots , d\}$, on observed values ``$M_j=0$'' only. Note that, by convention, the problem \eqref{split_criterion_cart_missvalue} is optimised empirically along the variables $j$ that have at least two observed entries (we do not allow splits to be performed along variables that have less than two observed entries). This splitting strategy is described in Algorithm~\ref{alg-split-missdata}. As the missing values were not used to calculate the criterion, it is still necessary to specify to which cell they are sent. The solution consisting in discarding missing data at each step would lead to a drastic reduction of the data set and is therefore not viable. The different methods to propagate missing data down the tree are explained below.  

\begin{algorithm}[htbp]
  \small
  \caption{\small Splitting strategy}
  \label{alg-split-missdata}   
  \begin{algorithmic}[1]
    \STATE \textbf{Input:} a node $A$, a sample $\dataset_n$ of observations falling into $A$. 
    \STATE For each variable $ j \in \{1, \hdots , d\}$ and position $z$, compute the CART splitting criterion on observed values ``$M_j=0$'' only.
    \STATE Choose a split $(j^{\star},z^{\star})$ that minimizes the previous criterion (see optimization problem \ref{split_criterion_cart_missvalue}).
    \STATE     Split the cell $A$ accordingly. Two new cells $A_L$ and $A_R$ are created.  
    \STATE \textbf{Output:} Split $(j^{\star},z^{\star})$, $A_L$, $A_R$. 
  \end{algorithmic}
\end{algorithm}


\paragraph*{Surrogate splits}
Once the best split is chosen via Algorithm~\ref{alg-split-missdata},
surrogate splits search for a split on another variable that induces a
data partition close to the original one. More precisely, for a selected
split $(j^{\star}_0, z^{\star}_0)$, to observations send
down the tree with no $j^{\star}_0$th variable, a new stump,
\textit{i.e.}, a tree with one cut, is fitted to the response
$\mathds{1}_{X_{j^{\star}_0}\leq z^{\star}_0}$, using variables
$(X_j)_{j\neq j^{\star}_0}$. The split $(j^{\star}_1, z^{\star}_1)$ which
minimizes the misclassification error is selected, and observations are
split accordingly. Those that lack both variables $j^{\star}_0$ and
$j^{\star}_1$ are split with the second best, $j^{\star}_2$, and so on
until the proposed split has a worse misclassification error than the
blind rule of sending all remaining missing values to the same daughter,
the most populated one. To predict, the training surrogates are kept.
This construction is described in Algorithm~\ref{alg-split-surrogate} and
is the default method in {\tt rpart} \citep{therneau1997introduction}.
Surrogate method is expected to be appropriate when there are
relationships between covariates. 

\begin{algorithm}[b!]
  \small
  \caption{\small Surrogate strategy}
  \label{alg-split-surrogate}   
  \begin{algorithmic}[1]
    \STATE \textbf{Input:} a node $A$, a sample $\dataset_n$ of observations falling into $A$, a split $(j_0^{\star}, z_0^{\star})$ produced by Algorithm~\ref{alg-split-missdata}
    \STATE Create the variable $\mathds{1}_{X_{j^{\star}_0}\leq z^{\star}_0}$  
    \STATE Let $\mathcal{J} = \{1, \hdots , d\} - \{j_0\}$
    \WHILE{Missing data have not been sent down the tree}
    \STATE Compute the misclassification error $\varepsilon^{\star}$ corrresponding to sending all remaining missing observation on the most populated side. 
    \FOR{all $j \in \mathcal{J}$}
        \STATE Fit a tree with one split along $j$, on data in $\mathcal{D}_n$ with observed values $X_j$ and $X_{j^{\star}_0}$, in order to predict $\mathds{1}_{X_{j^{\star}_0}\leq z^{\star}_0}$. 
        \STATE For $j, \in \mathcal{J}$, let $\varepsilon_j$ be the misclassification error of the previous trees
        \STATE Let $j_{min} \in \argmin_{j \in \mathcal{J}}~ \varepsilon_j$
    \ENDFOR 
    \IF{$\varepsilon_{j_{min}} < \varepsilon^{\star}$}
    \STATE Use the tree built on $j_{min}$ to send data with missing values on $j_0$ into $A_L$ or $A_R$, depending on the tree prediction
    \STATE $\mathcal{J} \leftarrow \mathcal{J} \cup \{j_{min}\}$
    \ELSE 
    \STATE Send all remaining missing values to the most populated cell ($A_L$ or $A_R$)
    \ENDIF
    \ENDWHILE
  \end{algorithmic}
\end{algorithm}

\paragraph*{Probabilistic splits}
Another option is to propagate missing observations according to a Bernoulli distribution $\mathcal B(\frac{\#L}{\#L+\#R})$, where $\#L$ (resp. $\#R$) is the number of points already on the left (resp. right) (see Algorithm~\ref{alg-probabilistic-split} in Appendix~\ref{subsec:app:algo}). This is the default method in C4.5 \citep{quinlan2014c4}.

\paragraph*{Block propagation} A third option is to choose the split on the
observed values, and then send all incomplete observations as a block, to
a side chosen by minimizing the error (see
Algorithm~\ref{alg-block-propagation} in Appendix~\ref{subsec:app:algo}).
This is the method used in LightGBM \citep{ke2017lightgbm}.\\

Note that \cite{Hothorn:2006:JCGS} proposed conditional trees, a variant of CART which also uses surrogate splits, but adapts the criterion \eqref{split_criterion_cart_missvalue} to missing values.
Indeed, this criterion implies a selection bias: it leads to
underselecting the variables with many missing values due to multiple
comparison effects \citep{StrBouAug:2007:CSDA}. As a result, it favors variables where many splits are available, and therefore those with
fewer missing values.
Conditional trees are based on the calculation of a linear statistic of
association between $Y$ and each feature $X_j$, $T=\langle X_j,Y\rangle$. Then,
its distribution under the null hypothesis of independence between $Y$ and $X_j$ is estimated by permutation, and the variable with the smallest $p$-value is selected.  As illustrated in Appendix \ref{sec:selection}, conditional trees are meant to improve the selection of the splitting variables 
but do not ensure an improvement in prediction performance. 

\subsection{Splitting criterion with missing values: MIA} 
\label{seq:MIA}

A second class of methods uses missing values to compute the
splitting criterion itself. Consequently, the splitting location depends
on the missing values, contrary to all methods presented in Section~\ref{sec:availablecase}.  Its most common instance is  ``missing incorporated in attribute" (MIA, \citealt{Twala:2008:GMC:1352941.1353228}, ). 
More specifically, MIA considers the following splits, for all splits $(j,z)$: 
\begin{itemize}
\item $\{\Xm_j\leq z ~\textrm{or}~  \Xm_j=\NA\}$ vs $ \{\Xm_j>z\}$,
    
\item $\{\Xm_j\leq z\}$ vs $\{\Xm_j>z ~\textrm{or}~ \Xm_j=\NA\}$,
    
\item  $\{\Xm_j\neq\NA\}$ vs $\{\Xm_j=\NA\}$.
\end{itemize}
In a nutshell, for each possible split, MIA tries to send all missing
values to the left or to the right, and compute for each choice the
corresponding error (right-hand side in \ref{split_criterion_cart}, as
well as the error associated to separating the observed values from
the missing ones. Finally, it chooses the split among the previous ones with the lowest error (see Algorithm~\ref{alg-MIA}  in Appendix~\ref{subsec:app:algo}). 
Note that block propagation can be seen as a greedy way of successively 
 optimizing the choices in two first options. However, as we show in 
Prop. \ref{prop_criterion}, these successive choices are sub-optimal.

Missing values are treated as a category by MIA, which is thus nothing but a greedy algorithm
minimizing the square loss between $Y$ and a function of $\bXm$ and consequently targets
the quantity \eqref{eq:2d_mod} which separate $\E[Y |\bXm]$  into $2^d$ terms. However, it is not exhaustive: at each step,
the tree can cut for each variable according to missing or non missing and selects this cut when it is relevant, \textit{i.e.} when it minimizes the prediction error. The final leaves can correspond to a cluster of missing values patterns (observations with missing values on the two first variables for instance and any missing patterns for the other variables).

MIA is thought to be a good method to apply when missing pattern is
informative, as this procedure allows to cut with respect to missing/ non
missing and uses missing data to compute the best splits. Note this
latter property implies that the MIA approach does not require a
different method to propagate missing data down the tree. Notably, MIA is
implemented in the R packages \texttt{partykit} \citep{partykit} and
\texttt{grf}  \citep{grf}, as well as in XGBoost \citep{chen2016xgboost} and 
for the {\tt HistGradientBoosting} models in scikit-learn \citep{scikit-learn}.

\begin{remark}{\bf Implementation:}\label{rem:mia}
A simple way to implement MIA consists in duplicating the incomplete columns, and replacing the missing entries once by $+\infty$ and once by $-\infty$ (or an extreme out-of-range value). This creates two dummy variables for each original one containing missing values. Splitting along a variable and sending all missing data to the left (for example) is the same as splitting along the corresponding dummy variable where missing entries have been completed by $-\infty$. 
Alternatively, MIA can be with two scans on a feature's values in
ascending and descending orders \citep[Alg 3]{chen2016xgboost}.
\end{remark}

\begin{remark}{\bf Implicit imputation:}
Whether it is in the case where the missing values are propagated in the available case method (Section \ref{sec:availablecase}),
or incorporated in the split choice in MIA, missing values are assigned either to the left or  the right interval.
Consequently, handling missing values in a tree can be seen as implicit
imputation by an interval value. 
\end{remark}

%

\subsection{Theoretical comparison of CART versus MIA} \label{sec:comparbre}


We now compare theoretically
the positions of the splitting point at the root and the prediction errors
on simple examples with MCAR values. Proposition \ref{prop_criterion}
computes the
splitting position of MIA and CART, and highlights that the
splitting position of MIA
varies even for MCAR missing data.
Proposition \ref{thm:comp_tree_theo}
then compares the risk of the different splitting strategies:
probabilistic split, block propagation, surrogate split, and MIA.
We prove that MIA and surrogate splits are the two best strategies, one of which may be better than the other depending on the dependence structure of  covariates. 

\begin{restatable}{proposition}{propcrit}
\label{prop_criterion}
Let $p \in [0,1]$. Consider the regression model 
\begin{align*}
    \left\lbrace
    \begin{array}{llc}
       Y & = & X_1\\
       X_1 & \sim & U ([0,1])
    \end{array}
    \right. 
    ,\quad 
    \left\lbrace
    \begin{array}{llc}
       \P[M_1=0] & = & 1-p\\
       \P[M_1=1] & = & p
    \end{array}
    \right. ,
\end{align*}
where $M_1 \indep (X_1, Y)$ is the missingness pattern on $X_1$. Let $C_{\textrm{MIA}}(1, s,q,p)$  be the value of the splitting MIA criterion computed on $X_1$ at threshold $s$ such that $(1,s) \in \mathcal{S}$, and $q \in \{\textrm{L}, \textrm{R}\}$, where $q$ stands for the side where missing values are sent. Therefore,  
\begin{enumerate}
      \item The best split $s^{\star}$ given by the CART criterion \eqref{split_criterion_cart_missvalue} is $s^{\star}=1/2$.
    
    \item The best splits $s^{\star}_{\textrm{MIA}, L}(p)$ and $s^{\star}_{\textrm{MIA}, R}(p)$ given by the MIA procedure (described in Section~\ref{seq:MIA}), assuming that all missing values are sent to the left node (resp. to the right node), satisfy
    \begin{equation}\label{eq:propmia}
    s^\star_{\textrm{MIA}, L}(p) = \underset{s\in[0,1]}{\argmin}~
    C_{\textrm{MIA}}(1,s, \textrm{L},p),
    \end{equation}
where 
\begin{align*}
    C_{\textrm{MIA}}(1, s, \textrm{L},p) & = \frac{1}{3} - \frac{1}{p + (1-p)s}  \Big( \frac{p}{2} + \frac{(1-p) s^2}{2}\Big)^2  - (1-p)(1-s)\Big(\frac{1+s}{2}\Big)^2,
\end{align*}
and  $s^{\star}_{\textrm{MIA}, R}(p) = 1 - s^{\star}_{\textrm{MIA}, L}(p)$.
\end{enumerate}
\end{restatable}

\begin{figure}[t!]
    \begin{minipage}{.5\linewidth}
    \caption{Split position chosen by  MIA and CART criterion, depending on the fraction $p$ of missing values on $X_1$.}%
    \label{fig:s_mia}
    \end{minipage}%
    \hfill%
    \begin{minipage}{.45\linewidth}
    \includegraphics[width=\linewidth]{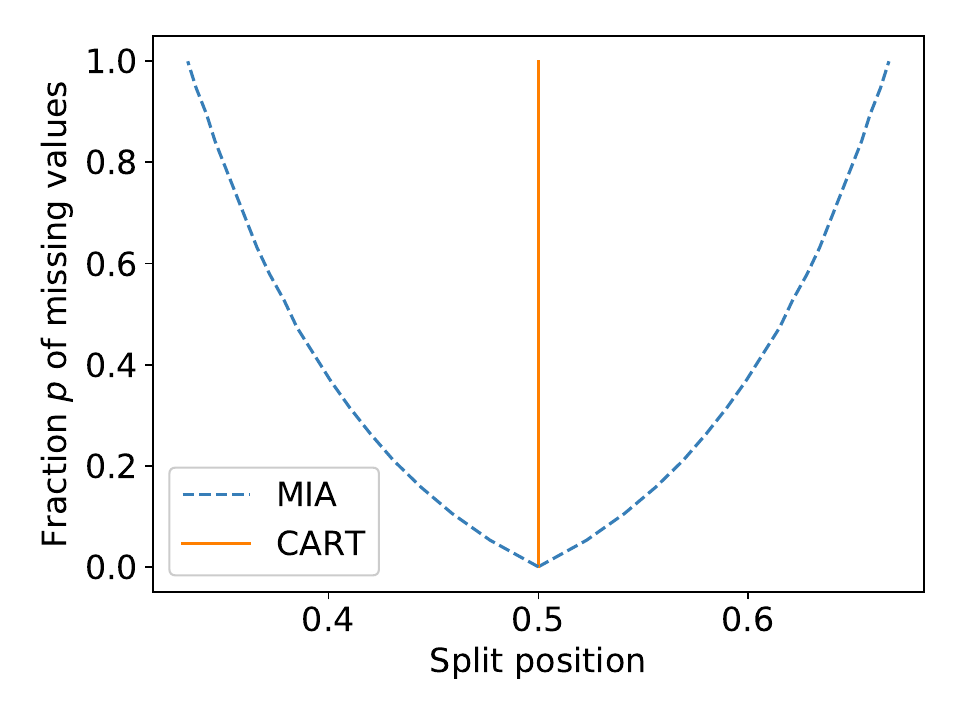}
    \end{minipage}
\end{figure}

The proof is given in Appendix \ref{apx:split}. Proposition \ref{prop_criterion} shows that the split given by optimizing the CART criterion does not depend on the percentage $p$ of missing values since the pattern is independent of $(X,Y)$. 
A numerical solution to equation \eqref{eq:propmia} is displayed in Figure \ref{fig:s_mia}. When there are no missing values ($p=0$), the split occur at $s=1/2$ as expected. 
When $p$ increases,  the threshold does not correspond anymore to the one
calculated using observed values only as it is influenced by the missing
entries  even in the MCAR setting. It may be surprising that the splitting position change in a MCAR setting but the missing values correspond to data from the whole interval $[0,1]$ and thus introduce noise in the side they are sent to. This implies that the cell that receives missing values must be bigger than usual so that the noise of the missing data is of the same order than the noise of original observations belonging to the cell. 
Recall that, since the threshold in MIA is chosen by taking into account missing values, it is straightforward to propagate a new element with missing values down the tree.


%

Recall that the quadratic risk $R$ of a function $f^{\star}$ is defined as 
$R(f^{\star}) = \E [(Y - f^{\star}(X))^2]$.
Proposition \ref{thm:comp_tree_theo} enables us to compare the risk of a tree with a 
single split computed with the different strategies. It highlights that even in the simple case of MCAR, MIA gives more accurate predictions than block propagation or probabilistic split.

\begin{proposition}
\label{thm:comp_tree_theo}
Consider the regression model
\begin{align*}
    \left\lbrace
    \begin{array}{llc}
       Y & = & X_1\\
       X_1 & \sim & U ([0,1])\\ 
       X_2 & = & X_1 \mathds{1}_{W=1}
    \end{array}
    \right. 
    ,\quad 
    \left\lbrace
    \begin{array}{llc}
       \P[W=0] & = & \eta\\
       \P[W=1] & = & 1 - \eta 
    \end{array}
    \right. ,
    \quad 
    \left\lbrace
    \begin{array}{llc}
       \P[M_1=0] & = & 1-p\\
       \P[M_1=1] & = & p,
    \end{array}
    \right. ,
\end{align*}
where $(M_1, W) \indep (X_1, Y)$. The random variable $M_1$ is the pattern of missingness for $X_1$ and $W$ stands for the link between $X_1$ and $X_2$. Let $f^{\star}_{\textrm{MIA}}$, $f^{\star}_{\textrm{block}}$, $f^{\star}_{\textrm{prob}}$, $f^{\star}_{\textrm{surr}}$ be respectively, the theoretical prediction resulting from one split according to MIA, CART with block propagation and CART with probabilistic splitting strategy, and a single split, where missing data are handled via surrogate split (in the infinite sample setting). We have  
\begin{align*}
 R(f^{\star}_{\textrm{MIA}}) & =  \underset{s\in[0,1]}{\min}~C_{\textrm{MIA}}(1, s, \textrm{L}, p) \mathds{1}_{p \leq  \eta} +   \underset{s\in[0,1]}{\min}~C_{\textrm{MIA}}(1, s, \textrm{L}, \eta) \mathds{1}_{p > \eta}, \\
R(f^{\star}_{\textrm{block}}) & = C_{\textrm{MIA}}(1, 1/2, L, p) = C_{\textrm{MIA}}(1, 1/2, R, p)\\   R(f^{\star}_{\textrm{prob}}) & = - \frac{p^2}{16} + \frac{p}{8} + \frac{1}{48},\\
R(f^{\star}_{\textrm{surr}}) & = \frac{1}{48} + \frac{6}{48} \eta p.
\end{align*}
where $C_{\textrm{MIA}}(1, s, \textrm{L}, p)$ is defined in Proposition~\ref{prop_criterion}. In particular,
\begin{align*}
R(f^{\star}_{\textrm{MIA}}) \leq R(f^{\star}_{\textrm{block}}) \quad \textrm{and} \quad
R(f^{\star}_{\textrm{MIA}}) \leq R(f^{\star}_{\textrm{prob}}).
\end{align*}
\end{proposition}

\begin{figure}[h!]
    \centering
    \includegraphics[width=\textwidth]{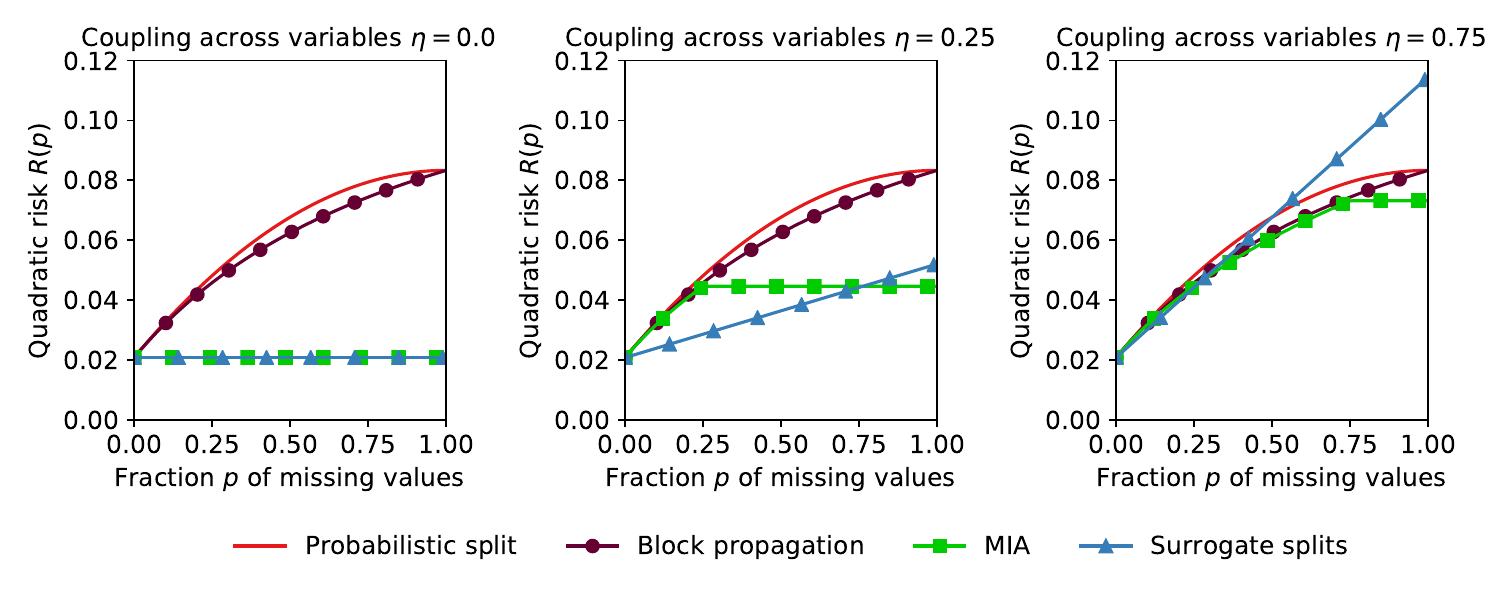}
    \caption{Theoretical risk of the splitting methods as a function
of $p$, for three values of $\eta$ parameter that controls the amount of
coupling between $X_1$ and $X_2$ in the model of Proposition~\ref{thm:comp_tree_theo}.}
    \label{fig:exp2}
\end{figure}

Proof is given in Appendix \ref{apx:split}.
Figure \ref{fig:exp2} depicts the risk of each estimate, in the context
of proposition \ref{thm:comp_tree_theo}, resulting from a split computed via one of the four methods described above. Only  surrogate and MIA risks depend on the value $\eta$ which measures the independence between $X_1$ and $X_2$. As proved, the risk of probabilistic split and block propagation is larger than that of MIA. Besides, surrogate split is better than MIA if the link between $X_1$ and $X_2$ is strong (small values of $\eta$) and worse if this link is weak (high values of $\eta$).

\subsection{Splitting algorithms}
\label{subsec:app:algo}

\begin{algorithm}[htbp]
  \small
  \caption{\small Probabilistic split}
  \label{alg-probabilistic-split}   
  \begin{algorithmic}[1]
    \STATE \textbf{Input:} a node $A$, a sample $\dataset_n$ of observations falling into $A$, a split $(j_0^{\star}, z_0^{\star})$ produced by Algorithm~\ref{alg-split-missdata}
    \STATE Compute the number $n_L$ of points  with observed $X_{j_0^{\star}}$ falling into $A_L$.
    \STATE Compute the number $n_R$ of points  with observed $X_{j_0^{\star}}$ falling into $A_R$.
    \FOR{all data with missing value along $j_0^{\star}$}
    \STATE Send the data randomly to $A_L$ (resp. $A_R$) with probability $n_L/(n_L + n_R)$ (resp. $n_R/(n_L + n_R)$)
    \ENDFOR 
  \end{algorithmic}
\end{algorithm}

\begin{algorithm}[htbp]
  \small
  \caption{\small Block propagation}
  \label{alg-block-propagation}   
  \begin{algorithmic}[1]
    \STATE \textbf{Input:} a node $A$, a sample $\dataset_n$ of observations falling into $A$, a split $(j_0^{\star}, z_0^{\star})$ produced by Algorithm~\ref{alg-split-missdata}
    \STATE Consider sending all observations with missing values on $j_0^{\star}$ into $A_L$. Compute the corresponding error (criterion on the right-hand side in \ref{split_criterion_cart})
    \STATE Consider sending all observations with missing values on $j_0^{\star}$ into $A_R$. Compute the corresponding error (criterion on the right-hand side in \ref{split_criterion_cart}).
    \STATE Choose the alternative with the lowest error and send all missing data on the same side accordingly. 
\end{algorithmic}
\end{algorithm}

\begin{algorithm}[htbp]
  \small
  \caption{\small Missing Incorporated in Attribute (MIA)}
  \label{alg-MIA}   
  \begin{algorithmic}[1]
    \STATE \textbf{Input:} a node $A$, a sample $\dataset_n$ of observations falling into $A$
    \FOR{each split $(j,z)$}
    \STATE Send all observations with missing values on $j$ on the left side. Compute the error $\varepsilon_{j,z,L}$ (right-hand side in \ref{split_criterion_cart})
    \STATE Send all observations with missing values on $j$ on the right side. Compute the error $\varepsilon_{j,z,R}$ (right-hand side in \ref{split_criterion_cart}) 
    \ENDFOR 
    \FOR{each $j \in \{1, \hdots, d\}$}
    \STATE Compute the error $\varepsilon_{j, - , - }$ associated to separating observations with missing data on $j$ from the remaining ones. 
    \ENDFOR
    \STATE Choose the split corresponding to the lowest error  $\varepsilon_{\cdot, \cdot,\cdot}$. Split the data accordingly. 
\end{algorithmic}
\end{algorithm}

\subsection{Proof of Proposition \ref{prop_criterion}}

\textbf{Cart splitting criterion.}
Under the model given in Proposition \ref{prop_criterion}, simple calculations show that 
\begin{align*}
    \E[Y|X\in [0,s]] & = \frac{s}{2}, \quad 
    \E[Y^2 | X \in [0,s]] = \frac{s^2}{3}\\
    \E[Y | X \in [s,1]]  & = \frac{1+s}{2}, \quad 
    \E[Y^2|X\in [s,1]] = \frac{1 - s^3}{3(1-s)} \\
    \P[X \in [0,s]] & = s, \quad    \P[X \in [s,1]] = 1- s.
\end{align*}
Thus the CART spltting criterion can be written as
\begin{align*}
C(1,s) = ~& \E[Y^2] - (\P[X \in [0,s]](\E[Y|X\in [0,s]])^2+
    \P[X \in [s,1]] (\E[Y | X \in [s,1]])^2 )\\
    = ~& \frac{1}{3} - \Big( s \left( \frac{s}{2}\right)^2 + (1-s) \left( \frac{1+s}{2}\right)^2\Big)\\
    = ~& \frac{s (s-1)}{4}  + \frac{1}{12}.
\end{align*}
By definition, 
\begin{align*}
s^\star
&=\underset{s\in[0,1]}{\mathrm{argmin}}~\left( \frac{1}{4} s (s-1) + \frac{1}{12} \right) = 1/2,
\end{align*}
and the criterion evaluated in $s=1/2$ is equal to $1/48$. The calculations are exactly the same when a percentage of missing value is added if $M_1 \indep X_1$.

\textbf{MIA splitting criterion.} By symmetry, we can assume than missing values are sent left. It is equivalent to observing
$$
X' = 0 \mathds{1}_{M=1} + X \mathds{1}_{M=0}.
$$
The MIA splitting criterion is then defined as 
\begin{align*}
s^\star_{\textrm{MIA}, \textrm{L}}
&=\underset{s\in[0,1]}\argmin~
\E\left[\left(Y
-\E[Y|X'\leq s]\mathds 1_{X'\leq s}
-\E[Y|X'>s]\mathds 1_{X'>s}\right)^2\right] \\
&=\underset{s\in[0,1]}\argmin~
\mathbb P(X'\leq s)\E\left[\left(Y
-\E[Y|X'\leq s]\right)^2\middle| X'\leq s\right] \\
& \quad +
\mathbb P(X'> s)\E\left[\left(Y
-\E[Y|X'> s]\right)^2\middle| X'>s\right].
\end{align*}

We have
\begin{align*}
\mathds{E}[Y | X' \in [0,s] ] & = \mathds{E}[X | X' \in [0,s] ] \\
& = \mathds{E}[X \mathds{1}_{M=1} + X \mathds{1}_{M=0} | X' \in [0,s] ] \\
& = \frac{1}{\mathbb P[X' \in [0,s]]} \mathds{E}[X \mathds{1}_{M=1, X' \in [0,s]} + X \mathds{1}_{M=0, X' \in [0,s]} ] \\
& = \frac{1}{p + (1-p)s} \Big( \frac{p}{2} + \frac{(1-p) s^2}{2}\Big).
\end{align*}
Besides,
\begin{align*}
\mathds{E}[Y^2 | X' \in [0,s] ] & = \mathds{E}[X^2 | X' \in [0,s] ] \\
& = \mathds{E}[X^2 \mathds{1}_{M=1} + X^2 \mathds{1}_{M=0} | X' \in [0,s] ] \\
& = \frac{1}{p + (1-p)s} \mathds{E}[X^2 \mathds{1}_{M=1, X' \in [0,s]} + X^2 \mathds{1}_{M=0, X' \in [0,s]} ] \\
& = \frac{1}{p + (1-p)s} \Big( \frac{p}{3} + \frac{(1-p)s^3}{3}\Big) \\
\end{align*}
Thus the left-part of the criterion is given by
\begin{align*}
& \mathbb P(X' \in [0,s]) \E [(Y - \E[Y | X' \in [0,s]] )^2 |X' \in [0,s]] \\
& = \Big(p + (1-p)s \Big) \Big(\E[Y^2 | X' \in [0,s]] - (\E[Y | X' \in [0,s]])^2\Big)\\
& = \Big(p + (1-p)s \Big) \Big(\frac{1}{p + (1-p)s} \Big( \frac{p}{3} + \frac{(1-p) s^3}{3}\Big) \\
& \quad - \Big(\frac{1}{p + (1-p)s} \Big( \frac{p}{2} + \frac{(1-p) s^2}{2}\Big) \Big)^2\Big)\\
 &= \Big( \frac{p}{3} + \frac{(1-p) s^3}{3}\Big) - \frac{1}{p + (1-p)s}  \Big( \frac{p}{2} + \frac{(1-p) s^2}{2}\Big)^2
\end{align*}

\bigskip

On the other hand, we have
\begin{align*}
\mathds{E}[Y | X' \in [s,1] ] & = \mathds{E}[X | X' \in [s,1] ] \\
& = \mathds{E}[X \mathds{1}_{M=1} + X \mathds{1}_{M=0} | X' \in [s,1] ] \\
& = \frac{1}{(1-p)(1-s)} \mathds{E}[X \mathds{1}_{M=1, X' \in [s,1]} + X \mathds{1}_{M=0, X' \in [s,1]} ] \\
& = \frac{1}{(1-p)(1-s)} \Big( (1-p) \frac{1-s^2}{2} \Big)\\
& = \frac{1+s}{2}.
\end{align*}
Besides,
\begin{align*}
\mathds{E}[Y^2 | X' \in [s,1] ] & = \mathds{E}[X^2 | X' \in [s,1] ] \\
& = \mathds{E}[X^2 \mathds{1}_{M=1} + X^2 \mathds{1}_{M=0} | X' \in [s,1] ] \\
& = \frac{1}{(1-p)(1-s)} \mathds{E}[X^2 \mathds{1}_{M=1, X' \in [s,1]} + X^2 \mathds{1}_{M=0, X' \in [s,1]} ] \\
& = \frac{1}{(1-p)(1-s)} \Big( (1-p) \frac{1 - s^3}{3}\Big) \\
& = \frac{1- s^3}{3(1-s)}.
\end{align*}
Thus the right-part of the criterion is given by
\begin{align*}
& \mathbb P(X' \in [s,1]) \E [(Y - \E[Y | X' \in [s,1]] )^2 |X' \in [s,1]] \\
& = \Big((1-p)(1-s)\Big) \Big(\E[Y^2 | X' \in [s,1]] - (\E[Y | X' \in [s,1]])^2\Big)\\
& = \Big((1-p)(1-s)\Big) \Big(\frac{1- s^3}{3(1-s)} - (\frac{1+s}{2})^2\Big)\\
&= (1-p)\frac{1- s^3}{3} - (1-p)(1-s)\Big(\frac{1+s}{2}\Big)^2.\\
\end{align*}
Finally,
\begin{align*}
s^\star_{\textrm{MIA}, \textrm{L}} = \underset{s\in[0,1]}{\argmin}~
 \Bigg\lbrace \Big( & \frac{p}{3} + \frac{(1-p) s^3}{3}\Big) - \frac{1}{p + (1-p)s}  \Big( \frac{p}{2} + \frac{(1-p) s^2}{2}\Big)^2 \\
 & +
(1-p)\frac{1- s^3}{3} - (1-p)(1-s)\Big(\frac{1+s}{2}\Big)^2\Bigg\rbrace,
\end{align*}
which concludes the proof.

\subsection{Proof of proposition \ref{thm:comp_tree_theo}}

\textbf{Probabilistic and block propagation}. 
First, note that the variable $ X_2  = X_1 \mathds{1}_{W=1}$ is similar to the variable studied for the computation of the MIA criterion in Proposition~\ref{prop_criterion}.  Therefore, the value of the CART splitting criterion along the first variable is $C_{\textrm{MIA}}(1, 1/2, \textrm{L}, 0)$ and its value along the second variable is $C_{\textrm{MIA}}(2, s^{\star}_{\textrm{MIA}, \textrm{L}}, \textrm{L}, \eta)$. Since the function 
$$\alpha \mapsto C_{\textrm{MIA}}(\cdot, s^{\star}_{\textrm{MIA}, \textrm{L}}, \textrm{L},  \alpha)$$ is increasing, splitting along the first variable leads to the largest variance reduction. Thus, for probabilistic and block propagation, splits occur along the first variable. Let us now compare the value of these criteria. We have
\begin{align*}
    \P[X_1 \leq 1/2] = \P[X_1 \geq 1/2] = 1/2.
\end{align*}
The quantities related to the left cell are given by 
\begin{align*}
    \E[Y | X_1\leq 1/2] & = \frac{p+1}{4} \quad \textrm{and} \quad \E[Y^2  | X_1\leq 1/2]  = \frac{p}{4} + \frac{1}{12}.
\end{align*}
The quantities related to the right cell are given by 
\begin{align*}
    \E[Y | X_1\geq 1/2] & = \frac{3-p}{4} \quad \textrm{and} \quad     \E[Y^2  | X_1\geq 1/2]  = \frac{7}{12} - \frac{p}{4}.
\end{align*}
Thus, the value of the criterion satisfies
\begin{align*}
    R(f^{\star}_{\textrm{prob}}) &= \frac{1}{2} \Big( \frac{p}{4} + \frac{1}{12} - \left(\frac{p+1}{4} \right)^2\Big)\\
    & \quad + \frac{1}{2} \Big( \frac{7}{12} - \frac{p}{4} - \left(\frac{3-p}{4} \right)^2\Big)\\
    & = - \frac{p^2}{16} + \frac{p}{8} + \frac{1}{48}.
\end{align*}

Let, for all $p \in [0,1]$, 
\begin{align*}
    h(p) & = R(f^{\star}_{\textrm{prob}}) - R(f^{\star}_{\textrm{block}})\\
    & = - \frac{p^2}{16} + \frac{p}{8} + \frac{1}{48} - \left( - \frac{11}{48} + \frac{1}{8} \frac{3p+2}{2p+1}\right)\\
    & = - \frac{p^2}{16} + \frac{p}{8} + \frac{1}{16} - \frac{1}{16} \frac{1}{2p+1}.
\end{align*}
We have, 
\begin{align*}
    h'(p) = - \frac{p}{8} + \frac{1}{8} +  \frac{1}{8(2p+1)^2},
\end{align*}
and consequently, 
\begin{align*}
    h''(p)= - \frac{1}{8} - \frac{1}{2(2p+1)^3}.
\end{align*}
An inspection of the variation of $h$ reveals that $h(p) \geq 0$ for all $p \in [0,1]$, which concludes the first part of the proof.  

\medskip 

\textbf{MIA.} As noticed above, the criterion computed along the second variable is given by 
\begin{align*}
    C_{\textrm{MIA}}(2, s^{\star}_{\textrm{MIA}, \textrm{L}}, \textrm{L}, \eta)
\end{align*}
Since the function $$\alpha \mapsto C_{\textrm{MIA}}(\cdot, s^{\star}_{\textrm{MIA}, \textrm{L}}, \textrm{L},  \alpha)$$ is increasing, MIA split will occur along the first variable if $p \leq \eta$ and along the second variable if $p \geq \eta$. Therefore, the risk of the MIA splitting procedure is given by 
\begin{align*}
    R(f^{\star}_{\textrm{MIA}}) & =  \underset{s\in[0,1]}{\min}~C_{\textrm{MIA}}(1, s, \textrm{L}, p) \mathds{1}_{p \leq \eta} +   \underset{s\in[0,1]}{\min}~C_{\textrm{MIA}}(1, s, \textrm{L}, \eta) \mathds{1}_{p > \eta}.
\end{align*}

\textbf{Surrogate split.}
Consider the model $Y = X_1$ and $X_2 = X_1 \mathds{1}_{W=1}$, where 
$\P[W = 0] = \eta$.
Let us determine the best split along $X_2$ to predict $Z = \mathds{1}_{X_1 < 0.5}$. Since $\{ X_2 \leq s \} = \{X_1 \leq s, W = 1\} \cup \{ W=0\}$, and $\{X_2 > s\} = \{ X_1 > s, W = 1\}$, 
\begin{align*}
\P[X_2 \leq s ] & = s (1-p) + p \quad \textrm{and} \quad \P[X_2 > s ] = (1-s) (1-p).
\end{align*}
Consequently, 
\begin{align*}
\E[Z | X_2 \leq s ] & = \frac{\E[\mathds{1}_{X_1 \leq 0.5, X_2 \leq s}]}{\P[X_2 \leq s ]}\\
& = \frac{1}{s(1-p) +p} \E[\mathds{1}_{X_1 \leq 0.5, X_1 \leq s, W=1}+ \mathds{1}_{X_1 \leq 0.5, W=0}]\\
& =  \frac{1}{s(1-p) +p} \big[ (1-p) \min(0.5, s) + \frac{p}{2} \big].
\end{align*}

\begin{align*}
\E[Z | X_2 \geq s ] & = \frac{\E[\mathds{1}_{X_1 \leq 0.5, X_2 > s}]}{\P[X_2 > s ]}\\
& = \frac{1}{(1-s)(1-p)} \P[X_1 \leq 0.5, W=1, X_1 \geq s]\\
& =  \frac{(0.5 - s)_+}{1-s}.
\end{align*}
Besides, note that $\E[Z^2] = \P[X_1 \leq 0.5] = 0.5$. Therefore, the splitting criterion to predict $\mathds{1}_{X_1 \leq 0.5}$ with $X_2$ is given by 
\begin{align*}
f(s) & = \frac{1}{2} - \P[X_2 \leq s ] ( \E[Z | X_2 \leq s ])^2 - \P[X_2 > s] (\E[Z | X_2 > s])^2\\
& = \frac{1}{2} - \frac{1}{s(1-p) + p} \Big( (1-p) \min(0.5, s) + \frac{p}{2} \Big)^2 - \frac{1-p}{1-s} ((0.5 - s)_+)^2.
\end{align*}
For $s \geq 1/2$, 
\begin{align*}
    h(s) & = \frac{1}{2} - \frac{1}{4 (s(1-p) +p)},
\end{align*}
which is minimal for $s=1/2$. For $s \leq 1/2$, 
\begin{align*}
h(s) & = \frac{1}{2} - \frac{1}{4} \left( \frac{p^2}{p + s(1-p)} + \frac{1-p}{1-s} \right).
\end{align*}
Hence, 
\begin{align*}
h'(s) & = - \frac{1-p}{4} \frac{(1-2p) s^2 + 2ps}{(1-s)^2 (s(1-p) + p)^2}.
\end{align*}
Let $g(s) = (1-2p)s^2 + 2ps$. If $p \leq 1/2$, the solutions of $g(s)=0$ are negative, thus, $g(s)\geq 0$ for all $s \in [0,1/2]$ and thus the minimum of $h$ is reached at $s=1/2$.
If $p \geq 1/2$, one solution of $g(s)=0$ is zero and the other is $s = 2p/(2p-1) > 1$. Thus,  $g(s)\geq 0$ for all $s \in [0,1/2]$ and the minimum of $h$ is reached at $s=1/2$.
Finally, the minimum of $h$ is reached at $s=1/2$. 
The risk of the surrogate estimate is then given by
\begin{align*}
R(f^{\star}_{\textrm{surr}}) &= \E[(Y- f^{\star}_{\textrm{surr}}(\bX))^2] \\
& = \E[(Y- f^{\star}_{\textrm{surr}}(\bX))^2 \mathds{1}_{M_1=0} + (Y- f^{\star}_{\textrm{surr}}(\bX))^2 \mathds{1}_{M_1=1}].
\end{align*}
Here, 
\begin{align*}
& \E[(Y - f^{\star}_{\textrm{surr}}(\bX))^2 | M_1=1]\\
= ~& \E[(X_1 - 0.25)^2 \mathds{1}_{X_2 < 0.5} + (X_1 - 0.75)^2 \mathds{1}_{X_2 \geq 0.5}]\\
= ~& \eta \E[(X_1 - 0.25)^2] + (1 - \eta) \E[(X_1 - 0.25)^2 \mathds{1}_{X_1 \leq 0.5}] \\
& \quad + ( 1- \eta) \E[(X_1 - 0.75)^2 \mathds{1}_{X_1 > 0.5}]\\
= ~& \frac{1}{48} + \frac{6\eta }{48}.
\end{align*}
Finally, 
\begin{align*}
R(f^{\star}_{\textrm{surr}}) & = \frac{1-p}{48} + p \Big( \frac{1}{48} + \frac{6\eta }{48} \Big) = \frac{1}{48} + \frac{6}{48} \eta p.
\end{align*}

\section{Miscellaneous}

\subsection{Variable selection properties of the tree methods with missing values} \label{sec:selection}

Decision trees based on the CART criterion (implemented in the R library rpart) and on  conditional trees (implemented in the the R library partykit) lead to different ways of selecting splitting variables. We illustrate this behaviour on the simple following model:

\begin{align*}
\left\{
    \begin{array}{rl}
        X_1\indep X_2 &\sim \mathcal N(0,1) \\
        \varepsilon &\sim \mathcal N(0, 1) \\
        Y &= 0.25 X_1 + \varepsilon.
    \end{array}
\right.
\end{align*}

We insert MCAR  values, either on the first variable or on both variables. Stumps (decision trees of depth one) are fit on 500 Monte-Carlo repetitions. We vary the sample size and the percentage of missing values. Figure \ref{fig:ratio_mcariid} show that CART and conditional trees give similar results when there are missing values on both variables. However, Figure \ref{fig:ratio_mcarx1} shows that CART has a tendency to underselect $X_1$ when there are missing values only on $X_1$. For instance, for a sample of size 50 with 75\% missing values, CART selects the non-informative variable $X_2$ more frequently than $X_1$, while conditional trees keep selecting $X_1$ more often.

\begin{figure}[hbtp]
    \centering
    \begin{subfigure}{0.35\textwidth}
        \includegraphics[width=\textwidth]{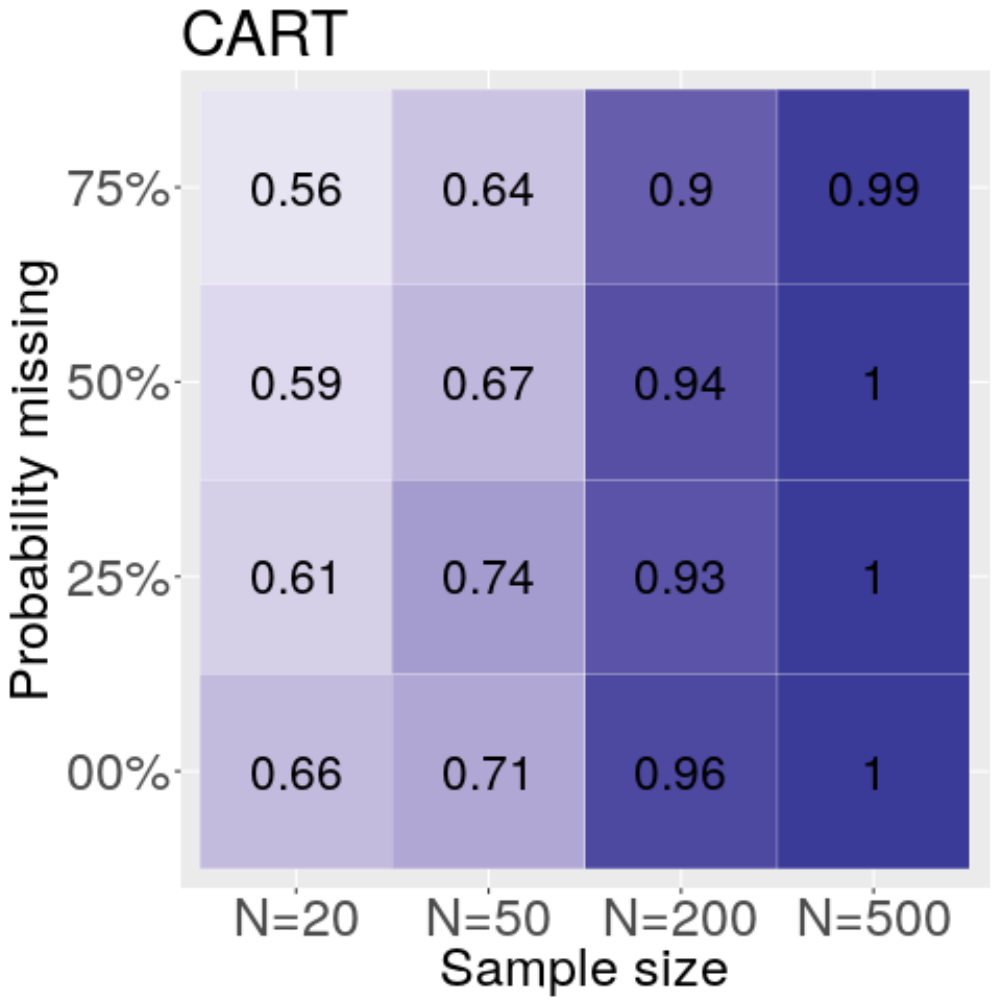}
        \caption{CART}
    \end{subfigure}%
    \begin{subfigure}{0.35\textwidth}
        \includegraphics[width=\textwidth]{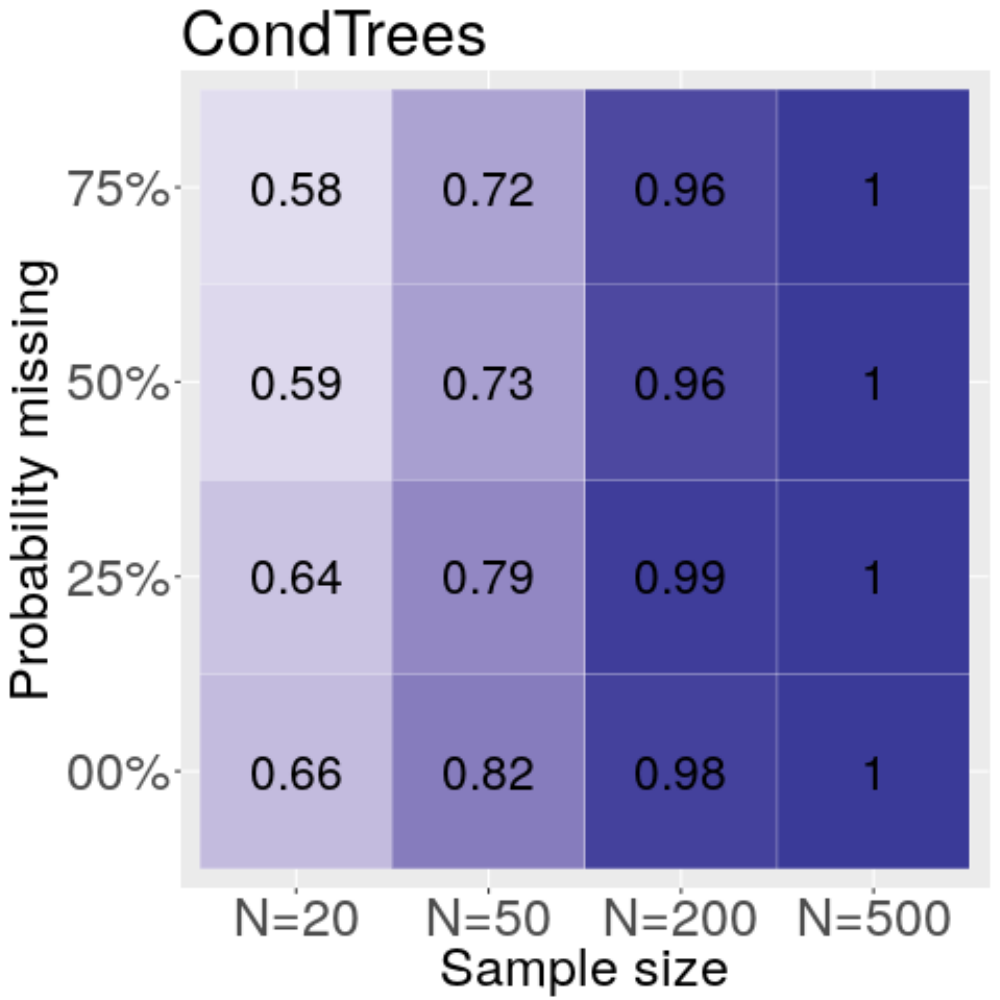}
        \caption{Conditional trees}
    \end{subfigure}
    \caption{Frequency of selection of $X_1$ when there are missing values on $X_1$ and $X_2$}
    \label{fig:ratio_mcariid}
\end{figure}


\begin{figure}[hbtp]
    \centering
    \begin{subfigure}{0.35\textwidth}
        \includegraphics[width=\textwidth]{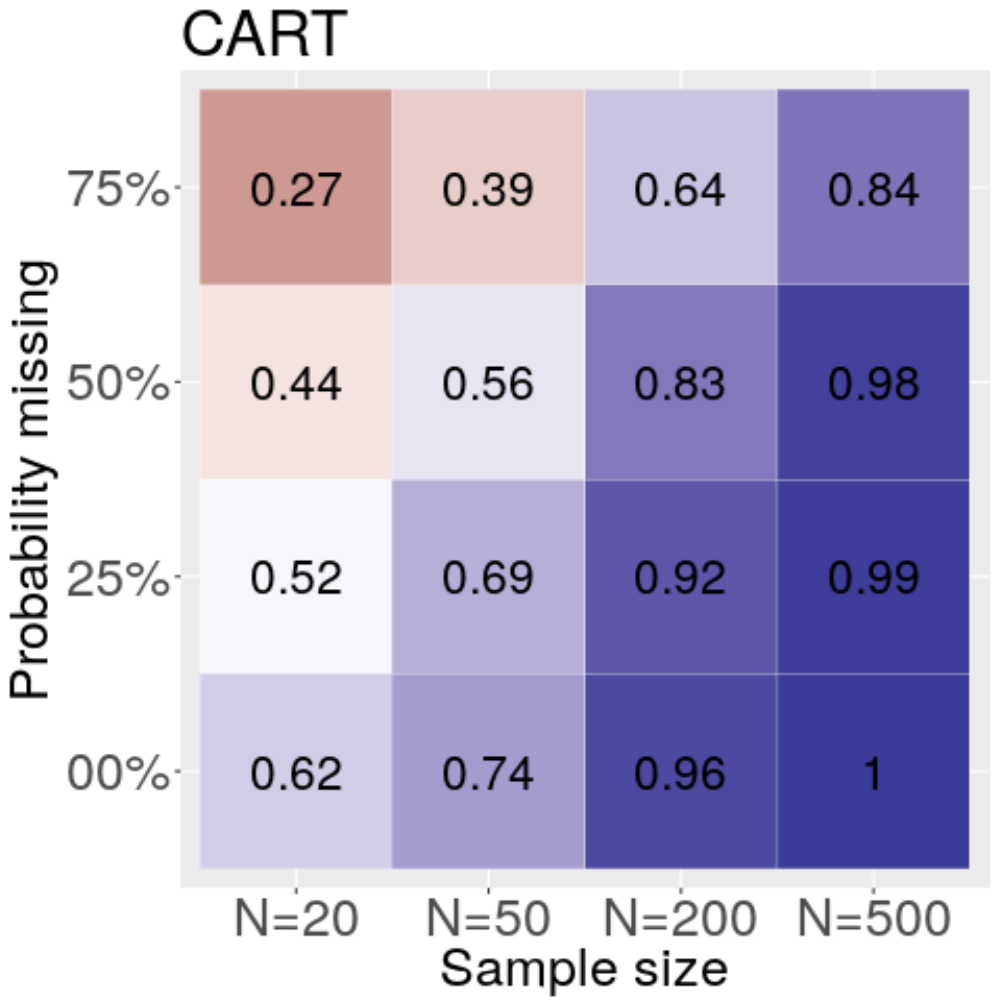}
        \caption{CART}
    \end{subfigure}%
    \begin{subfigure}{0.35\textwidth}
        \includegraphics[width=\textwidth]{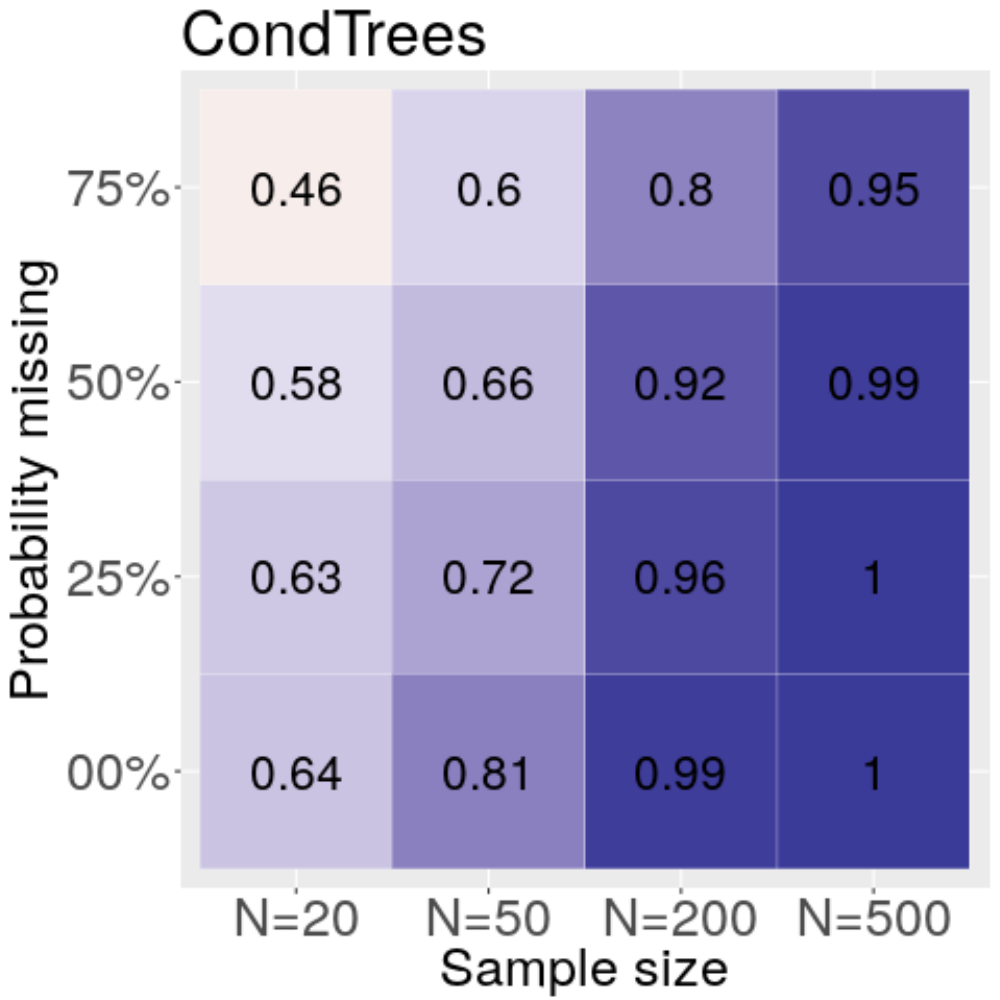}
        \caption{Conditional trees}
    \end{subfigure}
    \caption{Frequency of selection of $X_1$ when there are missing values on $X_1$ only}
    \label{fig:ratio_mcarx1}
\end{figure}

\subsection{Example of EM algorithm}\label{sec:exem}

Let us consider a simple case of $n$ observations $(\mathbf x_1, \mathbf x_2)=(x_{i1},x_{i2})_{1\leq i\leq n}$ sampled from the distribution of $(X_1,X_2)$, a bivariate Gaussian distribution with parameters $(\mu, \Sigma)$. We assume that $X_2$ is subjected to missing values and that only $r$ values are observed. The aim is to get the maximum likelihood estimates of $(\mu, \Sigma)$ from the incomplete data set. The algorithm described below can be straightforwardly extended to the multivariate case. 
Note that from $(\hat \mu, \hat \Sigma)$, it is then possible to directly estimate the parameters of a linear regression model and thus to perform linear regression with missing values. 

We denote by $f_{1,2}(\mathbf x_1,\mathbf x_2;\mu, \Sigma)$, $f_1(\mathbf x_1;\mu_1, \sigma_{11})$ and $f_{2|1}(\mathbf x_2|\mathbf x_1; \mu, \Sigma)$, respectively, the probability of joint distribution of $(X_1,X_2)$, marginal distribution of $X_1$ and conditional distribution of $X_2|X_1$. The joint distribution of observed data can be decomposed as:
\[
f_{1,2}(\mathbf x_1,\mathbf x_2;\mu, \Sigma)=\prod_{i=1}^n f_1(x_{i1};\mu_1, \sigma_{11})\prod_{j=1}^r f_{2|1}(x_{j2}|x_{j1}; \mu, \Sigma),
\]

and the observed log-likelihood is written (up to an additional constant that does not appear in the maximization and that we therefore drop):

$$\ell(\mu, \Sigma;\mathbf x_1,\mathbf x_2)=
-\frac{n}{2}\log(\sigma_{11}^2)-\frac{1}{2}\sum_{i=1}^n\frac{(x_{i1}-\mu_1)^2}{\sigma_{11}^2}-\frac{r}{2}\log \left((\sigma_{22}-\frac{\sigma_{12}^2}{\sigma_{11}})^2\right)$$
$$-\frac{1}{2}\sum_{i=1}^r\frac{\left(x_{i2}-\mu_2-\frac{\sigma_{12}}{\sigma_{11}}(x_{i1}-\mu_1)\right)^2}{(\sigma_{22}-\frac{\sigma_{12}^2}{\sigma_{11}})^2}$$

We skip the computations and directly give the expression of the closed form maximum likelihood estimates of the mean:
$$  {\hat{\mu}}_1=n^{-1}\sum_{i=1}^nx_{i1} $$
$$  \hat{\mu}_2=\hat{\beta}_{20.1}+\hat{\beta}_{21.1}\hat{\mu}_1,$$
where
$$\hat{\beta}_{21.1}=s_{12}/s_{11}, \quad \hat{\beta}_{20.1}=\bar{x}_2-\hat{\beta}_{21.1}\bar{x}_1,$$ $$\bar{x}_j=r^{-1}\sum_{i=1}^rx_{ij} \quad \text{and} \quad s_{jk}=r^{-1}\sum_{i=1}^r(x_{ij}-\bar{x}_j)(x_{ik}-\bar{x}_k), \quad j,k=1,2.$$

In this simple setting, we have an explicit expression of the maximum likelihood estimator despite missing values. However, this is not always the case but it is possible to use an EM algorithm to get the maximum likelihood estimators in the cases where data are missing.

The EM algorithm consists in maximizing the observed likelihood through successive maximization of the complete likelihood (if we had observed all $n$ realizations of $\mathbf x_1$ and $\mathbf x_2$). Maximizing the complete likelihood 
$$\ell_c(\mu, \Sigma;\mathbf x_1,\mathbf x_2)=-\frac{n}{2}\log\left(\det(\Sigma)\right)-\frac{1}{2}\sum_{i=1}^n(x_{i1}-\mu_1)^T\Sigma^{-1}(x_{i1}-\mu_1)$$

would be straightforward if we had all the observations. However elements
of this likelihood are not available. Therefore, we replace them by the conditional expectation given observed data and the parameters of the current iteration. These two steps of computation of the conditional expectation (E-step) and maximization of the completed likelihood (M step) are repeated until convergence. 
The update formulas for the E and M steps are as follows:

\textbf{E step}:
The sufficient statistics of the likelihood are:  
$$s_1=\sum_{i=1}^nx_{i1},\quad s_2=\sum_{i=1}^nx_{i2},\quad s_{11}=\sum_{i=1}^nx_{i1}^2,\quad s_{22}=\sum_{i=1}^nx_{i2}^2,\quad s_{12}=\sum_{i=1}^nx_{i1}x_{i2}.$$

Since some values of $\mathbf \mathbf x_2$ are not available, we fill in the sufficient statistics with:
$$\E[x_{i2}|x_{i1};\mu,\Sigma]=\beta_{20.1}+\beta_{21.1}x_{i1}$$
$$\E[x_{i2}^2|x_{i1};\mu,\Sigma]=(\beta_{20.1}+\beta_{21.1}x_{i1})^2+\sigma_{22.1}$$
$$\E[x_{i2}x_{i2}|x_{i1};\mu,\Sigma]=(\beta_{20.1}+\beta_{21.1}x_{i1})x_{i1}.$$
with,
 $\beta_{21.1}=\sigma_{12}/\sigma_{11}$, $\beta_{20.1}=\mu_2-\beta_{21.1}\mu_1$, and $\sigma_{22.1}=\sigma_{22}-\sigma^2_{12}/\sigma_{11}$.
 
\textbf{M step}: 
The M step consists in computing the maximum likelihood estimates as usual. 
Given $s_1, s_2, s_{11}, s_{22}, \text{and } s_{12},$ update $\hat{\mu}$ and $\hat{\sigma}$ with
$$\hat{\mu}_1=s_1/n\text{, }\hat{\mu}_2=s_2/n,$$
$$\hat{\sigma}_1=s_{11}/n-\hat{\mu}_1^2\text{, }\hat{\sigma}_2=s_{22}/n-\hat{\mu}_2^2\text{, }\hat{\sigma}_{12}=s_{12}/n-\hat{\mu}_1\hat{\mu}_2$$

Note that $s_1$, $s_{11}$, $\hat{\mu}_1$ and $\hat{\sigma}_1$ are constant across iterations since we do not have missing values on $\mathbf x_1$.

\begin{remark}
Note that EM imputes the sufficient statistics and not the data.
\end{remark}

\clearpage

\section{Additional experiments}
\label{sec:appendix_additional_experiments}

\subsection{Varying the correlation strength and the missing rate}

\begin{figure}[H]
    \centering
    \includegraphics[scale=0.45]{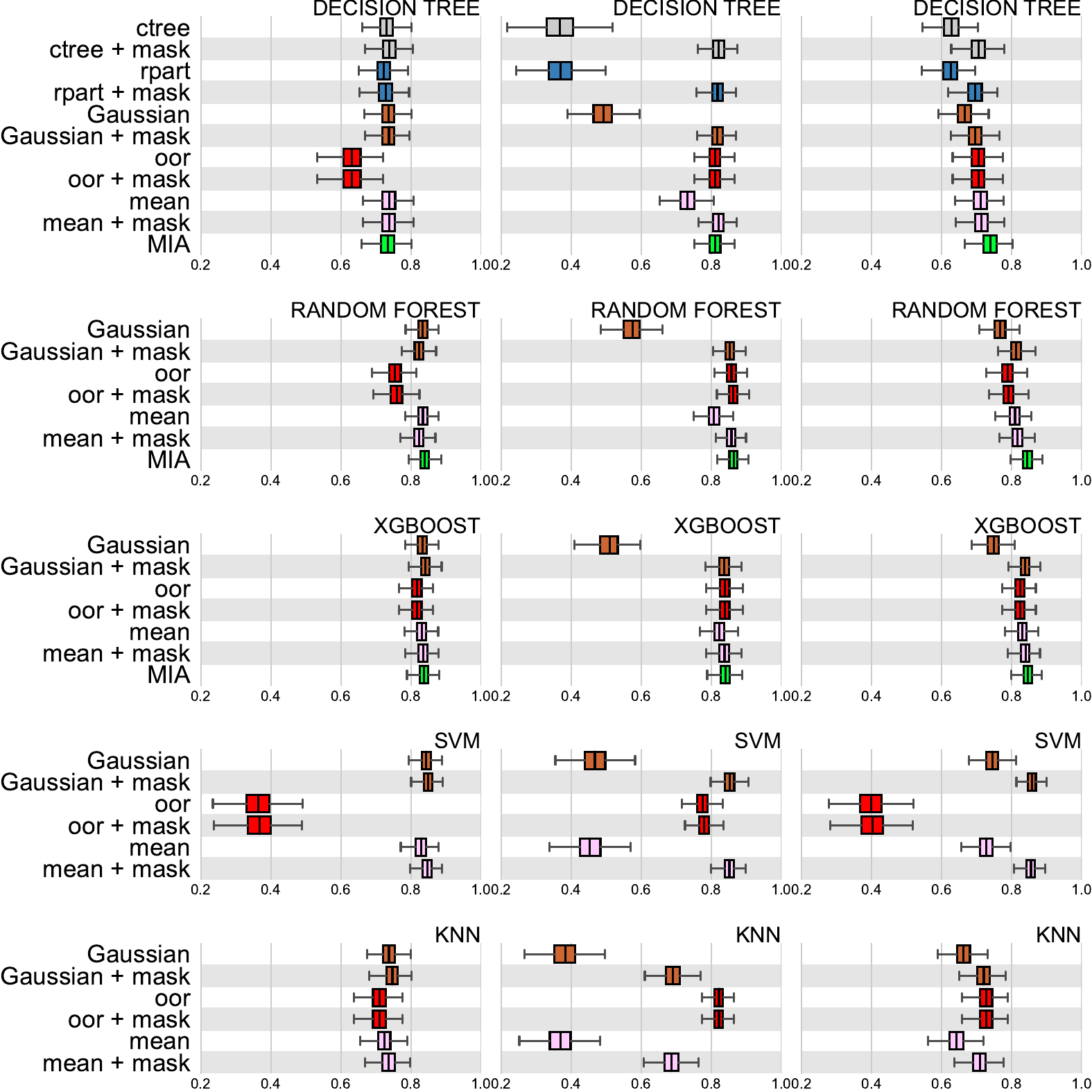}
    \llap{\raisebox{1.00\linewidth}{%
    {\sffamily%
    \hspace{.4\textwidth}MCAR
    \hspace{.2\textwidth}MNAR
    \hspace{.2\textwidth} Predictive M
    \hspace*{.1\textwidth}}}}

    \caption{\textbf{$R^2$ scores on model 1 $\bullet$} \modif{Normalized}
explained variance for different mechanisms with 20\% of missing values,
$n=1000$, $d=9$ and $\rho = 0.2$.}
    \label{fig:simu_mcar_nonlinear_multivar_miss2_rho2}
\end{figure}
\begin{figure}[H]
    \centering
    \includegraphics[scale=0.45]{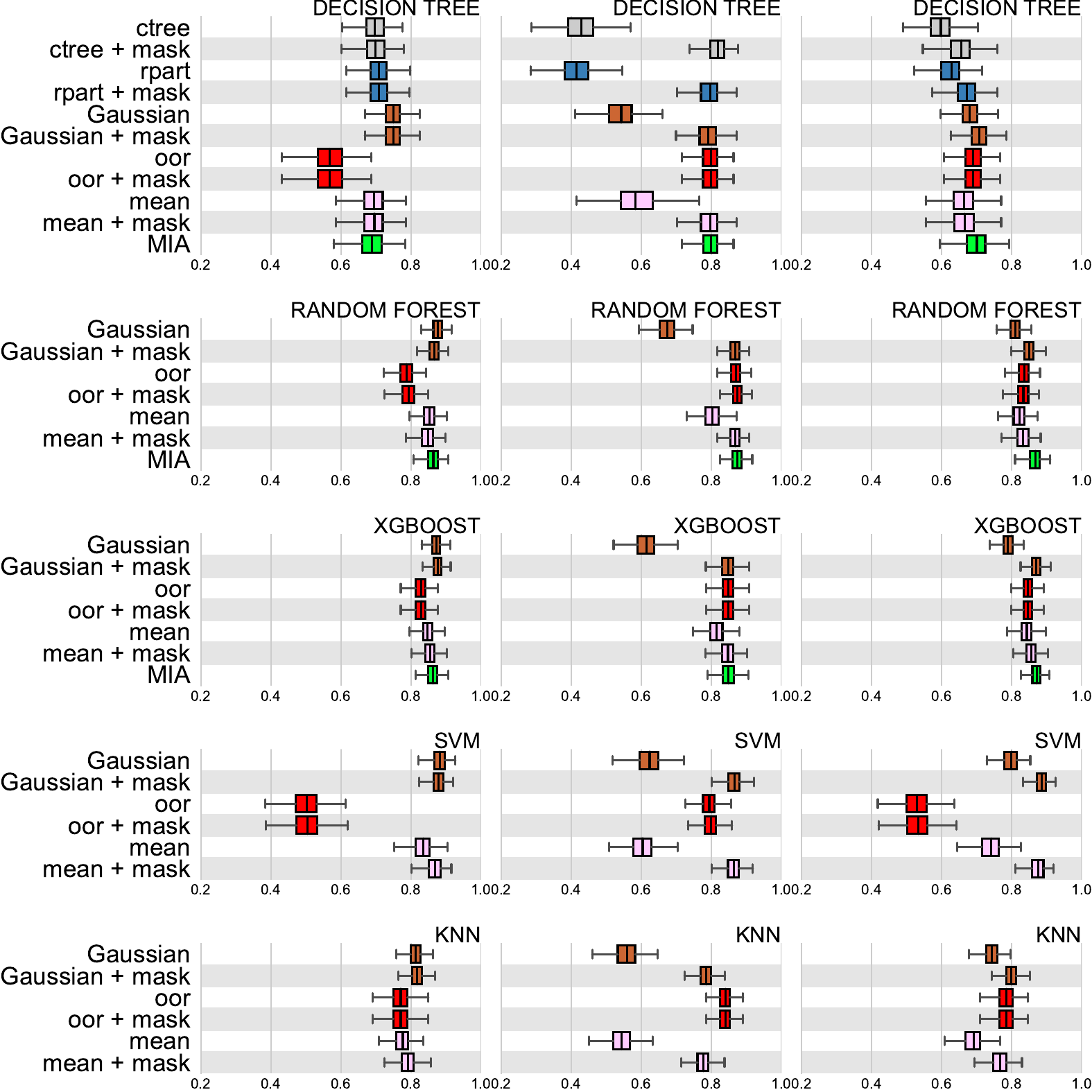}
    \llap{\raisebox{1.00\linewidth}{%
    {\sffamily%
    \hspace{.4\textwidth}MCAR
    \hspace{.2\textwidth}MNAR
    \hspace{.2\textwidth} Predictive M
    \hspace*{.1\textwidth}}}}

    \caption{\textbf{$R^2$ scores on model 1 $\bullet$} \modif{Normalized}
explained variance for different mechanisms with 20\% of missing values,
$n=1000$, $d=9$ and $\rho = 0.5$.}
    \label{fig:simu_mcar_nonlinear_multivar_miss2_rho5_appendix}
\end{figure}
\begin{figure}[H]
    \centering
    \includegraphics[scale=0.45]{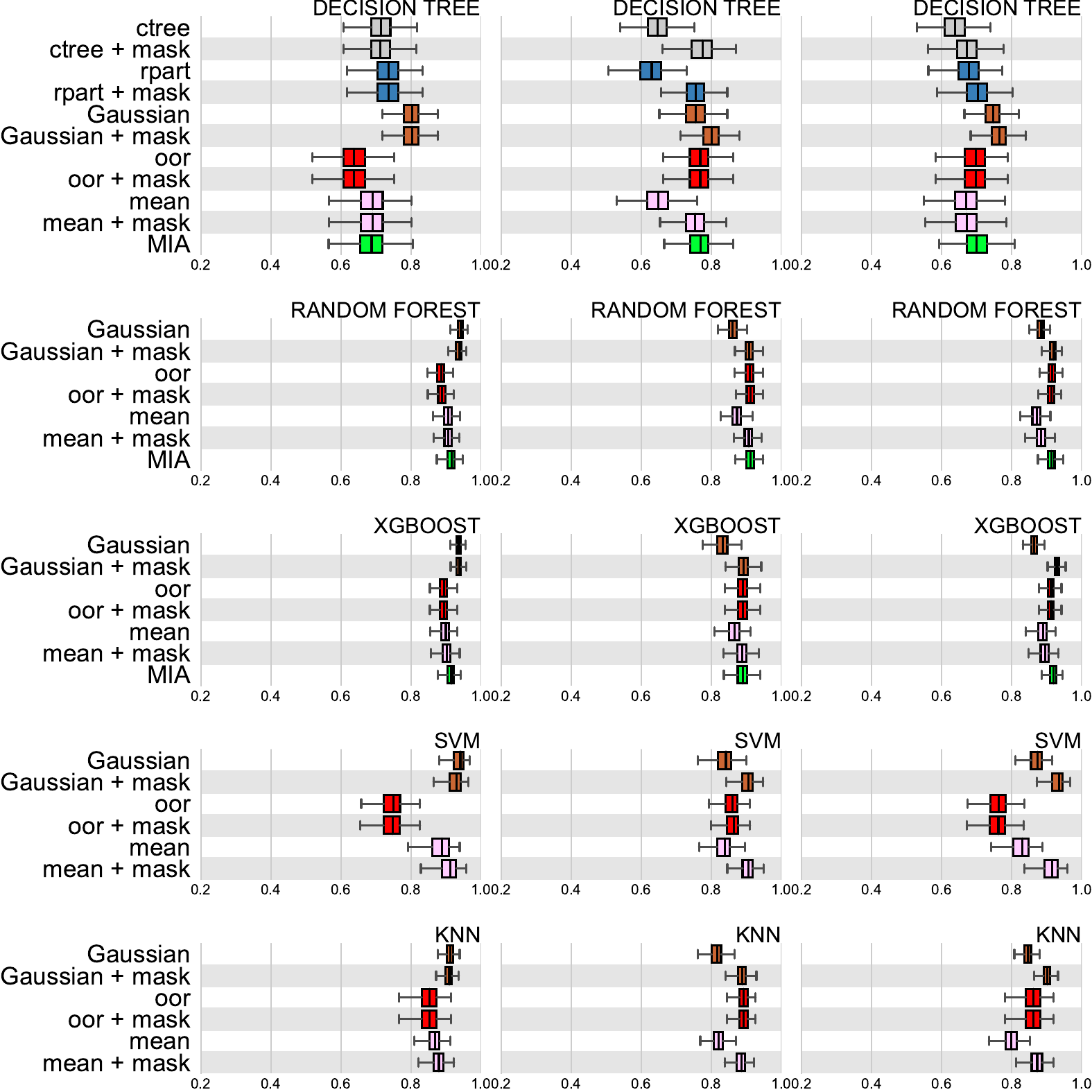}
    \llap{\raisebox{1.00\linewidth}{%
    {\sffamily%
    \hspace{.4\textwidth}MCAR
    \hspace{.2\textwidth}MNAR
    \hspace{.2\textwidth} Predictive M
    \hspace*{.1\textwidth}}}}

    \caption{\textbf{$R^2$ scores on model 1 $\bullet$} \modif{Normalized}
explained variance for different mechanisms with 20\% of missing values,
$n=1000$, $d=9$ and $\rho = 0.8$.}
    \label{fig:simu_mcar_nonlinear_multivar_miss2_rho8}
\end{figure}
\begin{figure}[H]
    \centering
    \includegraphics[scale=0.45]{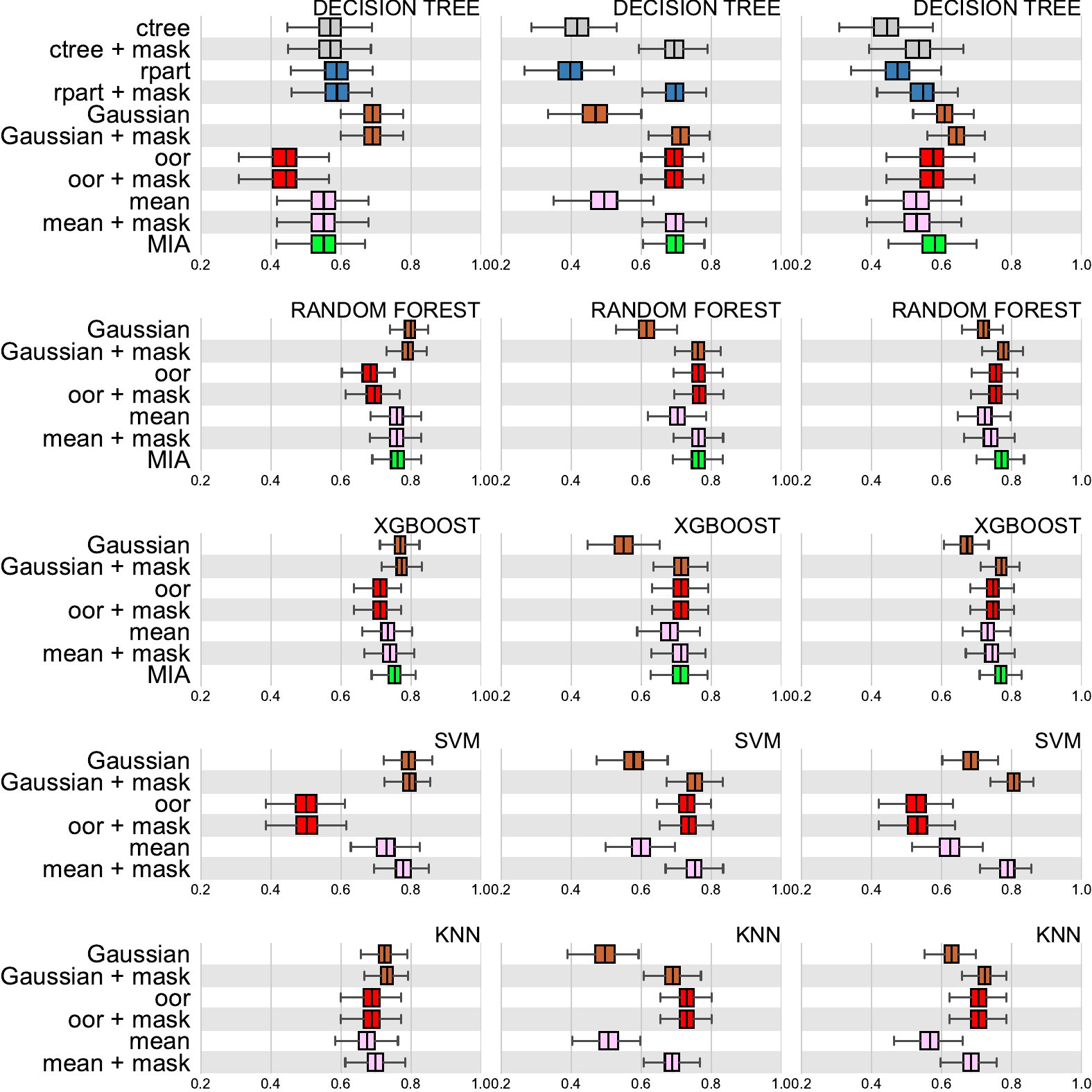}
    \llap{\raisebox{1.00\linewidth}{%
    {\sffamily%
    \hspace{.4\textwidth}MCAR
    \hspace{.2\textwidth}MNAR
    \hspace{.2\textwidth} Predictive M
    \hspace*{.1\textwidth}}}}

    \caption{\textbf{$R^2$ scores on model 1 $\bullet$} \modif{Normalized}
explained variance for different mechanisms with 40\% of missing values,
$n=1000$, $d=9$ and $\rho = 0.5$.}
    \label{fig:simu_mcar_nonlinear_multivar_miss4_rho5}
\end{figure}
\begin{figure}[H]
    \centering
    \includegraphics[scale=0.45]{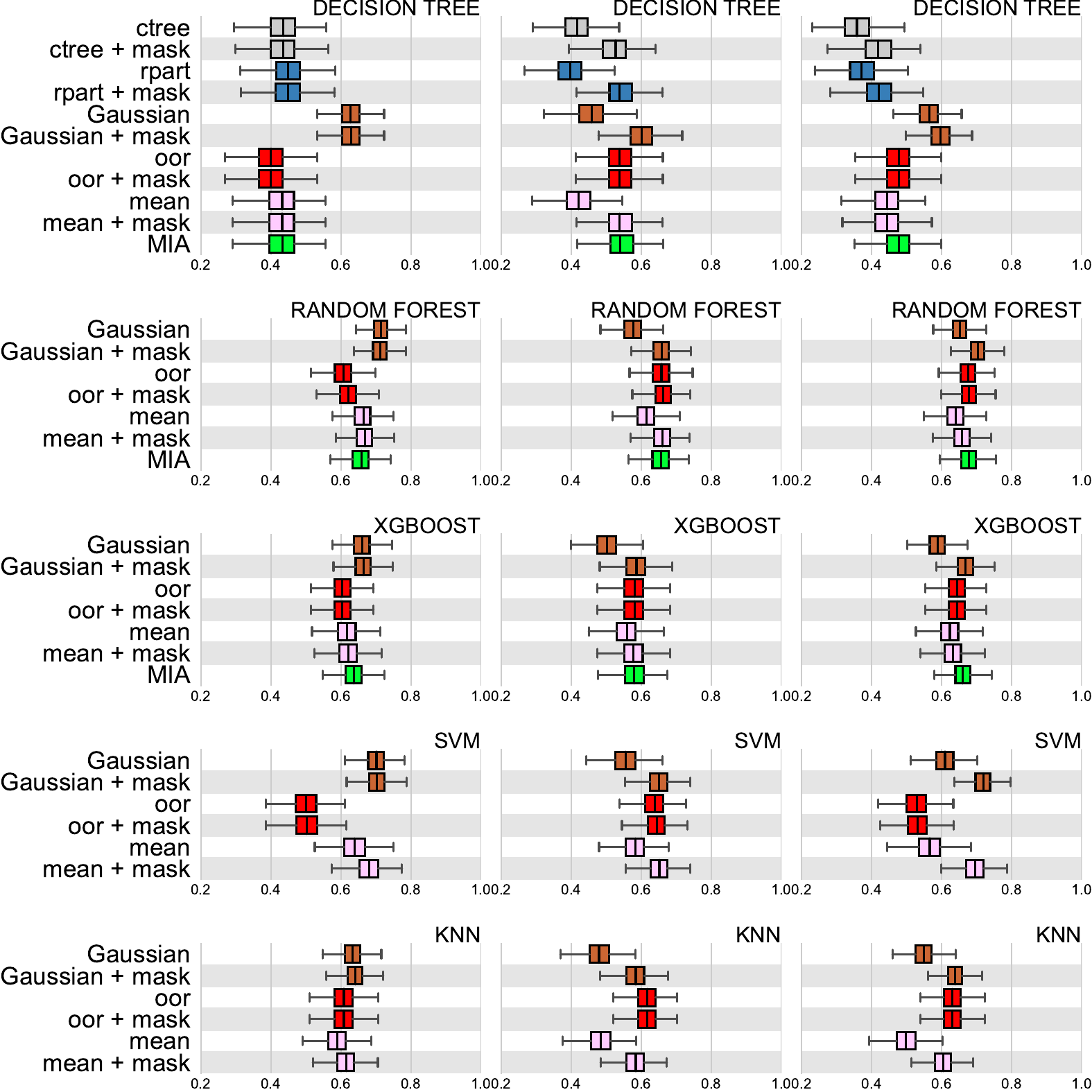}
    \llap{\raisebox{1.00\linewidth}{%
    {\sffamily%
    \hspace{.4\textwidth}MCAR
    \hspace{.2\textwidth}MNAR
    \hspace{.2\textwidth} Predictive M
    \hspace*{.1\textwidth}}}}

    \caption{\textbf{$R^2$ scores on model 1 $\bullet$} \modif{Normalized}
explained variance for different mechanisms with 60\% of missing values,
$n=1000$, $d=9$ and $\rho = 0.5$.}
    \label{fig:simu_mcar_nonlinear_multivar_miss6_rho5}
\end{figure}

\newpage

\subsection{Statistical tests}
\label{sec:app_statistical_tests}

\begin{figure}[H]
    \centering
\includepdf[scale=0.8, nup=2x3, pages=1-5]{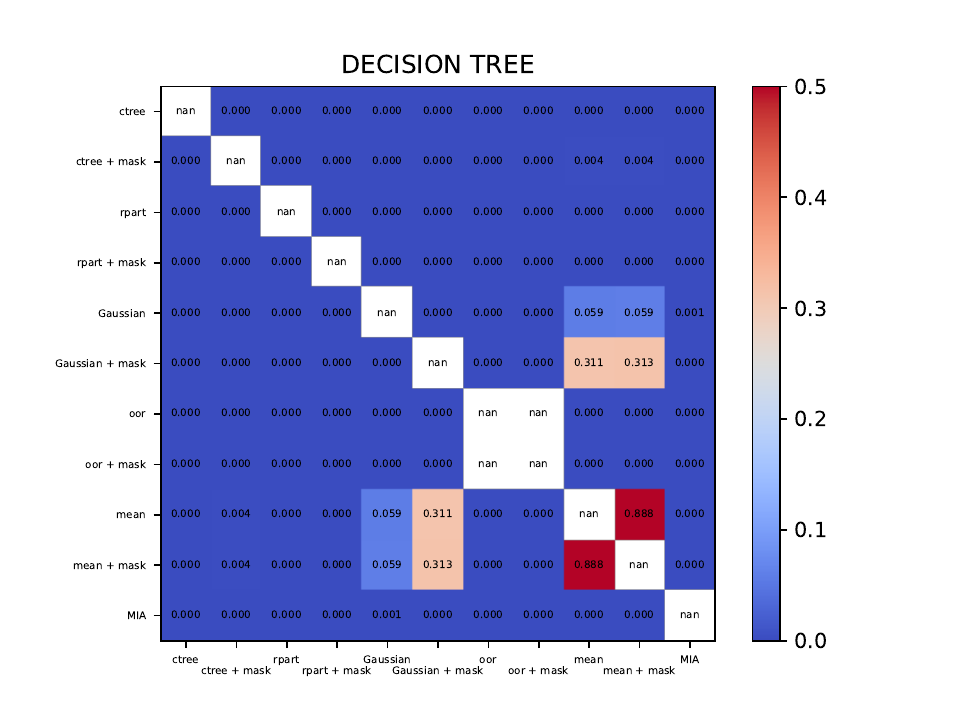}
\vspace*{-10mm}
\caption{P-values of paired t-tests for each pair of imputation method and for each learning algorithm in a MCAR model ($\rho=0.2$, missing rate of $20\%$)}
\label{fig:ttests_mcar_miss2_rho2}
\end{figure}

\newpage 

\begin{figure}[H]
    \centering
\includepdf[scale=0.8, nup=2x3, pages=1-5]{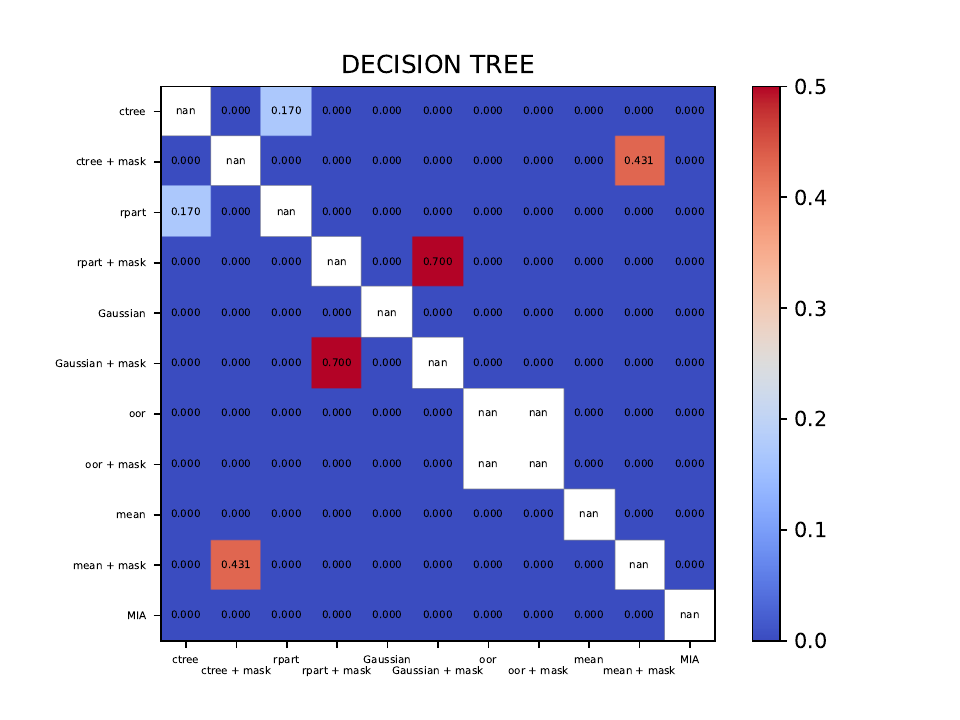}
\vspace*{-10mm}
\caption{P-values of paired t-tests for each pair of imputation method and for each learning algorithm in a MNAR model ($\rho=0.2$, missing rate of $20\%$)}
\label{fig:ttests_mnar_miss2_rho2}
\end{figure}

\newpage 

\begin{figure}[H]
    \centering
\includepdf[scale=0.8, nup=2x3, pages=1-5]{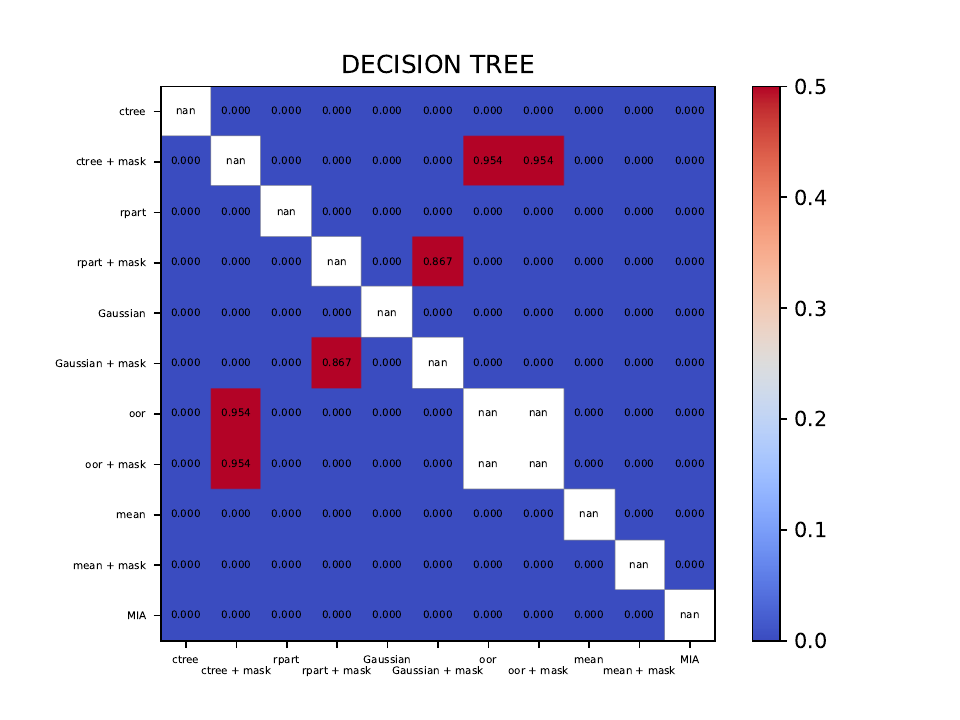}
\vspace*{-10mm}
\caption{P-values of paired t-tests for each pair of imputation method and for each learning algorithm in a Predictive M model ($\rho=0.2$, missing rate of $20\%$)}
\label{fig:ttests_pred_miss2_rho2}
\end{figure}

\newpage

\begin{figure}[H]
    \centering
\includepdf[scale=0.8, nup=2x3, pages=1-5]{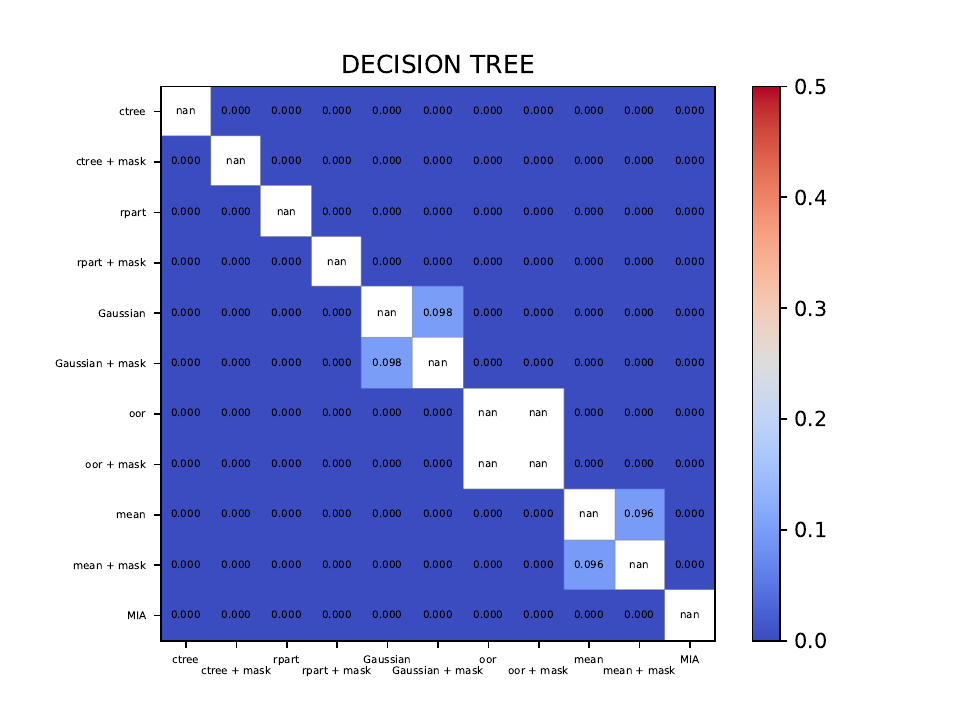}
\vspace*{-10mm}
\caption{P-values of paired t-tests for each pair of imputation method and for each learning algorithm in a MCAR model ($\rho=0.5$, missing rate of $20\%$)}
\label{fig:ttests_mcar_miss2_rho5}
\end{figure}

\newpage 

\begin{figure}[H]
    \centering
\includepdf[scale=0.8, nup=2x3, pages=1-5]{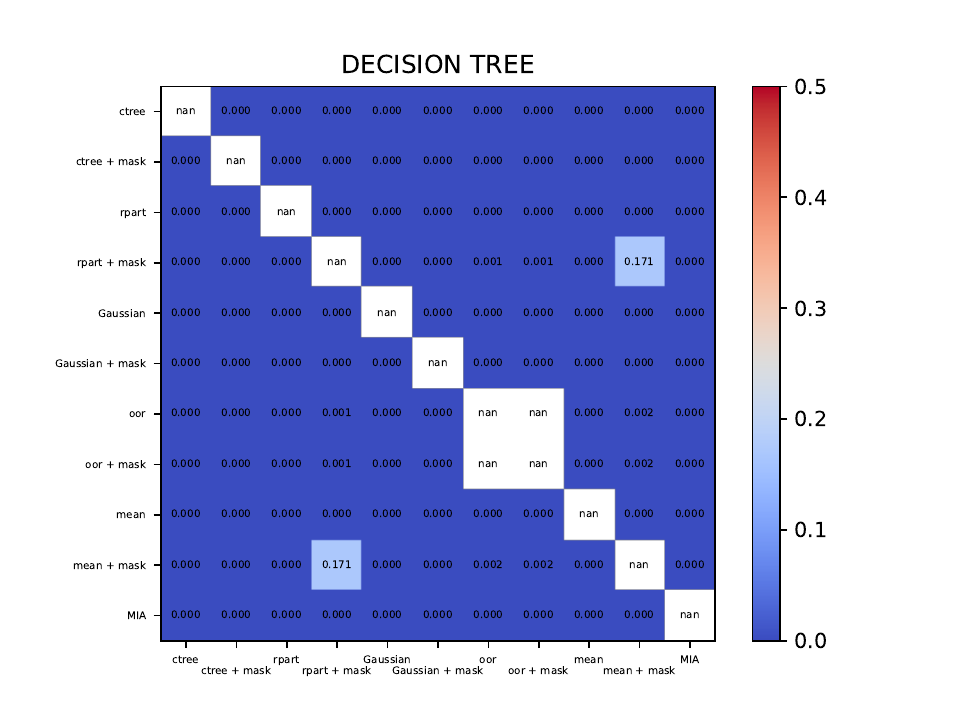}
\vspace*{-10mm}
\caption{P-values of paired t-tests for each pair of imputation method and for each learning algorithm in a MNAR model ($\rho=0.5$, missing rate of $20\%$)}
\label{fig:ttests_mnar_miss2_rho5}
\end{figure}

\newpage 

\begin{figure}[H]
    \centering
\includepdf[scale=0.8, nup=2x3, pages=1-5]{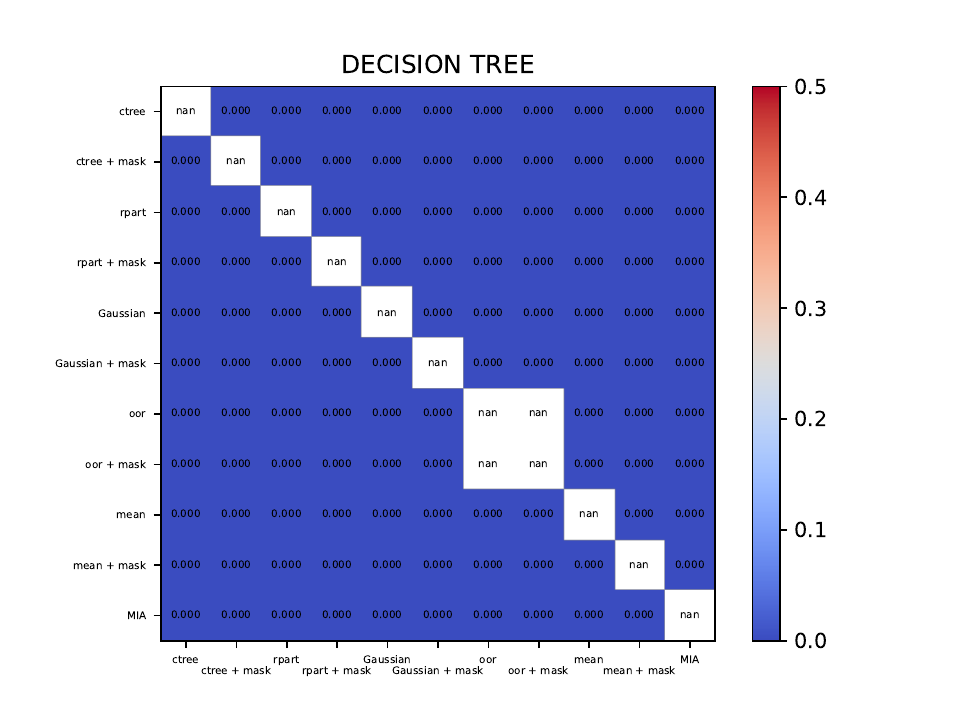}
\vspace*{-10mm}
\caption{P-values of paired t-tests for each pair of imputation method and for each learning algorithm in a Predictive M model ($\rho=0.5$, missing rate of $20\%$)}
\label{fig:ttests_pred_miss2_rho5}
\end{figure}

\newpage

\begin{figure}[H]
    \centering
\includepdf[scale=0.8, nup=2x3, pages=1-5]{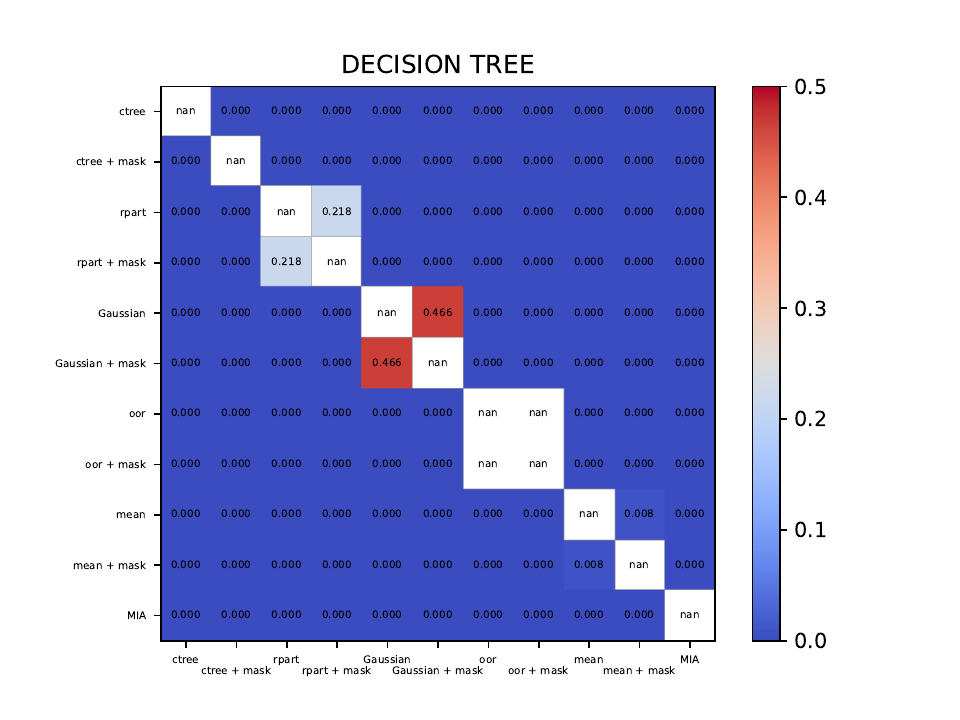}
\vspace*{-10mm}
\caption{P-values of paired t-tests for each pair of imputation method and for each learning algorithm in a MCAR model ($\rho=0.8$, missing rate of $20\%$)}
\label{fig:ttests_mcar_miss2_rho8}
\end{figure}

\newpage 

\begin{figure}[H]
    \centering
\includepdf[scale=0.8, nup=2x3, pages=1-5]{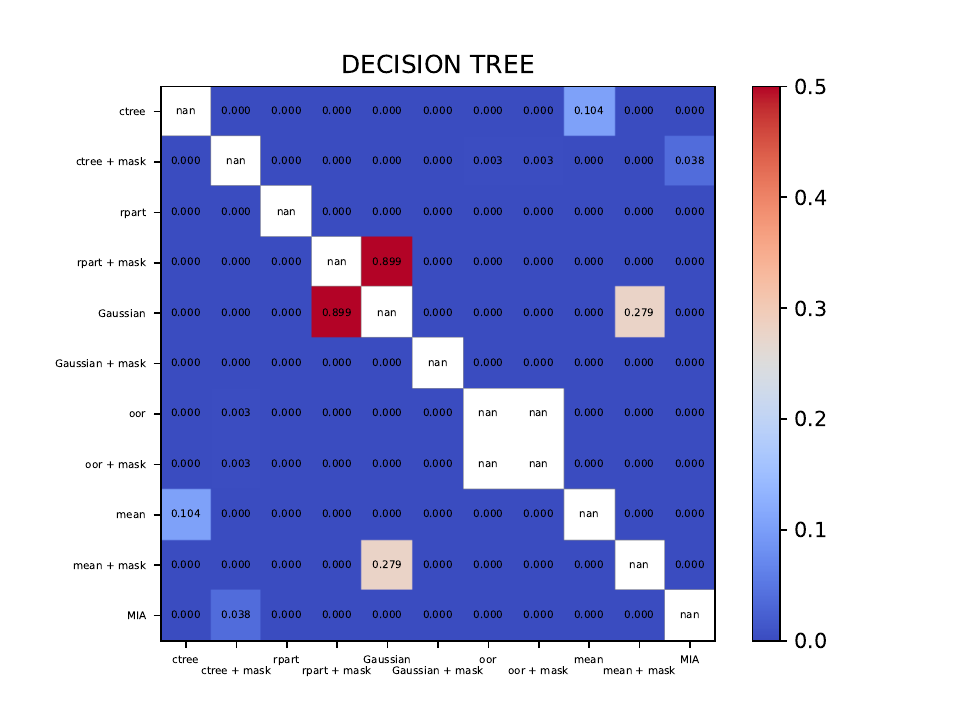}
\vspace*{-10mm}
\caption{P-values of paired t-tests for each pair of imputation method and for each learning algorithm in a MNAR model ($\rho=0.8$, missing rate of $20\%$)}
\label{fig:ttests_mnar_miss2_rho8}
\end{figure}

\newpage 

\begin{figure}[H]
    \centering
\includepdf[scale=0.8, nup=2x3, pages=1-5]{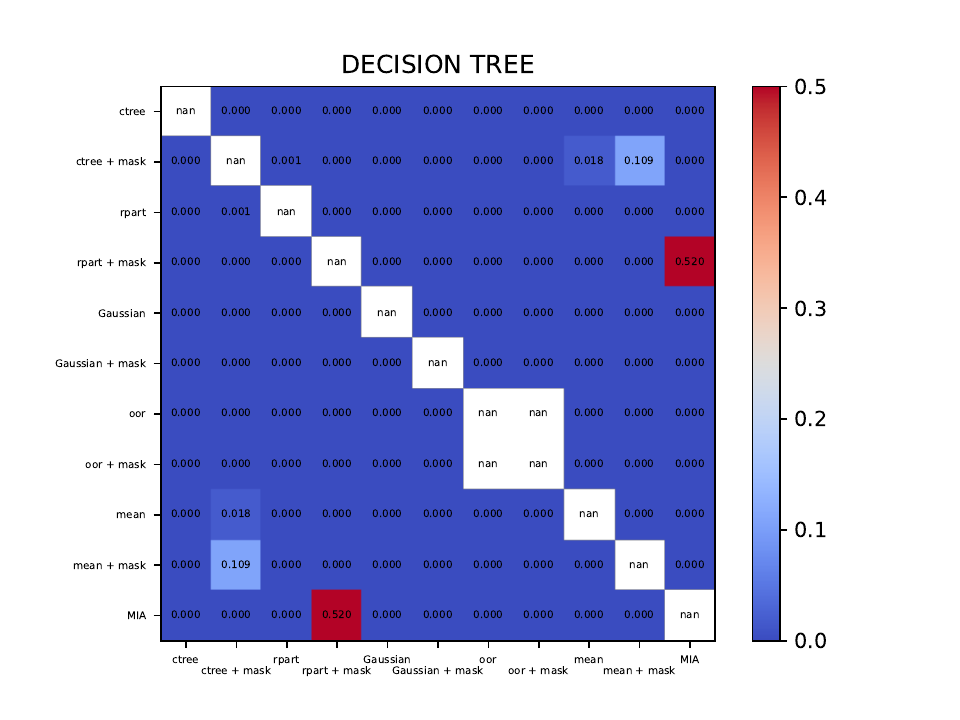}
\vspace*{-10mm}
\caption{P-values of paired t-tests for each pair of imputation method and for each learning algorithm in a Predictive M model ($\rho=0.8$, missing rate of $20\%$)}
\label{fig:ttests_pred_miss2_rho8}
\end{figure}

\newpage

\begin{figure}[H]
    \centering
\includepdf[scale=0.8, nup=2x3, pages=1-5]{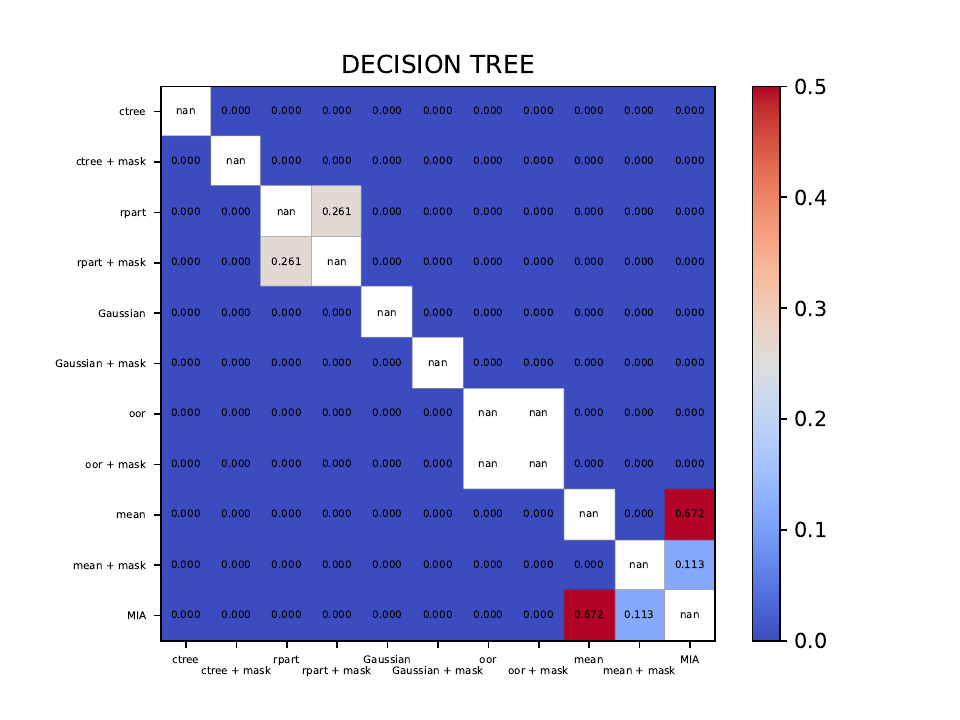}
\vspace*{-10mm}
\caption{P-values of paired t-tests for each pair of imputation method and for each learning algorithm in a MCAR model ($\rho=0.5$, missing rate of $40\%$)}
\label{fig:ttests_mcar_miss4_rho5}
\end{figure}

\newpage 

\begin{figure}[H]
    \centering
\includepdf[scale=0.8, nup=2x3, pages=1-5]{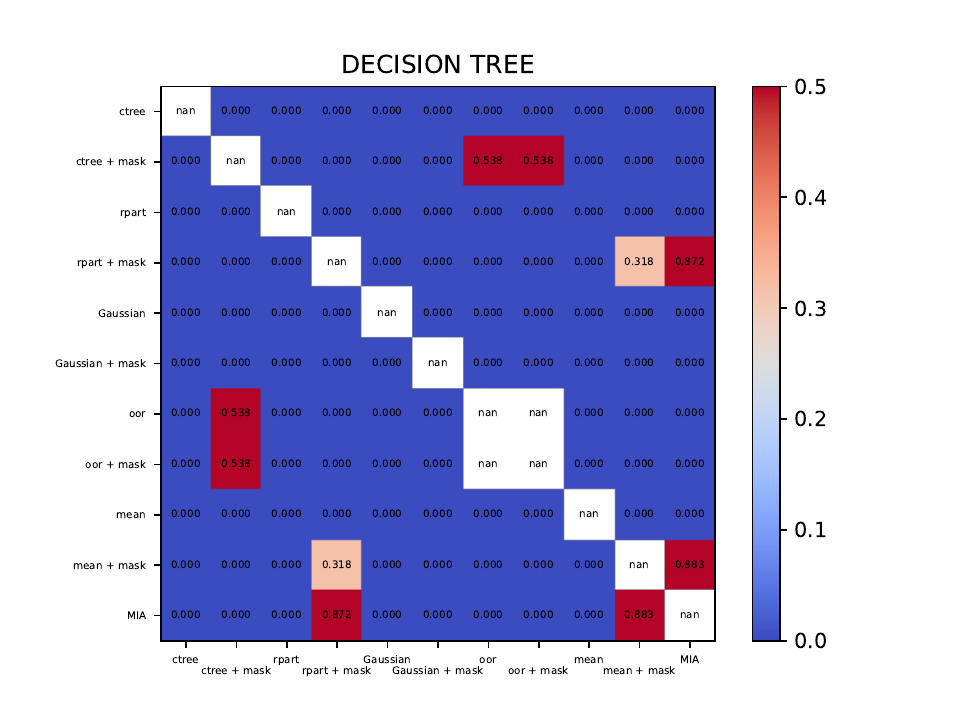}
\vspace*{-10mm}
\caption{P-values of paired t-tests for each pair of imputation method and for each learning algorithm in a MNAR model ($\rho=0.5$, missing rate of $40\%$)}
\label{fig:ttests_mnar_miss4_rho5}
\end{figure}

\newpage 

\begin{figure}[H]
    \centering
\includepdf[scale=0.8, nup=2x3, pages=1-5]{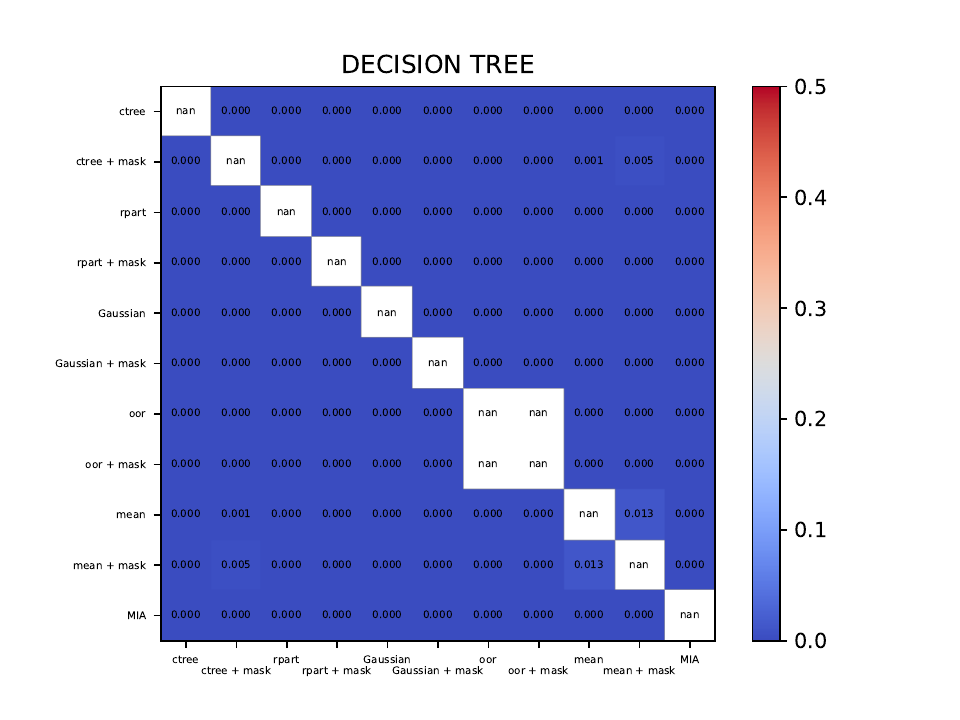}
\vspace*{-10mm}
\caption{P-values of paired t-tests for each pair of imputation method and for each learning algorithm in a Predictive M model ($\rho=0.5$, missing rate of $40\%$)}
\label{fig:ttests_pred_miss4_rho5}
\end{figure}

\newpage

\begin{figure}[H]
    \centering
\includepdf[scale=0.8, nup=2x3, pages=1-5]{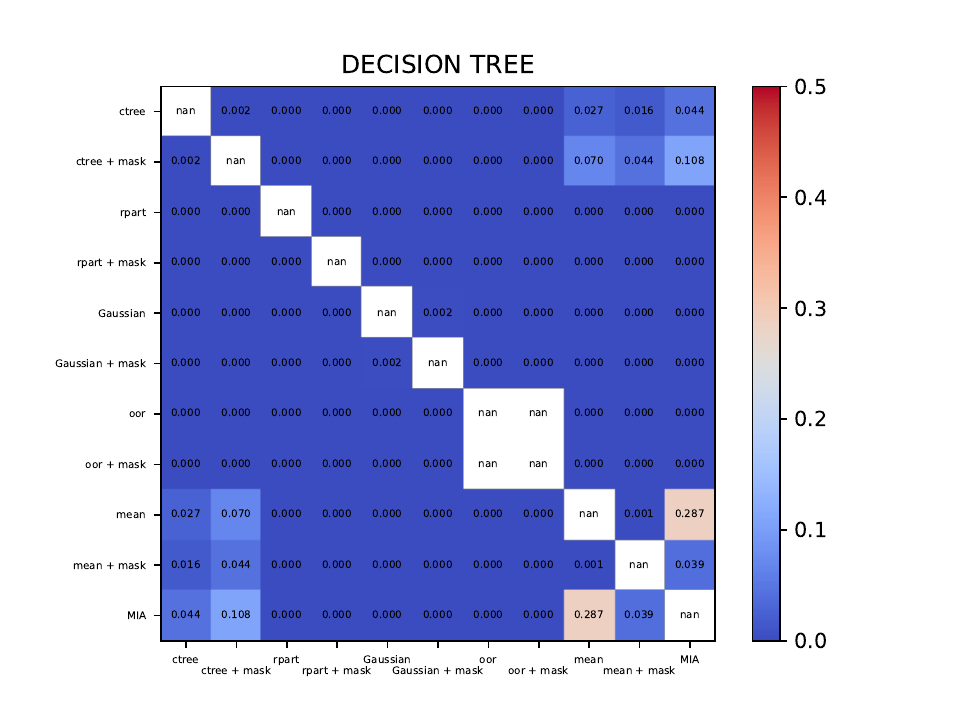}
\vspace*{-10mm}
\caption{P-values of paired t-tests for each pair of imputation method and for each learning algorithm in a MCAR model ($\rho=0.5$, missing rate of $60\%$)}
\label{fig:ttests_mcar_miss6_rho5}
\end{figure}

\newpage 

\begin{figure}[H]
    \centering
\includepdf[scale=0.8, nup=2x3, pages=1-5]{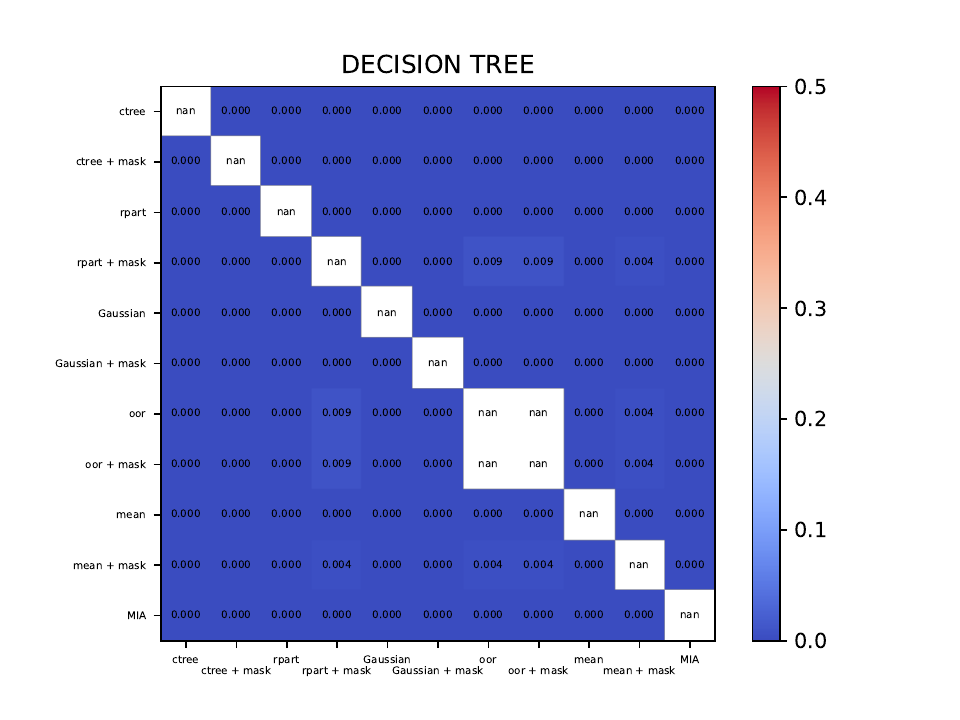}
\vspace*{-10mm}
\caption{P-values of paired t-tests for each pair of imputation method and for each learning algorithm in a MNAR model ($\rho=0.5$, missing rate of $60\%$)}
\label{fig:ttests_mnar_miss6_rho5}
\end{figure}

\newpage 

\begin{figure}[H]
    \centering
\includepdf[scale=0.8, nup=2x3, pages=1-5]{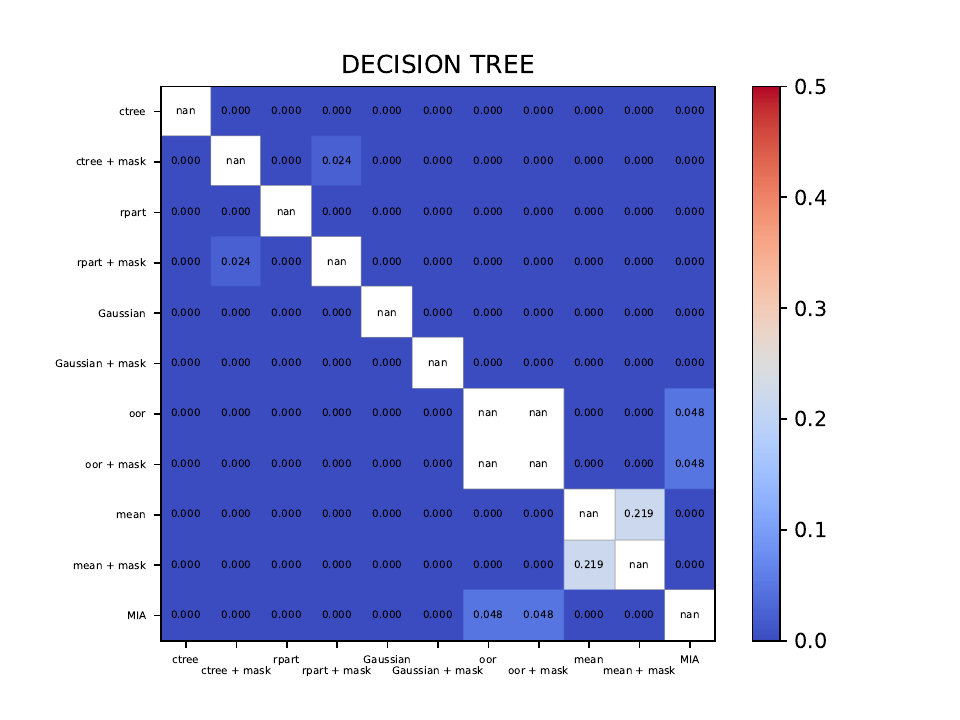}
\vspace*{-10mm}
\caption{P-values of paired t-tests for each pair of imputation method and for each learning algorithm in a Predictive M model ($\rho=0.5$, missing rate of $60\%$)}
\label{fig:ttests_pred_miss6_rho5}
\end{figure}

\end{document}